\RequirePackage[l2tabu, orthodox]{nag}
\documentclass[12pt]{article}
\usepackage[T1]{fontenc}
\usepackage{microtype}

\usepackage{amsmath}
\usepackage{amssymb}
\usepackage{amsthm}
\usepackage{mathtools}
\usepackage{thmtools}
\usepackage{mathrsfs}

\usepackage{libertine}             

\usepackage{titlesec}

\titleformat*{\section}{\Large\bfseries}

\usepackage[usenames,dvipsnames]{xcolor}
\definecolor{shadecolor}{gray}{0.9}

\usepackage[english]{babel}
\usepackage[parfill]{parskip}
\usepackage{afterpage}
\usepackage{framed}
\usepackage{nicefrac}

{\endMakeFramed}

\DeclareRobustCommand{\parhead}[1]{\textbf{#1}~}

\newcounter{parcount}

\usepackage{graphicx}
\usepackage[labelfont=bf]{caption}
\usepackage[format=hang]{subcaption}
\usepackage{wrapfig}

\usepackage{booktabs}
\usepackage{multirow}
\usepackage{longtable}

\usepackage[sort]{natbib}
\usepackage{bibunits}

\usepackage[algoruled]{algorithm2e}
\usepackage{listings}
\usepackage{fancyvrb}
\fvset{fontsize=\normalsize}

\usepackage{hyperref}
\hypersetup{colorlinks,linktoc=all}
\usepackage[all]{hypcap}
\hypersetup{citecolor=Violet}
\hypersetup{linkcolor=black}
\hypersetup{urlcolor=MidnightBlue}

\usepackage[nameinlink]{cleveref}

\usepackage[acronym,smallcaps,nowarn]{glossaries}

\usepackage{listings}
\usepackage{lstbayes}
\lstdefinestyle{mystyle}{
    commentstyle=\color{OliveGreen},
    numberstyle=\tiny\color{black!60},
    stringstyle=\color{BrickRed},
    basicstyle=\ttfamily\scriptsize,
    breakatwhitespace=false,
    breaklines=true,
    captionpos=b,
    keepspaces=true,
    numbers=none,
    numbersep=5pt,
    showspaces=false,
    showstringspaces=false,
    showtabs=false,
    tabsize=2
}
\usepackage{enumitem}

\DeclareRobustCommand{\mb}[1]{\ensuremath{\mathbf{\boldsymbol{#1}}}}

\DeclareMathOperator*{\argmax}{arg\,max}
\DeclareMathOperator*{\argmin}{arg\,min}

\crefname{lemma}{lemma}{lemmas}
\Crefname{lemma}{Lemma}{Lemmas}
\crefname{thm}{theorem}{theorems}
\Crefname{thm}{Theorem}{Theorems}
\crefname{prop}{proposition}{propositions}
\Crefname{prop}{Proposition}{Propositions}
\crefname{defn}{definition}{definitions}
\Crefname{defn}{Definition}{Definitions}

\newtheorem{thm}{Theorem} 
\newtheorem{defn}[thm]{Definition} 
\newtheorem{lemma}[thm]{Lemma}

\newcommand{\mbW}{\mb{W}}

\newcommand{\mbx}{\mb{x}}
\newcommand{\mbX}{\mb{X}}

\newcommand{\mby}{\mb{y}}
\newcommand{\mbY}{\mb{Y}}

\newcommand{\mbz}{\mb{z}}
\newcommand{\mbZ}{\mb{Z}}

\newcommand{\mbc}{\mb{c}}
\newcommand{\mbC}{\mb{C}}

\newcommand\dif{\mathop{}\!\mathrm{d}}

\newcommand{\supp}{\textrm{supp}}

\newcommand{\cN}{\mathcal{N}}

\newcommand{\cX}{\mathcal{X}}
\newcommand{\cY}{\mathcal{Y}}

\newcommand{\cS}{\mathcal{S}}

\newcommand{\g}{\, | \,}
\newcommand{\s}{\, ; \,}

\newcommand{\rmdo}{\mathrm{do}}

\newcommand{\E}[2]{\mathbb{E}_{#1}\left[#2\right]}

\newacronym{POMDP}{pomdp}{partially observable {M}arkov decision process}
\newacronym{MDP}{mdp}{{M}arkov decision process}

\newacronym{PNS}{pns}{probability of necessity and sufficiency}
\newacronym{PS}{ps}{probability of sufficiency}
\newacronym{PN}{pn}{probability of necessity}
\newacronym{POC}{poc}{probabilities of causation}
\newacronym{PPCA}{ppca}{probabilistic principal component analysis}

\newacronym{GMM}{gmm}{Gaussian mixture model}

\newacronym{VAE}{vae}{variational autoencoder}
\newacronym{RDR}{rdr}{related disequilibrium regression}

\newacronym{SCM}{scm}{structural causal model}

\newacronym{Causal-Rep}{causal-rep}{Causal Representation}

\newacronym{ELBO}{elbo}{evidence lower bound}
\newacronym{OOD}{ood}{out-of-distribution}

\newacronym{IOSS}{ioss}{independence-of-support score}
\newacronym[
  longplural={factors of variation}
]{FOV}{fov}{factor of variation}

\usepackage[margin=1in]{geometry}
\usepackage{tcolorbox}
\usepackage{adjustbox}

\usepackage[defaultlines=3,all]{nowidow}

\title{Desiderata for Representation Learning: \\ A Causal
Perspective}

\author{
  Yixin Wang\\
        UC Berkeley, EECS\\
  ywang@eecs.berkeley.edu\\
  \and
  Michael I.~Jordan\\
          UC Berkeley, EECS and Statistics\\
  jordan@cs.berkeley.edu\\
  }

\date{\today}
\begin{document}
\maketitle

\begin{bibunit}[alp]

\begin{abstract}
Representation learning constructs low-dimensional representations to
summarize essential features of high-dimensional data. This learning
problem is often approached by describing various desiderata
associated with learned representations; e.g., that they be
non-spurious, efficient, or disentangled. It can be challenging,
however, to turn these intuitive desiderata into formal criteria that
can be measured and enhanced based on observed data.  In this paper,
we take a causal perspective on representation learning, formalizing
non-spuriousness and efficiency (in supervised representation
learning) and disentanglement (in unsupervised representation
learning) using counterfactual quantities and observable consequences
of causal assertions.  This yields computable metrics that can be used
to assess the degree to which representations satisfy the desiderata
of interest and learn non-spurious and disentangled representations
from single observational datasets.

\end{abstract}



\section{Introduction}

\glsreset{PNS}
\label{sec:intro}

Representation learning constructs low-dimensional representations
that summarize essential features of high-dimensional data. For
example, one may be interested in learning a low-dimensional
representation of MNIST images, where each image is a
$784$-dimensional vector of pixel values. Alternatively, one may be
interested in a product review corpus; each review is a
$5,000$-dimensional word count vector. Given an $m$-dimensional data
point, $\mbX=(X_1, \ldots, X_{m})\in \mathbb{R}^m$, the goal is to
find a $d$-dimensional representation $\mbZ = (Z_1, \ldots, Z_d)
\triangleq (f_1(\mbX), \ldots, f_d(\mbX))$ that captures $d$ important
features of the data, where $f_j:\mathbb{R}^m\rightarrow
\mathbb{R}, j=1, \ldots, d$ are $d$ deterministic functions and~$d \ll
m$.

A heuristic approach to the problem has been to fit a neural network
that maps from the high-dimensional data to a set of labels, and then
take the top layer of the neural network as the representation of the
image. When labels are not available, a related heuristic is to fit a
latent variable model (e.g., a variational
autoencoder~\citep{kingma2014auto}) and output a low-dimensional
representation based on the inferred latent variables. In both cases,
the hope is that these low-dimensional representations will be useful
in the performance of downstream tasks and also provide an
interpretation of the statistical relationships underlying the data.

These heuristic approaches do not, however, always succeed in
producing representations with desirable properties. For example, as
we will discuss in detail, common failure modes involve capturing
spurious features that do not transfer well or finding dimensions that
are entangled and are hard to interpret. For example, in fitting a
neural network to images of animals, with the goal of producing a
labeling of the species found in the images, a network may capture
spurious background features (e.g., grass) that are highly correlated
with the animal features (e.g., the face of a dog). Such spurious
features can often predict the label well. But they are generally not
useful for prediction in a different dataset or for performing other
downstream tasks. The learned representation may also be entangled---a
single dimension of the representation may encode information about
multiple features (e.g., animal fur and background lighting). Such
entangled representations are hard to interpret. The representation
itself does not provide guidance on how to separate learned dimensions
into more informative dimensions that describe animal fur and
background lighting.

While non-spuriousness or disentanglement are natural desiderata of
representations, they are often intuitively defined, challenging to
evaluate, and hard to optimize over algorithmically. Lacking formal
metrics for these desiderata, and not having access to manually
labeled features (e.g., grass, animal fur, or background lighting)
that provide empirical guidance, prevents us from developing
representation learning algorithms that satisfy natural desiderata.

In this work, we take a causal inference perspective on representation
learning.  This perspective allows us to formalize representation
learning desiderata using causal notions, specifically causal
relationships among the label $Y$ and the $d$ features captured by
$Z_1, \ldots, Z_d$. This yields calculable metrics obtained from the
observable implications of the underlying causal relationships. These
metrics then enable representation learning algorithms that target
these desiderata. (\Cref{fig:representation-workflow} illustrates this
workflow.)

We focus on two sets of desiderata in representation learning: (1)
efficiency and non-spuriousness in supervised representation
learning---i.e., the representations shall efficiently capture
non-spurious features of the data---and (2) disentanglement in
unsupervised representation learning---i.e., the representations shall
encode different features along separate dimensions.

In the supervised setting, the key idea is to view representations as
capturing features that are potential \emph{causes} of the label. From
this perspective, a non-spurious representation should capture
features that are \emph{sufficient causes} of the label. This property
ensures that the same representation is still likely to be informative
of the label when employed on a new dataset. In the running example of
images, a representation that captures the dog-face feature is
non-spurious because the presence of a dog face is sufficient to
determine the dog label (i.e., whether a dog is present in the image).
In contrast, a representation based on the grass feature is spurious
because it cannot causally determine the dog label, even though it is
highly correlated. We formalize this connection between non-spurious
representations and sufficient causes by defining non-spuriousness
using the notion of \emph{probability of sufficiency} (PS), due
to~\citet{Pearl2011}.

The same causal perspective also allows us to formalize the efficiency
of representations, which are those that do not pick up redundant
features. Again viewing the representation as a potential cause of the
label, its efficiency corresponds to the necessity of the cause. In
the image example, a representation of both ``dog-face'' and
``four-leg'' is inefficient because it is not necessary to know both
``dog-face'' and ``four-leg'' to determine the dog label. As dogs
always have four legs (in our simplified world), ``four-leg'' is a
redundant feature given ``dog-face.'' We therefore formalize
efficiency of a representation using the notion of \emph{probability
of necessity} (PN). As we will discuss, PN and PS are aspects of a
general notion of \emph{probabilities of causation}, also due to
\citet{Pearl2011}.

\begin{figure}[t]
  \centering
    \vspace{-150pt}
  \includegraphics[width=\linewidth]{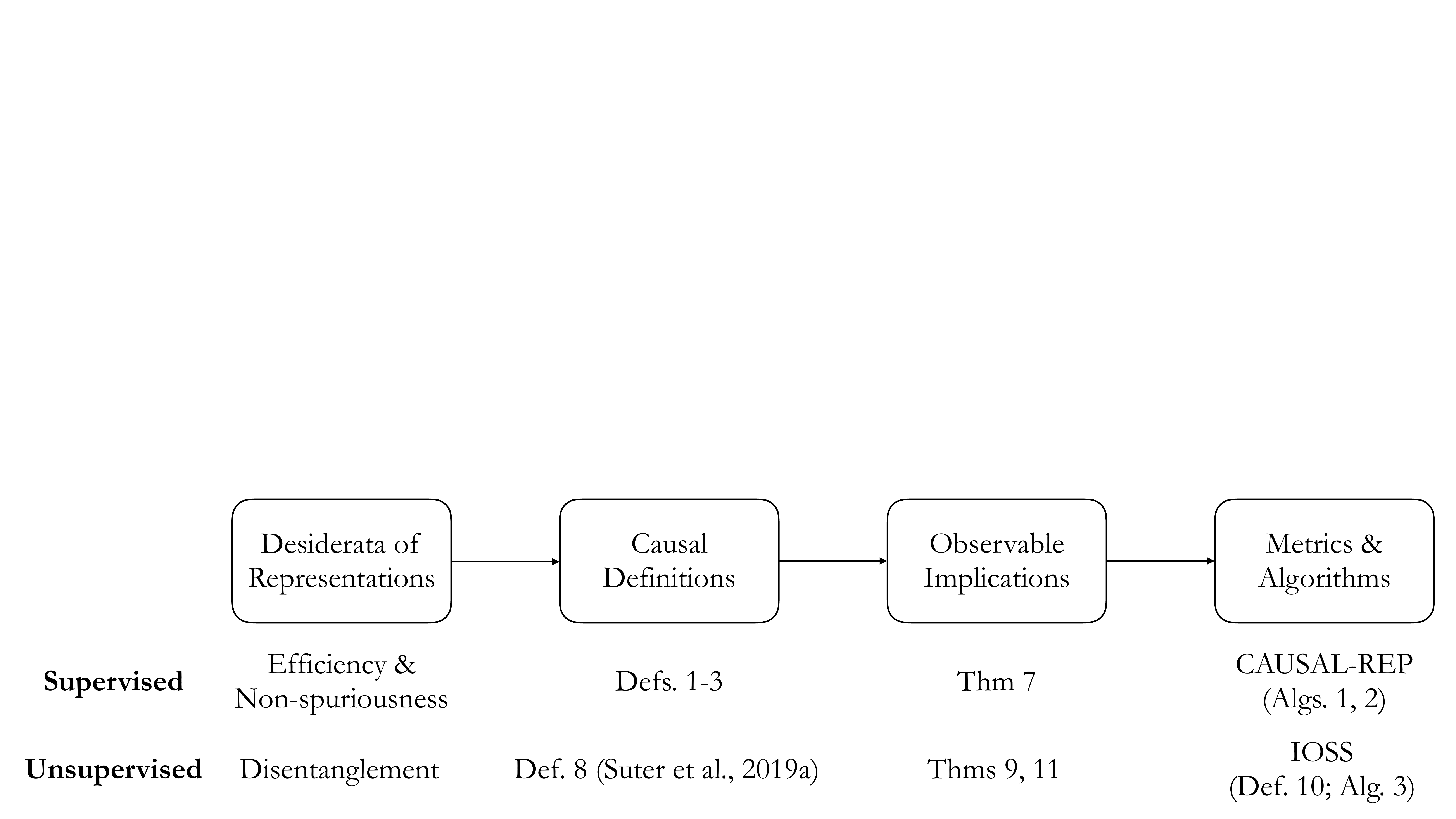}
  \caption{A causal perspective of representation learning: From
  desiderata to algorithms}
\label{fig:representation-workflow}
\end{figure}

While these causal definitions formalize efficiency and
non-spuriousness, they do not immediately imply calculable metrics for
these desiderata, because not all causal quantities are estimable from
observational data. To obtain calculable metrics, we have to study
observable implications of these causal definitions, a problem known
as \emph{causal identification}~\citep{Pearl2011}. We discuss the
unique challenges of causal identification with high-dimensional
data---specifically, high-dimensional data are often
rank-degenerate~\citep{stewart1984rank,golub1976rank}---and we develop
identification strategies to address these challenges. These
strategies lead to calculable metrics of efficiency and
non-spuriousness in the high-dimensional setting, along with the
conditions under which they are valid.

Based on these definitions, we develop an algorithmic framework that
we refer to as CAUSAL-REP that formulates representation learning as a
task of finding necessary and sufficient causes. In a range of
empirical studies, we find that CAUSAL-REP is more successful at
finding non-spurious representations in both images and text compared
to standard benchmarks.

We now turn to the unsupervised setting where our focus will be the
desideratum of \emph{disentangled representations}. Intuitively,
disentanglement requires that different dimensions of the learned
representation correspond to independent degrees of freedom of
objects; that is, we can view disentangled representations as making
it possible to generate new examples of objects by separately
manipulating the values of the features encoded by such
representations.

To develop metrics for disentanglement, we begin with a causal
definition of disentanglement due to~\citet{Suter2019}: different
dimensions of the learned representation should correspond to
different features that do not causally affect each other. (These
features may still be correlated.)  While a useful starting place,
this definition unfortunately leaves us short of the goal of measuring
and optimizing for disentanglement. Naively, it requires one to assess
whether causal connections exist among different features, which is a
generally impossible task without substantial knowledge about the
underlying causal structure~\citep{Pearl2011,imbens2015causal}. In
particular, we may observe correlation among variables, with or
without causal connections. Thanks to this difficulty, existing
disentanglement metrics often rely on ground-truth features or
auxiliary labels~\citep{Suter2019,thomas2018disentangling}, which we
often do not have access to in practice.

\glsreset{IOSS}

To obtain an operational measure of disentanglement, we again turn to
studying the observable implications of its causal definition. The key
observation is that the observable implications can exist in the
\emph{support} of the representation, i.e., the set of values the
representation can take. Specifically, we find that
disentanglement---the absence of causal connections among the features
captured by different dimensions of the representation---implies that
their support must be independent in observational data (under a
standard positivity condition). Visually, features with independent
support must have a scatter plot that occupies a hyperrectangular
region: the set of possible values each feature can take does not
depend on the values of other features. See
\Cref{fig:uncorr_ys,fig:corr_ys,fig:entangle_ys} for examples.

Building on this connection between disentanglement and independent
support, we develop the \emph{independence-of-support score} (IOSS) as
an unsupervised disentanglement metric; it evaluates whether different
dimensions of the representation have independent support. We
establish the sufficiency of this metric by showing that
representations with independent support are identifiable (under
suitable conditions). This fundamental identifiability result enables
us to enforce disentanglement in representation learning via an IOSS
penalty.

Through the study of non-spuriousness/efficiency and disentanglement,
we illustrate how causality can provide a fruitful perspective for
representation learning. Many intuitively defined desiderata in
representation learning can be formalized using causal notions. These
causal notions can further lead to metrics that quantify these
desiderata and representation learning algorithms that enforce these
desiderata.


\parhead{Related work.} There has been a flurry of recent work at
the intersection of representation learning and causal inference. We
provide a brief review, highlighting the contrast between this
literature and our work.

\emph{Learning non-spurious representations via invariance.} Many
works have considered causal formulations of representation learning
given datasets from multiple
environments~\citep{khasanova2017graph,zhao2019learning,moyer2018invariant,lu2021invariant,mitrovic2020representation,moraffah2019deep,arjovsky2019invariant,cheng2017causal,veitch2021counterfactual,creager2021environment,puli2021predictive}.
The basic idea is to enforce the invariance of the mapping between the
learned representation and the outcome label, thereby encouraging
non-spurious representations and enabling out-of-distribution
prediction. Similar to these works, our CAUSAL-REP algorithm targets
non-spurious representations. However, it differs in its focus on the
setting in which only a single observational dataset is available.
That is, we do not assume access to datasets from multiple
environments, nor do we leverage the invariance principle.

\emph{Reverse causal inference.} Our work is also related to a body of
work on reverse causal inference, a task that aims to find ``causes of
effects''~\citep{scholkopf2013semi,janzing2015semi,weichwald2014causal,kilbertus2018generalization,scholkopf2012causal,paul2017feature,wang2020identifying,wang2020robustness,kommiya2021towards,galhotra2021explaining,watson2021local,Chalupka2015,Chalupka2017,gelman2013ask}.
Existing approaches formulate the search for causes as causal
hypotheses generation and testing. In contrast, our CAUSAL-REP
algorithm formulates this search as maximizing
\glsreset{POC}\gls{POC}~\citep{Pearl2011,Pearl2019,tian2000probabilities,mueller2021causes},
specifically in the context of representation learning. We develop
identification conditions for these \gls{POC} for high-dimensional
data such as images and
text~\citep{nabi2020semiparametric,puli2020causal,wang2019blessings,wang2020towards,ranganath2019multiple,wang2021proxy,pryzant2020causal,grimmer2020causal,fong2016discovery,wang2020identifying,wang2020robustness}.
Our results build on existing identification results around multiple
causes with shared unobserved
confounding~\citep{wang2019blessings,wang2020towards,puli2020causal}
and its positivity
issues~\citep{d2020overlap,d2019multi,d2019comment,imai2019comment},
but are tailored to the representation learning setting where no
unobserved confounding is present.

\emph{Disentangled representation learning.} Our
\glsreset{IOSS}\gls{IOSS} relates to the broad
literature on disentangled representation learning. Early approaches
to disentanglement often enforce statistical independence among
different dimensions of the representation
\citep{bengio2013representation,achille2018emergence,higgins2017beta,
burgess2018understanding, kim2018disentangling, chen2018isolating,
kumar2018variational}. However, \citet{locatello2019challenging} show
that this inductive bias of statistical independence is insufficient
for disentanglement due to its non-identifiability. Many recent
approaches to disentanglement thus incorporate auxiliary information
to enable identifiability in
disentanglement~\citep{locatello2019disentangling,khemakhem2020variational,bouchacourt2018multi,
hosoya2019group,shu2019weakly,locatello2020weakly,Trauble2020,yang2020causalvae,shen2020disentangled}.
In contrast to these works, \gls{IOSS} relies on the causal
perspectives of \citet{Suter2019} and establishes
identifiability in disentanglement without the need for supervision.
Focusing on compactly supported representations, we show that
representations with independent support are identifiable under
suitable conditions.

\emph{Causal structure learning.} Our causal approach to
disentanglement also relates to causal structure learning; both
involve assessing the existence of causal connections between
variables. Traditional approaches to causal structure learning relies
on independence tests or score-based methods, often assuming all
variables are observed; see \citet{heinze2018causal} for a review.
\gls{IOSS} differs from these works in that we allow for unobserved
common causes among the observed variables; we also rely on a
different observable implication of the lack of causal connections.

\emph{Representation learning for causal inference.} Representation
learning and dimensionality reduction have significantly improved the
estimation efficiency of causal inference with high-dimensional
covariates and
treatments~\citep{johansson2019support,nabi2020semiparametric,johansson2020generalization,shi2020invariant,wu2021identifying,veitch2019using,veitch2020adapting}.
The focus in these work is on how representation learning can help
causal inference with high-dimensional covariates. In contrast, we
focus on how causal inference can help produce useful representations.


\parhead{This paper.} The rest of the paper is organized as follows.
\Cref{sec:nonspuriousness} develops a causal approach to efficiency
and non-spuriousness in supervised representation learning.
\Cref{subsec:define-pns} formalizes efficiency and non-spuriousness
using counterfactual notions, namely the probability of necessity and
sufficiency in causal inference. \Cref{subsec:pns-id} discusses the
identification of these counterfactuals, which leads to metrics for
efficiency and non-spuriousness. \Cref{subsec:pns-learn} formulates
representation learning as a task of finding necessary and sufficient
causes of the label and develops CAUSAL-REP, an algorithm that targets
efficient and non-spurious features. \Cref{sec:causalrep-empirical}
provides an empirical study of CAUSAL-REP, showing that it captures
non-spurious features and improves downstream out-of-distribution
prediction. \Cref{section:disentanglement} switches gears to develop a
causal approach to disentanglement in unsupervised representation
learning. \Cref{subsec:disentangle-def} reviews a causal definition
of disentanglement.
\Cref{subsec:disentangle-metric} studies its observable implications,
showing that causal disentanglement implies independent support of the
representation. This leads to \gls{IOSS}, an unsupervised
disentanglement metric. \Cref{subsec:ioss-learn} establishes the
identifiability of representations with independent support,
enabling disentangled representation learning with an \gls{IOSS}
penalty. \Cref{sec:ioss-empirical} demonstrates the effectiveness of
\gls{IOSS} in empirical studies. Finally, \Cref{sec:conclusion}
concludes the paper with a discussion of causal perspectives on
representation learning.


\section{Supervised Representation Learning: Efficiency \\and
Non-spuriousness}
\label{sec:nonspuriousness}

We begin by formulating a set of desiderata for supervised
representation learning. As a running example, we consider a dataset
of images and their labels. We focus on settings where the labels are
obtained by annotators viewing each image and then labeling; they
describe the perceived features of the image, as opposed to the
intended features of the image. The goal is to construct a
low-dimensional image representation that efficiently captures its
essential features.

Formally, we consider a dataset that contains $n$  $m$-dimensional
data points and their labels, $\{\mbx_i, y_i\}_{i=1}^n\in
\mathcal{X}^m\times\mathcal{Y}$, assumed to be sampled i.i.d., with
$\mbx_i = (x_{i1}, \ldots, x_{im})$. The variable $x_{il}$ might be
the value of the $l$th pixel of an image, with $\mathcal{X}=[0,255]$;
it might also be the number of times the $l$th word in the vocabulary
occur in the document, with $\mathcal{X}=\mathbb{Z}^+$. The goal of
representation learning is to find a deterministic function
$f:\mathcal{X}^m\rightarrow \mathbb{R}^d$ mapping the $m$-dimensional
data into a $d$-dimensional space $(d\ll m)$. Thus each data point
$\mbx_i$ has a $d$-dimensional representation $\mbz_i\triangleq
f(\mbx_i)$, or equivalently $\mbz_i = (z_{i1}, \ldots, z_{id})
\triangleq (f_1(\mbx_i), \ldots, f_d(\mbx_i))$, with
$f_j:\mathcal{X}^m \rightarrow
\mathbb{R}$.

\glsreset{POC}

Ideally, such a representation should be efficient and non-spurious.
It shall efficiently capture features of the image that are essential
for determining the label. Below we formalize the desiderata of
efficiency and non-spuriousness using counterfactual notions.

\subsection{Defining Efficiency and Non-spuriousness using
Counterfactuals}

\label{subsec:define-pns}

To formalize the desiderata of efficiency and non-spuriousness, we
posit a \gls{SCM} that describes the data generating process. This
\gls{SCM} will enable us to define these desiderata through
counterfactual quantities defined via this model.

\glsreset{SCM}
\subsubsection{A structural causal model of supervised representation
learning}

We describe a \gls{SCM} for supervised representation learning and the
counterfactual quantities therein.

\glsreset{SCM}
\parhead{The \gls{SCM} of the labeled dataset.} We posit
\Cref{fig:causalrep-scm} as the \gls{SCM} of a data-generating
process. (For notation simplicity, we suppress the data point index
$i$ and focus on categorical $Y$.) The figure embodies the assumption
that the high-dimensional object $\mbX=(X_1, \ldots, X_m)$ causally
affects its label~$Y$. Moreover, there is no confounder between the
two; the label~$Y$ can be fully determined by the object $\mbX$. These
assumptions hold because we assume the labels are obtained by
annotators viewing the images; the labeling is solely based on the
images.

\Cref{fig:causalrep-scm} also posits that different dimensions of the
high-dimensional object, $X_1, \ldots, X_m$, are correlated due to
some unobserved common cause $\mbC$, which can potentially be
multi-dimensional. For example, $\mbC$ can represent the design of an
image; nearby pixels tend to be highly correlated if an image has a
smooth design. This assumption holds because pixel values of images
are not independent; they are determined by common latent
variables~\citep{kingma2014auto,goodfellow2014generative}.

\parhead{The representation $f(\mbX)$ and its interventions.} The
representation $\mbZ=f(\mbX)$ captures features of the
high-dimensional data $\mbX$. For example, $\mbZ = (Z_1, Z_2) =
(\mathbb{I}\{X_{125} \times X_{38} < 0.01\}, X_{164} + 0.3 \cdot
X_{76}^2)$ is a two-dimensional representation of $\mbX$, whose first
dimension represents whether pixels 125 and 38 are close to black.
More complex representations may capture more meaningful features;
e.g., whether a dog face is present in the image.

The representation $\mbZ=f(\mbX)$ is not included in the \gls{SCM} as
a separate causal variable in addition to~$\mbX$. The reason is that
the representation $\mbZ=f(\mbX)$ is a \emph{deterministic} function
of the image pixels $\mbX$ and the causal variables of a \gls{SCM}
must be related by functional relationships perturbed by \emph{random}
disturbances~\citep{Pearl2011,pearl1995causal}. In particular, $\mbZ$
is not a descendant of $\mbX$ because a value change in $\mbZ$ would
imply a value change in $\mbX$ (unless $f$ is a constant function).
Nor is $\mbZ$ an ancestor of $\mbX$ because a change in the value of
$\mbX$ can also imply a change in $\mbZ$.

Despite $\mbZ$ being not explicitly included in the \gls{SCM},
\Cref{fig:causalrep-scm} does imply that $\mbZ=f(\mbX)$ is an ancestor
of the label $Y$. This is because $\mbX$ is an ancestor of $Y$.
Moreover, interventions on the representation $\mbZ=f(\mbX)$ are also
well defined; they are \emph{functional
interventions}~\citep{puli2020causal,correa2020calculus,eberhardt2007interventions},
also referred to as \emph{stochastic policies}~\citep[Ch.
4.2]{Pearl2011}. (We will define them rigorously in
\Cref{subsec:pns-id}.) Roughly, suppose a one-dimensional
representation $Z=f(\mbX)$ captures the univariate binary feature of
whether the grass is present in the image. Then the intervention
$\rmdo(Z=1)$---equivalently $\rmdo(f(\mbX)=1)$---means ``turning on''
the grass feature in the image. That is, holding the design of the
image $\mbC$ fixed, we change the pixels of the image $\mbX$ such that
the grass is present: $Z=1$. Accordingly, the intervention
$\rmdo(Z=0)$ means ``turning off'' the grass feature. Holding the
design $\mbC$ fixed, we change the pixels $\mbX$ such that the grass
is absent, $Z=0$.

\parhead{Counterfactuals in the \gls{SCM}.} The \gls{SCM}
(\Cref{fig:causalrep-scm}) results in a family of counterfactuals,
``what would the value of a variable be if we intervene on some
variables of the causal model.'' Here we focus on the counterfactuals
obtained when we intervene on the image features captured by the
representation $\mbZ$. These counterfactuals will allow us to define
the efficiency and non-spuriousness of the representations.

We denote $Y(\mbZ=\mbz)$ as the counterfactual label of an image if we
force its representation $\mbZ$ to take value $\mbz$.\footnote{The
counterfactual label $Y(\mbZ=\mbz)$ is also commonly written as
$Y_\mbz$~\citep{Pearl2011}. Here we employ the parenthesis notation
$Y(\mbZ=\mbz)$ to avoid subscripts.} For example, if a one-dimensional
representation $Z$ captures the feature ``whether grass is present in
the image,'' then $Y(Z=0)$ is the counterfactual label had the image
had no grass. Accordingly, $Y(Z=1)$ is the counterfactual label had
the image had grass being present.

The key idea here is to evaluate the quality of the representation
$\mbZ$ by reasoning about the counterfactual labels: What would the
label $Y$ be had $\mbZ$ taken different values? We will show how the
desiderata of non-spuriousness and efficiency can be formalized
via these counterfactuals. This connection will allow us to
develop metrics and algorithms for these desiderata.

\begin{figure}
\centering
\centering
\begin{adjustbox}{height=6cm}
\includegraphics{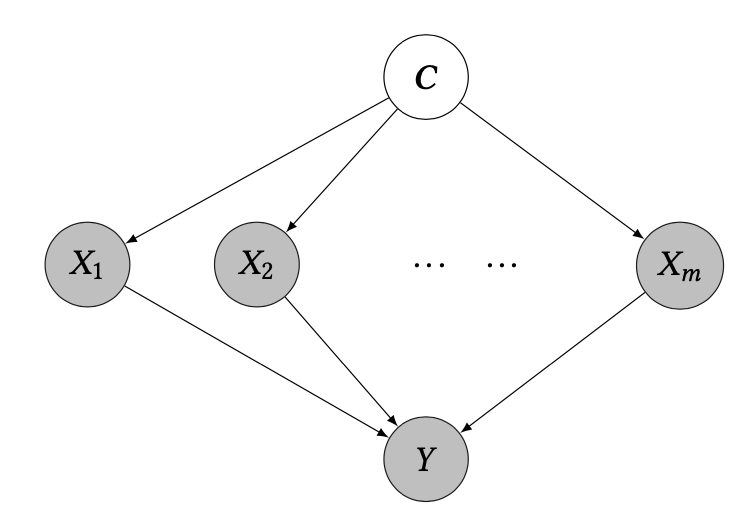}
\end{adjustbox}
\caption{A \gls{SCM} for supervised representation learning.
$\mbX=(X_1,\ldots,X_m)$ represents the high-dimensional object (e.g.,
pixels of an image), $Y$ represents the outcome label, and $\mbC$
denotes the unobserved common cause of $X_1, \ldots,
X_m$.\label{fig:causalrep-scm}}
\end{figure}

\subsubsection{Counterfactual definitions of non-spuriousness and
efficiency}

We next discuss how these counterfactual notions can allow us to
formalize the non-spuriousness and efficiency of a representation
$f(\mbx)$, for a single observed data point $(\mbX=\mbx, Y=y)$.

\textbf{Non-spuriousness and its counterfactual definition.} A
non-spurious representation captures features that can (causally)
determine the label. In the causal model, we say that a representation
captures non-spurious features if, given an image without the feature,
including this feature would change its label.

As an example, suppose the label $Y$ indicates whether the image
contains a dog. Then a representation $Z$ capturing the presence of
dog-face is a non-spurious feature. Given a non-dog image without a
dog face $(Y=0, Z=0)$, adding a dog face to the image turns on the dog
label, i.e., $Y(Z=1) = 1$. In other words, dog-face is a sufficient
case of a dog label. In contrast, a representation $Z$ that captures
the presence of grass is a spurious feature. Though the presence of
grass and the dog label may be highly correlated, adding grass to a
non-dog image does not turn on the dog label.

\glsreset{PS}

Viewing the representation $\mbZ$ as a potential cause of the label, a
non-spurious representation shall be a \emph{sufficient cause} of the
label. Turning on the feature (captured by) $\mbZ$ should be
sufficient to turn on the label. We thus measure the non-spuriousness
of a representation by the \gls{PS} of the feature $\mbZ$ causing the
label $Y$.

\glsreset{PS}
\begin{defn}[Non-spuriousness of representations]
\label{defn:ps}
Suppose we observe a data point with representation $\mbZ=\mbz$ and
label $Y=y$. Then the non-spuriousness of the representation $\mbZ$
for label $Y$ is the \gls{PS} of $\mathbb{I}\{\mbZ=\mbz\}$ for
$\mathbb{I}\{Y=y\}$:
\begin{align}
\label{eq:ps}
\gls{PS}_{\mbZ=\mbz, Y=y}=P(Y(\mbZ=\mbz)=y\g \mbZ\ne\mbz, Y\ne y).
\end{align}
When both the representation $Z$ and the label $Y$ are univariate
binary with $z=1, y=1$, then \Cref{eq:ps} coincides with classical
definition of \gls{PS}~(Definition 9.2.2 of \citet{Pearl2011}).
\end{defn}

$\gls{PS}_{\mbZ=\mbz, Y=y}$ is the probability of the representation
$\mathbb{I}\{\mbZ=\mbz\}$ being a \emph{sufficient cause} of the label
$\mathbb{I}\{Y=y\}$. In Pearl's language, it describes the capacity of
the representation to ``produce'' the label. \gls{PS} measures the
probability of a positive counterfactual label if we \emph{make} the
feature $\mathbb{I}\{\mbZ=\mbz\}$ be present (and all else equal),
conditional on the feature being absent and the label being negative
$\{\mathbb{I}\{\mbZ=\mbz\}=0, \mathbb{I}\{Y=y\}=0\}$. A non-spurious
representation shall have a high \gls{PS} for the label~of~interest.

\textbf{Efficiency and its counterfactual definition.} An efficient
representation captures only essential features of the data; it does
not capture any redundant features. In the causal model, we say that
the representation is efficient if, given an image with the feature it
captures, removing this feature would change its label.

Returning to the image example with the dog label, a representation
that captures ``dog-face + four-leg'' is inefficient. The reason is
that, for an image with both a dog face and four legs, removing one of
the dog legs and hence turning off the ``dog-face + four-leg'' feature
does not necessarily turn off the ``dog label.'' In contrast, the
feature ``dog-face'' is efficient as removing a dog face from an image
will turn off the ``dog label.''

An efficient representation must capture features that are
\emph{necessary causes} of the label. We can therefore define the
efficiency of a representation by the \gls{PN} of the feature causing
the label~\citep{Pearl2011}.

\glsreset{PN}
\begin{defn}[Efficiency of representations]
\label{defn:pn}
Suppose we observe a data point with representation $\mbZ=\mbz$ and
label $Y=y$. Then the efficiency of the representation $\mbZ$ for the
label $Y$ is the \gls{PN} of $\mathbb{I}\{\mbZ=\mbz\}$ for
$\mathbb{I}\{Y=y\}$:\footnote{The distribution of the counterfactual
$Y(\mbZ\ne\mbz)$ is defined as that of a soft
intervention~\citep{correa2020calculus,eberhardt2007interventions};
a.k.a.\ a stochastic policy~\citep[Ch. 4.2]{Pearl2011}:
\begin{align*}
P(Y(\mbZ\ne\mbz))
= \int P(Y(\mbZ))P(\mbZ\g \mbZ\ne\mbz)\dif \mbZ.
\end{align*}}
\begin{align}
\label{eq:pn}
\gls{PN}_{\mbZ=\mbz, Y=y}=P(Y(\mbZ\ne\mbz)\ne y\g \mbZ=\mbz, Y= y).
\end{align}
When both the representation $Z$ and the label $Y$ are univariate
binary with $z=1, y=1$, then \Cref{eq:pn} coincides with classical
definition of \gls{PN}~(Definition 9.2.1 of \citet{Pearl2011}).\end{defn}

$\gls{PN}_{\mbZ=\mbz, Y=y}$ is the probability of the representation
$\mathbb{I}\{\mbZ=\mbz\}$ being a \emph{necessary cause} of the label
$\mathbb{I}\{Y=y\}$. It measures the probability of a negative
counterfactual label if we \emph{remove} the feature (i.e., setting
$\mathbb{I}\{\mbZ=\mbz\}=0$ and all else equal), conditional on the
feature being present and the label being positive
$\{\mathbb{I}\{\mbZ=\mbz\}=1, \mathbb{I}\{Y=y\}=1\}$. An efficient
representation shall have a high
\gls{PN} for the label of interest.

\subsubsection{Simultaneous assessment of non-spuriousness and
efficiency}

Representation learning often targets representations that are both
non-spurious and efficient. How can one assess non-spuriousness and
efficiency simultaneously?

\parhead{Assessing non-spuriousness and efficiency simultaneously.} We
invoke the notion of \gls{PNS}. In more detail, non-spuriousness is
intuitively the effectiveness of turning on the label by turning on
the feature captured by the representation; efficiency is the
effectiveness of turning off the label by turning off this feature.
Thus a representation is both non-spurious and efficient if the label
responds to the feature in both ways. If the feature is turned on,
then the label will be turned on; if the feature is turned off, then
the label will be turned off too. \gls{PNS} calculates exactly this
probability.

\glsreset{PNS}
\begin{defn}[Efficiency \& non-spuriousness of representations]
\label{defn:pns} 
Suppose we observe a data point with representation $\mbZ=\mbz$ and
label $Y=y$. Then the efficiency and non-spuriousness of the
representation $\mbZ$ for label $Y$ is the \gls{PNS} of
$\mathbb{I}\{\mbZ=\mbz\}$ for $\mathbb{I}\{Y=y\}$:
\begin{align}
\label{eq:pns}
\gls{PNS}_{\mbZ=\mbz, Y=y}=P(Y(\mbZ\ne\mbz)\ne y, Y(\mbZ=\mbz)=y)).
\end{align}
When both the representation $Z$ and the label $Y$ are univariate
binary with $z=1, y=1$, then \Cref{eq:pns} coincides with classical
definition of \gls{PNS}~(Definition 9.2.3 of \citet{Pearl2011}).
\end{defn}

Requiring both necessity and sufficiency of the cause is a stronger
requirement than requiring only necessity (or only sufficiency).
Accordingly,
\gls{PNS} is a weighted combination of \gls{PN} and
\gls{PS},
\begin{align*}
\gls{PNS}_{\mbZ=\mbz, Y=y} = P(\mbZ=\mbz, Y=y)\cdot \gls{PN}_{\mbZ=\mbz, Y=y} + P(\mbZ\ne\mbz, Y\ne
y)\cdot \gls{PS}_{\mbZ=\mbz, Y=y},
\end{align*}
as per Lemma 9.2.6 of \citet{Pearl2011}. We note that our definitions
of \gls{PN}, \gls{PS}, \gls{PNS} (\Cref{eq:ps,eq:pn,eq:pns})
generalize those in \citet[Ch. 9]{Pearl2011} from univariate binary
causes to general (continuous, discrete, or multi-dimensional) causes.
We next consider two further extensions of the~\gls{PNS}~notion.

\parhead{Extension: Efficiency and non-spuriousness over a dataset.}
The discussion of efficiency and non-spuriousness
(\Cref{defn:ps,defn:pn,defn:pns}) has focused on individual data
points up to this point, reflecting the fact that
\glsreset{POC}\gls{POC} notions are most commonly discussed with
respect to a single occurred event---given that the event $\mbZ=\mbz$
and $Y=y$ has occurred, what is the probability that $\mbZ=\mbz$ is a
sufficient and/or necessary cause of $Y=y$?

In practice, however, we are often interested in whether a
representation, $f:\mathcal{X}^m\rightarrow \mathcal{Y}$, can produce
an efficient and non-spurious summary for datasets of $n$ i.i.d. data
points, $\{(\mbx_i, y_i)\}_{i=1}^n$. We thus extend \Cref{defn:pns} to
this setting:
\begin{align}
\label{eq:pns-dataset}
\gls{PNS}_n(\mbZ, Y)&\triangleq\prod_{i=1}^n \gls{PNS}_{\mbZ=\mbz_i, Y=y_i}=\prod_{i=1}^n P(Y(\mbZ\ne\mbz_i)\ne y_i, Y(\mbZ=\mbz_i)=y_i)),
\end{align}
where $\mbz_i=f(\mbx_i)$ is the representation for data point $i$.
Loosely, this means that a representation is efficient and non-spurious
for a dataset if its efficiency and non-spuriousness holds jointly
across all the data points. (One can similarly extend
\Cref{defn:ps,defn:pn} for \gls{PS} and \gls{PN}.)

\parhead{Extension: Conditional efficiency and non-spuriousness.}
For multi-dimensional representations, one is often interested in the
efficiency and non-spuriousness of each of its dimensions. We expect
each dimension of the representation to be efficient and non-spurious
conditional on all other dimensions.

We thus extend \Cref{defn:pns} to formalize a notion of
\emph{conditional efficiency and non-spuriousness}. Consider a
$d$-dimensional representation $\mbZ = (Z_1, \ldots, Z_d)=(f_1(\mbX),
\ldots, f_d(\mbX))$. The conditional efficiency and non-spuriousness
of the $j$th dimension $Z_j$ for data point $(\mbx_i, y_i)$ is
\begin{align}
\label{eq:pns-conditional}
\gls{PNS}_{Z_j=z_{ij}, Y=y_i \g \mbZ_{-j} = \mbz_{i,-j}}=P(Y(Z_j \ne
z_{ij}, \mbZ_{-j}=\mbz_{i,-j})\ne y_i, Y(Z_j = z_{ij},
\mbZ_{-j}=\mbz_{i,-j})=y_i),
\end{align}
where $z_{ij} = f_j(\mbx_i)$ is the $j$th dimension of the
representation, and $z_{i,-j} = (z_{ij'})_{j'\in\{1, \ldots,
d\}\backslash j}$. Accordingly, the conditional efficiency and
non-spuriousness of $Z_j$ across all $n$ data points is
\begin{align}
\label{eq:pns-conditional-dataset}
\gls{PNS}_n(Z_j, Y \g \mbZ_{-j})&\triangleq\prod_{i=1}^n
\gls{PNS}_{\mbZ=\mbz_i, Y=y_i \g \mbZ_{-j} = \mbz_{i,-j}}.
\end{align}
Conditional efficiency and non-spuriousness describes how the label
responds to the $j$th feature captured by the representation, holding
all other features fixed. Its definition resembles the definition of
(unconditional) efficiency and non-spuriousness. The only difference
is that the conditional notion contrasts the counterfactual label of
$\{Z_j\ne z_{ij}, \mbZ_{-j} = \mbz_{i,-j}\} $ and $\{Z_j= z_{ij},
\mbZ_{-j} = \mbz_{i,-j}\}$, while the unconditional notion contrasts
that of $\{Z_j\ne z_{ij}, \mbZ_{-j}
\ne \mbz_{i,-j}\} $ and $\{Z_j= z_{ij}, \mbZ_{-j} = \mbz_{i,-j}\}$.

The conditional efficiency and non-spuriousness across all dimensions
of a representation is generally a stronger requirement than the
(unconditional) efficiency and non-spuriousness. For example, suppose
a dataset of $(\mbX, Y)$ with $\mbX=(X_1,X_2,X_3)$ is generated by $Y
= X_1 + \epsilon_y$. Now consider a two-dimensional representation,
$\mbZ=(Z_1, Z_2)$, where $Z_1 = X_1$ and $Z_2=Z_1 + \epsilon,
\epsilon\sim\cN(0,\sigma^2)$. When $\sigma^2$ is small, $Z_1$ and
$Z_2$ are highly correlated. In this case, the (unconditional)
efficiency and non-spuriousness of $\mbZ$ is high because $(Z_1, Z_2)$
is indeed non-spurious and (nearly) efficient; $Z_2$ only introduces a
negligible amount of useless information. However, the conditional
efficiency and non-spuriousness of $Z_2$ given $Z_1$ is low because
$Z_2$ is completely useless given $Z_1$.

We will leverage this conditional efficiency and non-spuriousness
notion for representation learning in \Cref{subsec:pns-learn}. First,
however, we study how to measure (conditional) efficiency and
non-spuriousness from observational data.

\subsection{Measuring Efficiency and Non-spuriousness in
Observational Datasets}

\label{subsec:pns-id}

\Cref{defn:ps,defn:pn,defn:pns} formalize the efficiency and
non-spuriousness of representations using \gls{POC}. These
\gls{POC} are counterfactual quantities; calculating them requires
access to a causal structural equation (i.e., the true data
generating process), which is rarely available in practice.

In this section, we study the observable implications of
\Cref{defn:ps,defn:pn,defn:pns}. They lead to strategies to evaluate
the efficiency and non-spuriousness of representations with
observational data; these strategies are known as \emph{causal
identification} strategies~\citep{Pearl2011}. For simplicity of
exposition, we focus on identifying the \gls{PNS} in \Cref{defn:pns}.
(Identification formulas for
\gls{PN} and \gls{PS} in \Cref{defn:pn,defn:ps} can be similarly
derived using Theorem 9.2.15 of \citet{Pearl2011}.)

To identify the counterfactual quantity \gls{PNS} from observational
data, we perform two steps of reduction, climbing down the ladder of
causation~\citep{pearl2019seven}: (1) connect \gls{PNS} to
interventional distributions, and (2) identify these interventional
distributions from observational data. We describe these steps in
detail in the next sections. The end product is
\Cref{corollary:pns-final}, which provides an algorithm for
calculating a lower bound on the \gls{PNS}.

\subsubsection{From counterfactuals to interventional distributions} 

To connect the counterfactual quantity \gls{PNS} with intervention
distributions, we generalize the classical identification results for
\gls{PNS} that have been developed for univariate binary causes and outcomes. It
turns out that the \gls{PNS} cannot be point-identified by
interventional distributions; for each set of interventional
distributions, there may exist many values of \gls{PNS} that are
consistent with the interventional distributions~\citep{Pearl2011}.
However, \gls{PNS} can be bounded by the difference between these
two interventional distributions.

\begin{lemma}[A lower bound on \gls{PNS}]
\label{thm:pns-id}
Assuming the causal graph in \Cref{fig:causalrep-scm}, the \gls{PNS}
is lower bounded by the difference between two intervention
distributions:
\begin{align}
\begin{split}
\gls{PNS}_{\mbZ=\mbz, Y=y}&=P(Y(\mbZ=\mbz) = y, Y(\mbZ\ne \mbz) \ne y)\\
&\qquad\geq P(Y=y\g \rmdo(\mbZ=\mbz)) - P(Y=y\g \rmdo(\mbZ\ne \mbz)).
\end{split}
\label{eq:pns-id}
\end{align} 
The inequality becomes an equality when the outcome $Y$ is monotone in
the representation $\mbZ$ (in the binary sense); i.e.,
$P(Y(\mbZ=\mbz)\ne y, Y(\mbZ\ne\mbz)=y)=0$.
\end{lemma}

\Cref{thm:pns-id} generalizes Theorem 9.2.10 of \citet{Pearl2011} to
non-binary $\mbZ$. (The proof is in \Cref{sec:thm-pns-id-proof}.) It
connects the counterfactual quantity \gls{PNS} to the intervention
distribution $P(Y\g \mathrm{do}(\mbZ=\mbz))$. An upper bound on
\gls{PNS} can be similarly obtained by generalizing the upper bound in
Theorem 9.2.10 of \citet{Pearl2011}. We focus on  \gls{PNS} here
because a larger lower bound on \gls{PNS} implies a large \gls{PNS}.
In contrast, a larger upper bound on \gls{PNS} does not necessarily
imply a large~\gls{PNS}.

Reducing the counterfactual quantity \gls{PNS} to the lower bound in
\Cref{thm:pns-id} has climbed down one level of the ladder of
causation~\citep{pearl2019seven}; we have converted a counterfactual
quantity (level three) to one with interventional distributions (level
two). Next we will descend one level further, discussing how to
identify the interventional distributions $P(Y\g
\rmdo(\mbZ=\mbz))$ from the observational data distribution $P(\mbX,
Y)$ (level one); $\mbZ=f(\mbX)$ is the representation of $\mbX$ via
some known function $f$. These identification results would also
enable the identification of $P(Y\g\rmdo(\mbZ\ne\mbz)) = \int P(Y\g
\rmdo(\mbZ))P(\mbZ\g \mbZ\ne \mbz)\dif \mbZ$.

\subsubsection{From interventional distributions to observational data
distributions}

To calculate the lower bound of \gls{PNS} in \Cref{eq:pns-id}, we need
to identify the intervention distribution $P(Y\g
\rmdo(\mbZ=\mbz))$---equivalently $P(Y\g \rmdo(f(\mbX)=\mbz))$---from
observational datasets $(\mbX, Y)$. Below we first state the formal
definition of functional interventions $\rmdo(f(\mbX)=\mbz)$. We then
discuss the challenges in identifying $P(Y\g \rmdo(f(\mbX)=\mbz))$ for
high-dimensional $\mbX$. To tackle this challenge, we further provide
an identification strategy and discuss its theoretical and practical
requirements.

\parhead{The functional intervention $\rmdo(f(\mbX))$.} We begin with
stating the formal definition of a functional intervention.

\begin{defn}[Functional interventions~\citep{puli2020causal}] 
\label{defn:functional-intervention}
The intervention distribution under a functional intervention
$P(Y\g\rmdo(f(\mbX)=\mbz))$ is defined as
\begin{align}
\label{eq:functional-intervention-def}
P(Y\g\rmdo(f(\mbX)=\mbz)) \triangleq \int P(Y\g \rmdo(\mbX), \mbC)
P(\mbX\g \mbC, f(\mbX)=\mbz) P(\mbC)\dif
\mbX\dif \mbC,
\end{align}
where $\mbC$ denotes all parents of $\mbX$.
\end{defn}

The functional intervention on $f(\mbX)$ is a soft intervention on
$\mbX$~\citep{correa2020calculus,eberhardt2007interventions}
conditional on all its parents $\mbC$, with a stochastic policy
$p(\mbX\g \mbC, f(\mbX)=\mbz)$~\citep[Ch. 4.2]{Pearl2011}. In other
words, a functional intervention $\rmdo(f(\mbX)=\mbz)$ considers all
interventions on $\mbX$ that are consistent with the functional
constraint $f(\mbX)=\mbz$ and its parental structure $\mbC$.
Functional interventions recover the standard backdoor adjustment as
special cases if we take the function $f$ to be an identity function
$f(\mbX)=\mbX$ or one that returns a subset $f(\mbX)=\mbX_S,
S\subset\{1,
\ldots, m\}$; see \Cref{sec:functional-interventions-supp} for
detailed derivations.

Following this definition, one can write the intervention distribution
of interest, $P(Y\g\rmdo(f(\mbX)=\mbz))$, as follows:
\begin{align}
\label{eq:functional-intervention-naive-id}
P(Y\g\rmdo(f(\mbX)=\mbz)) = \int P(Y\g \mbX) \cdot \left[\int
P(\mbX\g \mbC, f(\mbX)=\mbz) P(\mbC)\dif \mbC\right]\dif
\mbX.
\end{align}
This equality is due to the \gls{SCM} in \Cref{fig:causalrep-scm}:
there is no unobserved confounding between $\mbX$ and $Y$, which
implies $P(Y\g \rmdo(\mbX), \mbC) = P(Y\g \mbX)$.

\Cref{eq:functional-intervention-naive-id} provides a way to identify
$P(Y\g\rmdo(f(\mbX)=\mbz))$, provided that one can calculate $P(Y\g
\mbX)$ and $\int P(\mbX\g \mbC, f(\mbX)=\mbz) P(\mbC)\dif \mbC$. At
first sight, both quantities seem easy to calculate: one can estimate
$P(Y\g \mbX)$ from observational sampling of $P(Y, \mbX)$. To calculate
$\int P(\mbX\g
\mbC, f(\mbX)=\mbz) P(\mbC)\dif \mbC$, one may leverage probabilistic
factor models (e.g., \gls{PPCA}, \gls{GMM}, or \gls{VAE}) because the
latent $\mbC$ renders $X_1, \ldots, X_m$ conditionally independent in
\Cref{fig:causalrep-scm}. One can often read off $P(\mbX\g \mbC,
f(\mbX)=\mbz)$ and $P(\mbC)$ from the probabilistic factor model fit,
if the factor model is identifiable.

However, $P(Y\g \mbX)$ turns out to be challenging to calculate in
practice, especially when $\mbX$ represents high-dimensional objects
such as images or text, as we now discuss.

\parhead{The challenge of identifying $P(Y\g \mbX)$ with
high-dimensional $\mbX$.} The key challenge in identifying~$P(Y\g
\mbX)$ lies in the high-dimensionality of $\mbX$. High-dimensional
objects (e.g. images or text) are often only supported on a
low-dimensional manifold, i.e., $X_j = g_0(\{X_1, \ldots,
X_m\}\backslash X_j)$ for some $j \in \{1, \ldots, m\}$ and some
function $g$; see \Cref{fig:high-dim-data} for examples. We refer to
this problem as the \emph{rank-degeneracy
problem}~\citep{stewart1984rank,golub1976rank}. It is akin to the
classical challenge of having highly correlated variables in a linear
regression; it is also an example of the underspecification problem in
deep learning~\citep{damour2020underspecification}.

The rank degeneracy of high-dimensional $\mbX=(X_1, \ldots, X_m)$
challenges the identification of $P(Y\g \mbX)$, even if both $\mbX$
and $Y$ are observed. Intuitively, this non-identifiability problem
occurs because the data can only inform $P(Y\g \mbX)$ on the
low-dimensional manifold in $\mathcal{X}^m$ where $P(\mbX)$ is
supported. The behavior of $P(Y\g \mbX)$ outside this manifold is
unconstrained, hence non-identifiable.\footnote{This
non-identifiability of $P(Y\g \mbX)$ is also why, for high-dimensional
$\mbX$, directly fitting a neural network for $P(Y\g
\mbX)$ can pick up spurious correlations and fail in
out-of-distribution prediction. Similar failure can also occur with
linear regression when $\mbX$ is approximately low rank; e.g., when
different dimensions of $\mbX$ are highly correlated.}

\begin{figure}[t]
  \captionsetup[subfigure]{justification=centering}
\begin{subfigure}{0.32\textwidth}
  \centering
    \includegraphics[width=\linewidth]{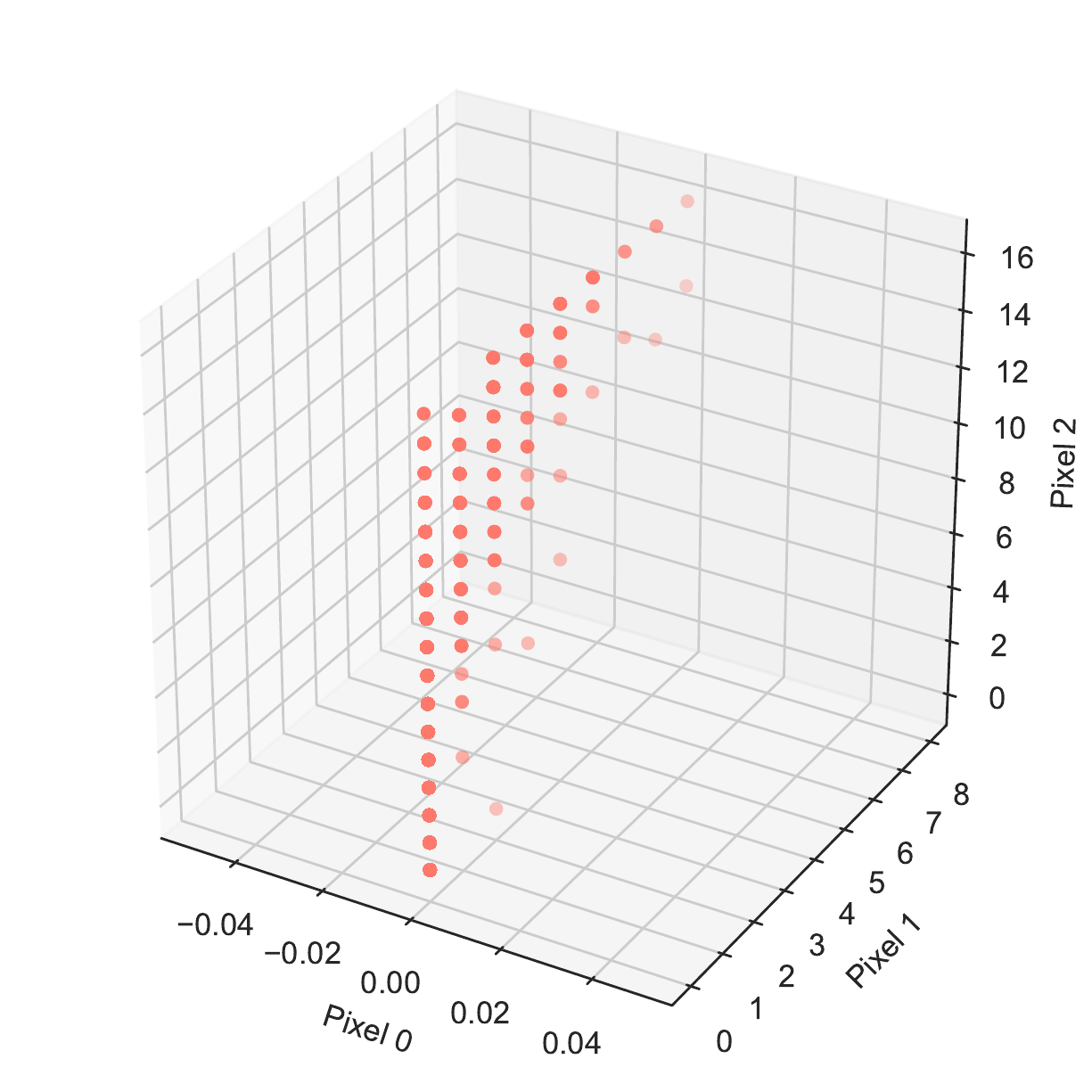}
  \centering 
  \caption{High-dim. image data: \\
  MNIST~\citep{deng2012mnist}\label{fig:high-dim-mnist}}
\end{subfigure}
\begin{subfigure}{0.32\textwidth}
  \centering
    \includegraphics[width=\linewidth]{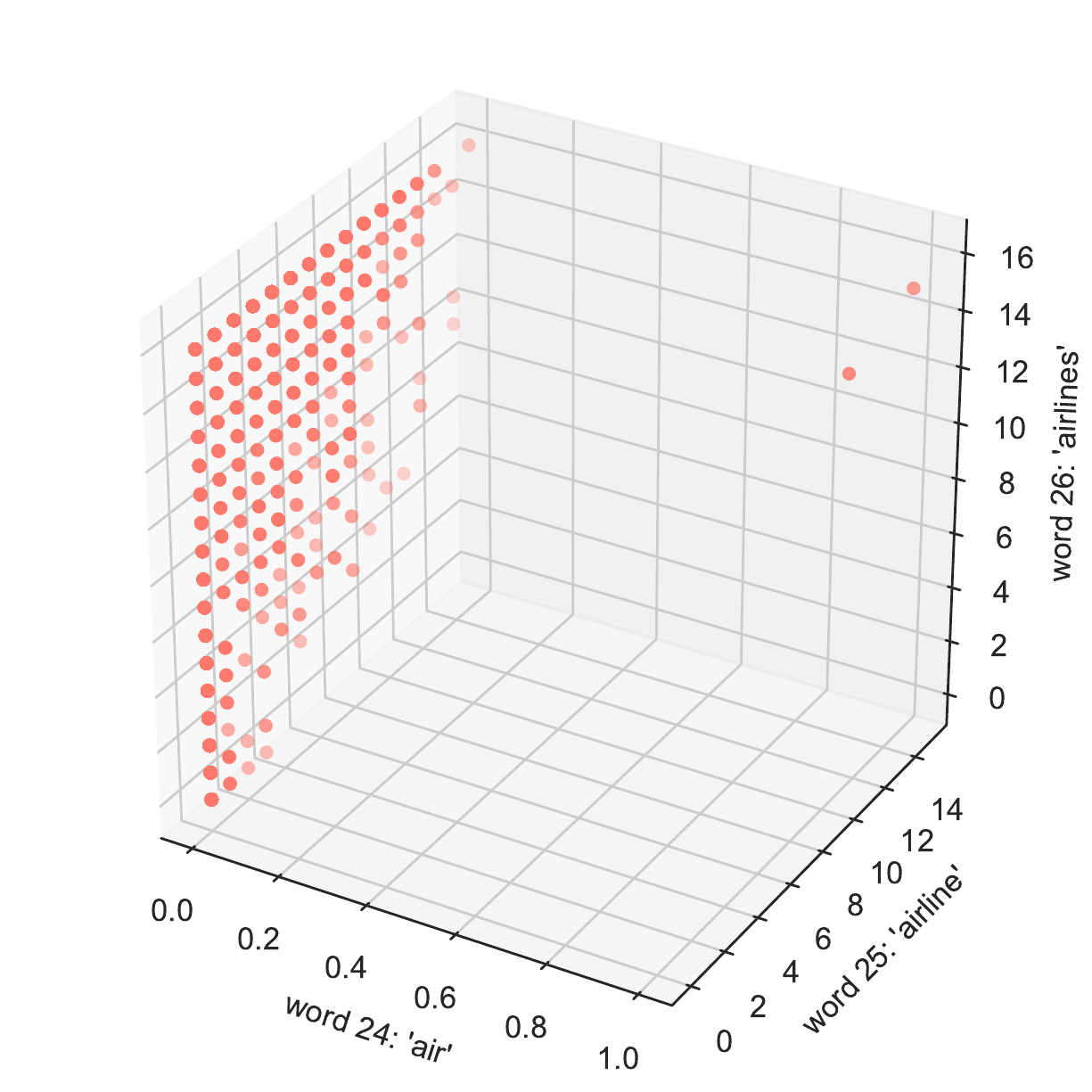}
  \centering 
  \caption{High-dim. text data: \\
  Airline tweets\label{fig:high-dim-tweet}}
  \end{subfigure}
\centering
\begin{subfigure}{0.32\textwidth}
  \centering
    \includegraphics[width=\linewidth]{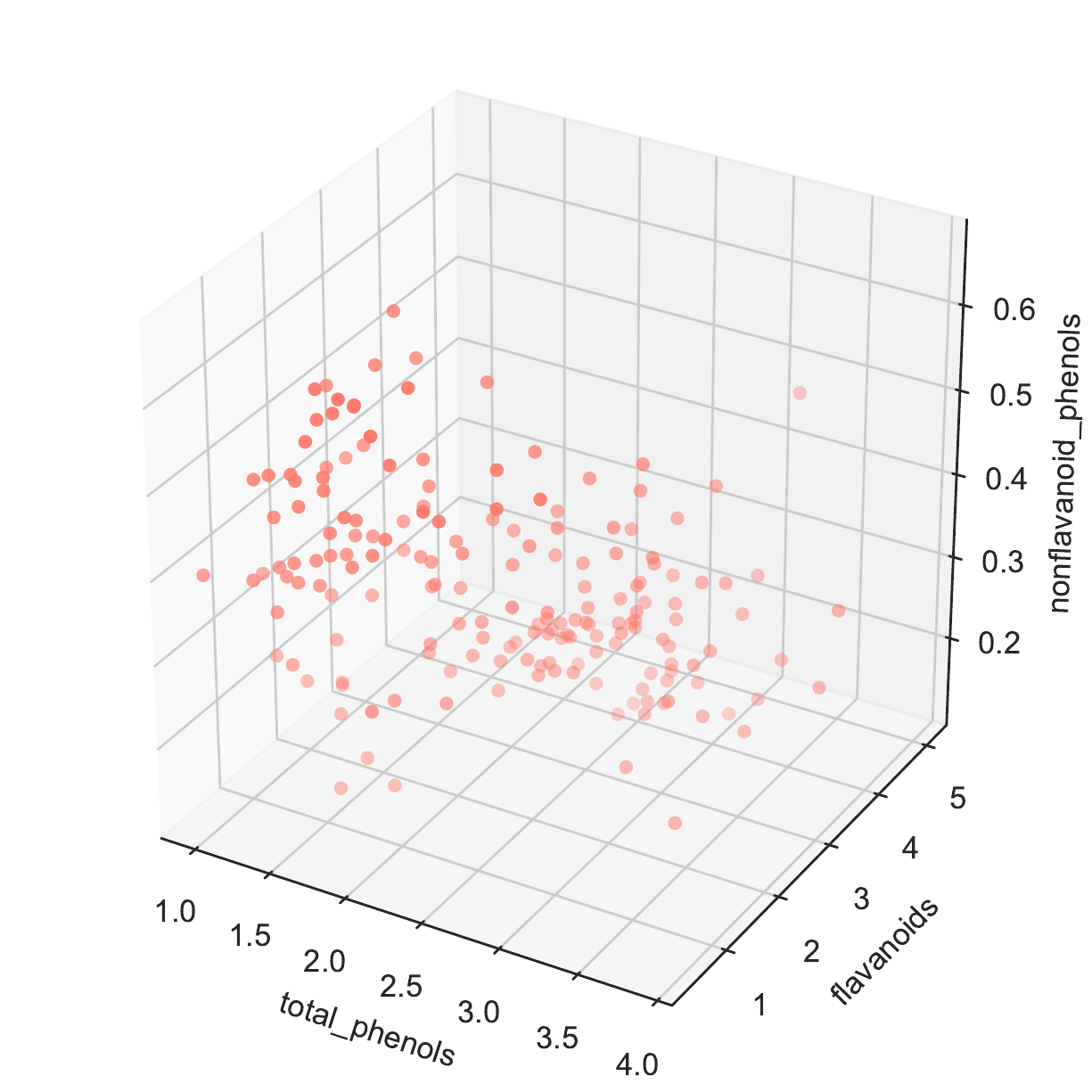}
  \centering
  \caption{Low-dim. data: \\ Wine
  features\label{fig:low-dim-wine}}
\end{subfigure}
\caption{High-dimensional data\protect\footnotemark  such as images
(\Cref{fig:high-dim-mnist}) or text~(\Cref{fig:high-dim-tweet}) are
often only supported on a low-dimensional manifold. Low-dimensional
data, e.g. the wine dataset~(\Cref{fig:low-dim-wine}), is often
supported in the whole space.}
\label{fig:high-dim-data}
\end{figure}

\footnotetext{https://www.kaggle.com/crowdflower/twitter-airline-sentiment, https://archive.ics.uci.edu/ml/datasets/wine}

In more detail, suppose $p(\mbx, y)$, $p(\mbx)$, $p(y\g \mbx)$ denote
the relevant densities (and assume they exist). Any $h_0(\cdot,
\cdot)$ function that satisfies
\begin{align}
\label{eq:density-eq}
h_0(\mbx, y)\cdot p(\mbx) = p(\mbx, y) \qquad \forall \mbx \in
\mathcal{X}^m
\end{align}
is a valid density for $p(y\g \mbx)$. Under rank degeneracy, it turns
out that there exist many different functions $h_0(\mbx, y)$ that
satisfy \Cref{eq:density-eq}. Hence $p(y\g \mbx)$ is non-identifiable.
The reason is that rank degeneracy implies $p(\mbx) = 0$ for $\mbx \in
\widetilde{\mathcal{X}}
\subseteq \mathcal{X}^m$, where $\widetilde{\mathcal{X}}$ is a set with
positive measure. If some function $h'_0$ satisfies
\Cref{eq:density-eq}, then the function $h''_0$ would also satisfy
\Cref{eq:density-eq} if they differ only on the set $(\mbx, y)
\in \widetilde{\mathcal{X}}\times \mathcal{Y}$. Note that $h'_0$ and
$h''_0$ are not almost surely equal; they differ on a set
$\widetilde{\mathcal{X}}$ with positive measure. Thus, $h'_0$ and
$h''_0$ are two different densities both valid for $p(y\g \mbx)$,
which implies the non-identifiability of $p(y\g \mbx)$.

As a more concrete example, consider a high-dimensional vector of
image pixels  $\mbX$ that lives on a low-dimensional manifold; i.e.,
such that $X_j - g_0(\{X_1, \ldots, X_m\}\backslash X_j)$ is
identically zero in the observational
data~\citep{kingma2014auto,goodfellow2014generative}. This rank
degeneracy implies that for any $p(y\g \mbx)=h_0(\mbx, y)$ compatible
with the observational data distribution, the conditional $p(y\g \mbx)
= h_0(\mbx, y) + \alpha\cdot (x_j - g_0(\{x_1, \ldots, x_m\}\backslash
x_j)), \forall \alpha\in \mathbb{R}$, is also compatible with the
observational data.

This fundamental non-identifiability of $P(Y\g \mbX)$ with
high-dimensional $\mbX$ suggests that we cannot hope to identify
functional interventions $P(Y\g\rmdo(f(\mbX)=\mbz))$ where the
function $f$ non-trivially depends on all $X_1, \ldots, X_m$, or any
of its subset that are also rank degenerate.

\glsreset{SCM}

\parhead{Causal identification of $P(Y\g \rmdo(f(\mbX))$ for a
restricted set of $f$.} Given the fundamental non-identifiability of
$P(Y\g \mbX)$ with high-dimensional $\mbX=(X_1, \ldots, X_m)$, we
restrict our attention to representations that only nontrivially
depends on a ``full-rank'' subset; i.e., $\mbZ=f(\mbX) =
\tilde{f}((X_j)_{j\in S})$, for some function
$\tilde{f}:\mathcal{X}^{|S|}\rightarrow \mathbb{R}^d$, and a set
$S\subseteq \{1, \ldots, m\}$, where $p((x_j)_{j\in S}) > 0$ for all
values $(x_j)_{j\in S}\in \mathcal{X}^{|S|}$. We term this requirement
``observability.''

Focusing on such representations $f(\mbX) =
\tilde{f}((X_j)_{j\in S})$, we calculate its intervention
distributions by returning to the definition of functional
interventions (\Cref{defn:functional-intervention}),
\begin{align}
\label{eq:subset-intervention-def}
P(Y\g \rmdo(f(\mbX)=\mbz)) = \int P(Y\g (X_j)_{j\in S}, \mbC) P((X_j)_{j\in S}\g
\mbC, f(\mbX)=\mbz) P(\mbC)\dif (X_j)_{j\in S} \dif \mbC.
\end{align}
\Cref{eq:subset-intervention-def} implies that $P(Y\g
\rmdo(f(\mbX)=\mbz))$ is identifiable as long as one can identify the
unobserved common cause $\mbC$ from data. This condition is often
called the ``pinpointability'' (or ``effective observability'') of
$\mbC$~\citep{wang2020towards,wang2019blessings}. The following lemma
summarizes the identification of $P(Y\g
\rmdo(f(\mbX)=\mbz))$ under these conditions.

\begin{lemma}[Identification of $P(Y\g \rmdo(f(\mbX)=\mbz)$]
\label{thm:do-prob-highdim}
Assume the causal graph in \Cref{fig:causalrep-scm}. Suppose the
representation only effectively depends on a subset $(X_j)_{j\in S}$
of $(X_1, \ldots, X_m)$; i.e., $f(\mbX) =
\tilde{f}((X_j)_{j\in S})$ for some function
$\tilde{f}:\mathcal{X}^{|S|}\rightarrow \mathbb{R}^d$ and some set
$S\subseteq\{1, \ldots, m\}$. Then the intervention distribution
$P(Y\g \rmdo(f(\mbX)=\mbz))$  is identifiable by
\begin{align} P(Y\g \rmdo(f(\mbX)=\mbz))=\int P(Y\g
f(\mbX)=\mbz, h(\mbX))\cdot P(h(\mbX))  \dif h(\mbX),
\label{eq:id-do-fx}
\end{align}
if the following conditions are satisfied:
\begin{enumerate}
\item (pinpointability) the unobserved common cause $\mbC$ is
pinpointable; i.e., $P(\mbC\g \mbX) = \delta_{h(\mbX)}$ for a
deterministic function $h$ known up to bijective transformations,
\item (positivity) $(X_j)_{j\in S}$ satisfies the positivity condition
given $\mbC$; i.e., $P((X_j)_{j\in S}\in \widetilde{\mathcal{X}}\g
\mbC) > 0$ for any set
$\widetilde{\mathcal{X}}\subset\mathcal{X}^{|S|}$ such that
$P((X_j)_{j\in S}\in
\widetilde{\mathcal{X}})>0$,
\item (observability) $P((X_j)_{j\in S} \in
\widetilde{\mathcal{X}})>0$ for all subsets
$\widetilde{\mathcal{X}}\subset\mathcal{X}^{|S|}$ with a positive
measure.
\end{enumerate}
\end{lemma}

\Cref{thm:do-prob-highdim} describes a set of representations
$\mbZ=f(\mbX)$ for which can we evaluate their intervention
distribution, and hence the lower bound of their \gls{PNS} for $Y$.
(The proof is in \Cref{sec:do-prob-highdim-proof}.)

\Cref{thm:do-prob-highdim} is based on three conditions. The
pinpointability ensures that the terms in
\Cref{eq:subset-intervention-def} involving $\mbC$ (e.g. $P(\mbC),
P((X_j)_{j\in S}\g \mbC, f(\mbX)=\mbz)$) are estimable from
observational data. We note that pinpointability serves a different
purpose here than in~\citet{wang2019blessings,wang2020towards} and
\citet{puli2020causal}. We invoke pinpointability for
\Cref{thm:do-prob-highdim} because functional interventions require
the knowledge of $\mbC$; there is no unobserved confounding issue. In
contrast, \citet{wang2019blessings,wang2020towards} and
\citet{puli2020causal} invoke pinpointability to handle unobserved
confounders.

Pinpointability is often approximately satisfied when the
high-dimensional $\mbX$ is driven by a low-dimensional factor $\mbC$.
In this case, $P(\mbC\g \mbX)$ is often close to a point
mass~\citep{chen2020structured}. For example, when the $m$-dimensional
data $\mbX$ is generated by a known $d$-dimensional probabilistic
factor model, then the latent factor $\mbC$ is approximately
pinpointable when $m \gg d$. In these probabilistic factor models,
$P(\mbC\g \mbX)$ becomes increasingly closer to a point mass as
$m\rightarrow \infty$ but~$d$
fixed~\citep{chen2020structured,bai2016maximum,wang2019frequentist}.
In practice, one assess pinpointability by fitting probabilistic
factor models to $\mbX$. We will discuss these practical details in
the next section.

The second condition of \Cref{thm:do-prob-highdim} is the positivity
of $(X_j)_{j\in S}$ given $\mbC=h(\mbX)$, also known as the overlap
condition~\citep{imbens2015causal}. This condition ensures that $P(Y\g
(X_j)_{j\in S}, \mbC)$ in \Cref{eq:subset-intervention-def} is
estimable from observational data. Positivity loosely requires that
all values of $(X_j)_{j\in S}$ are possible conditional on
$\mbC=h(\mbX)$. For example, it is violated when $S=\{1, \ldots, n\}$
because it is impossible to observe $\mbX=\mbx$ and $\mbC=\mbc$
simultaneously when $h(\mbx)\ne \mbc$. In practice, positivity is more
likely to be satisfied when $(X_j)_{j\in S}$ is a low-dimensional
vector; e.g., the image representation only depends on a few image
pixels that are far away from each other.

Finally, the observability condition requires that it must be possible
to observe all possible values of~$(X_j)_{j\in S}$. Together with the
positivity condition, this implies that $P((X_j)_{j\in S}\g \mbC)>0$
for all~$(X_j)_{j\in S}$ and, therefore, $P(Y\g (X_j)_{j\in S}, \mbC)$
in \Cref{eq:id-do-fx} is estimable from observational data. This
condition is violated if the observed $(X_j)_{j\in S}$ is
rank-degenerate; e.g., $(X_j)_{j\in S}$ with $S=\{1, \ldots, m\}$
represents the whole high-dimensional image vector supported only on a
low-dimensional manifold. This rank degeneracy would render
$P(Y\g(X_j)_{j\in S}, C)$ non-identifiable, for the same reason as why
$P(Y\g \mbX)$ is non-identifiable with rank-degenerate $\mbX$. In
practice, this observability condition is more likely to hold when
$(X_j)_{j\in S}$ is low-dimensional.

\Cref{thm:do-prob-highdim} describes a class of representations
$\mbZ=f(\mbX)$ whose intervention distributions are identifiable from observational data. Combining
\Cref{thm:do-prob-highdim,thm:pns-id} allows us to lower bound the
\gls{PNS} of these representations and evaluate their efficiency and
non-spuriousness.

\begin{thm}[Evaluating efficiency and non-spuriousness with
observational data]
\label{corollary:pns-final}
Under the assumptions of \Cref{thm:do-prob-highdim}, the efficiency
and non-spuriousness of the representation
$\mbZ=f(\mbX)=\tilde{f}((X_j)_{j\in S})$ in the dataset $\{\mbx_i,
y_i\}_{i=1}^n$ is lower bounded by
\begin{align}
\gls{PNS}_n(f(\mbX), Y)&\geq \underline{\gls{PNS}_n(f(\mbX), Y)} \nonumber\\
&\triangleq\prod_{i=1}^n \int
\left[P(Y=y_i \g f(\mbX) = f(\mbx_i), \mbC)\right.\label{eq:pns-final}\\
&\qquad\left.- P(Y=y_i \g f(\mbX) \ne
f(\mbx_i), \mbC)
\right]\cdot P(\mbC) \dif \mbC,\nonumber
\end{align}
where $\mbC=h(\mbX)$ is the unobserved common cause pinpointed by the
observational data $\mbX$ in \Cref{thm:do-prob-highdim}. Similarly,
the conditional efficiency and non-spuriousness of the $j$th dimension
of $(f_1(\mbX), \ldots, f_d(\mbX))$ is lower bounded by
\begin{align}
\gls{PNS}_n(f_j(\mbX), Y \g f_{-j}(\mbX))&\geq \underline{\gls{PNS}_n(f_j(\mbX), Y \g f_{-j}(\mbX))} \nonumber\\
&\triangleq \prod_{i=1}^n \int
\left[P(Y=y_i \g f_j(\mbX) = f_j(\mbx_i), f_{-j}(\mbX) =
f_{-j}(\mbx_i), \mbC)\right.\label{eq:pns-final-conditional}\\
&\qquad\left.- P(Y=y_i \g f_j(\mbX) \ne f_j(\mbx_i),
f_{-j}(\mbX) = f_{-j}(\mbx_i), \mbC)\right]\cdot P(\mbC) \dif \mbC.\nonumber
\end{align}
\end{thm}

\Cref{corollary:pns-final} is an immediate consequence of
\Cref{thm:do-prob-highdim,thm:pns-id}. It suggests that we evaluate
efficiency and non-spuriousness using
\Cref{eq:pns-final,eq:pns-final-conditional} in practice. We
operationalize this result in the next section.

\subsubsection{Measuring efficiency and non-spuriousness in practice}

\label{subsubsec:pns_practical}

We operationalize \Cref{corollary:pns-final} to measure efficiency and
non-spuriousness for a representation~$f(\mbX)$. The full algorithm is
in \Cref{alg:calculate-pns}. The algorithm involves three steps: (1)
Pinpoint the unobserved common cause $\mbC$; (2) Estimate the
conditional $P(Y\g f(\mbX), \mbC)$; (3) Calculate the lower bounds in
\Cref{corollary:pns-final}
(\Cref{eq:pns-final,eq:pns-final-conditional}) to measure efficiency
and non-spuriousness. We discuss each step below.

\parhead{Pinpointing the unobserved common cause $\mbC$.}
\Cref{corollary:pns-final} requires that the unobserved common cause
$\mbC$ must be pinpointable by the observational data: $P(\mbC\g
\mbX)=\delta_{h(\mbX)}$ for some deterministic function~$h$. Moreover,
this pinpointing function $h$ needs to be known up to bijective
transformations. How can we assess pinpointability and extract $h$ in
practice?

One can assess pinpointability by fitting a probabilistic factor
model, viewed as a causal generative model; e.g.,
\glsreset{PPCA}\gls{PPCA}~\citep{tipping1999probabilistic},
\glsreset{VAE}\gls{VAE}~\citep{kingma2014auto,rezende2014stochastic},
mixture models~\citep{mclachlan1988mixture}, or mixed membership
models~\citep{pritchard2000inference,blei2003latent,airoldi2008mixed,erosheva2005}.

Specifically, suppose the dataset of $n$ i.i.d. data points
$(\mbx_i)_{i=1}^n$ is assumed to be generated by a PPCA model, then
one can assess pinpointability by first fitting a \gls{PPCA} model,
\begin{align}
\label{eq:ppca}
\mbc_{i}  &\stackrel{i.i.d.}{\sim} \cN(0, I_K), &i=1, \ldots, n, \\
x_{il}\g \mbc_i & \stackrel{i.i.d.}{\sim} \cN(\mbc_i^\top \theta_l, \sigma^2), &l=1, \ldots, m,
\end{align}
where $\mbc_i$ is the latent variable for each data point
$\mbx_i=(x_{i1}, \ldots, x_{im})$, and $\boldsymbol{\theta}=(\theta_1,
\ldots, \theta_m)$ is the set of parameters. One then infers the
parameters $\boldsymbol{\theta}$ and the posterior $p(\mbc_i\g
\mbx_i)$ for each $i$, using standard posterior inference algorithms
like variational inference~\citep{blei2017variational} or Markov chain
Monte Carlo methods~\citep{gilks1995markov}. Pinpointability then
approximately holds if $p(\mbc_i\g \mbx_i)$ is close to a point mass
for all $i$'s.

The choice of the latent dimensionality $K$ is a challenging task. We
choose $K$ in practice by searching for a factor model that can both
fit data well (e.g., via a posterior predictive
check~\citep{gelman1996posterior}) and approximately satisfy
pinpointability. A larger $d$ tends to have a better fit to the data,
while a smaller $d$ is more likely to satisfy pinpointability. We
proceed if pinpointability holds, along with the other
conditions---observability and positivity--from
\Cref{corollary:pns-final}.

Given a fitted factor model, one can extract the pinpointing function
$h$ by computing the posterior mean, $h(\mbx) \approx \E{}{\mbc_i \g
\mbx_i=\mbx}$. The same procedure applies if the data is generated by
other probabilistic factor models.

A reader might ask: Why probabilistic factor models? The reason is
that the causal graph \Cref{fig:causalrep-scm} encodes a conditional
independence structure that coincides with the defining feature of
probabilistic factor models. In more detail, \Cref{fig:causalrep-scm}
assumes that each dimension of the data $\mbX=(X_1, \ldots, X_m)$ is
rendered conditionally independent given an unobserved common cause
$\mbC$. This conditional independence is precisely the defining
structure of probabilistic factor models,
\begin{align}
\label{eq:factor-1}
\mbc_i&\sim p(\cdot\s \lambda_\mbC),&i=1, \ldots, n,\\
\label{eq:factor-2}
x_{il}\g \mbc_i &\sim p(\cdot \g \mbc_i\s \theta_{l}), &l= 1, \ldots, m,
\end{align}
with i.i.d. data $\mbx_i = (x_{i1}, \ldots, x_{im})$ and latent
$\mbc_i$. We therefore use probabilistic factor models to assess the
pinpointability of $\mbC$ in \Cref{fig:causalrep-scm}.

\parhead{Calculate the conditional $P(Y\g f(\mbX), \mbC)$.} After
pinpointing $\mbC$, we still need the conditional $P(Y\g f(\mbX),
\mbC)$ to calculate the lower bound in \Cref{corollary:pns-final}. 

To estimate $P(Y\g f(\mbX), \mbC)$, we can fit a model to the
observational data $(\mbx_i, y_i)_{i=1}^n$, with $f(\cdot)$ being the
representation function of interest and $\mbc_i=h(\mbx_i)$ estimated
from the first step. As an example, we may posit a linear model,
\begin{align}
\label{eq:y-model-C}
P(Y\g f(\mbX), \mbC) = \cN((\beta_0 + \boldsymbol{\beta}^\top f(\mbX)  +
\boldsymbol{\gamma}^\top \mbC), \sigma^2),
\end{align} 
and estimate the parameters $\boldsymbol{\beta}=(\beta_1, \ldots,
\beta_d)$ and $\boldsymbol{\gamma}$ via maximum likelihood. One may
also posit other more flexible models like a partially linear model
(with a nonlinear $f_C$) and target a categorical outcome:
\begin{align}
\label{eq:softmax-y-model-C}
P(Y\g f(\mbX), \mbC) = \mathrm{Categorial}(\mathrm{softmax}(\beta_0 +
\boldsymbol{\beta}^\top f(\mbX)  + \gamma^\top f_C(\mbC)).
\end{align}

\parhead{Calculate the lower bounds in \Cref{corollary:pns-final}.} We
finally calculate the lower bounds in
\Cref{eq:pns-final,eq:pns-final-conditional} using the pinpointed
$\mbC$ and the estimated conditional $P(Y\g f(\mbX),
\mbC)$ from the previous two steps.

As an example, we calculate the lower bound of the conditional
efficiency and non-spuriousness $\underline{\gls{PNS}_n(f_j(\mbX), Y
\g f_{-j}(\mbX))}$ with the linear model (see \Cref{sec:calculate-pns-linear} for the detailed derivation):
\begin{align}
&\log \underline{\gls{PNS}_n(f_j(\mbX), Y \g f_{-j}(\mbX))}\nonumber\\
&\approx \left(\frac{1}{2\sigma^2}\left[
\sum_{i=1}^n(\beta_j\cdot (f_j(\mbx_i) -\E{}{f_j(\mbx_i)}))^2\right.\right.\nonumber\\
&\left.\left.\qquad + 2\cdot\sum_{i=1}^n
\beta_j\cdot (f_j(\mbx_i) -\E{}{f_j(\mbx_i)}) \cdot  \gamma^\top
(\mbc_i-\E{}{\mbc_i})\right]\right)+\mathrm{constant}.\label{eq:pns-conditional-lower-last-step}
\end{align}
Similarly, we can obtain the lower bound of the (unconditional)
efficiency and non-spuriousness,
\begin{align}
&\log \underline{\gls{PNS}_n(f(\mbX), Y)}\nonumber\\
&\approx \left(\frac{1}{2\sigma^2}\sum_{i=1}^n\left[ (\sum_{j=1}^d\beta_j\cdot (f_j(\mbx_i) -\E{}{f_j(\mbx_i)}))^2 \right.\right.\nonumber\\
&\left.\left.\qquad+ 2\cdot \sum_{j=1}^d\beta_j\cdot (f_j(\mbx_i) -\E{}{f_j(\mbx_i)}) \cdot  \gamma^\top (\mbc_i-\E{}{\mbc_i})\right]\right)+\mathrm{constant}.\label{eq:pns-lower-last-step}
\end{align}

\Cref{eq:pns-conditional-lower-last-step,eq:pns-lower-last-step}
illustrate how conditional efficiency and non-spuriousness differ from
the unconditional version. The conditional notion considers each
$\beta_j \cdot f_j(\mbX)$ at a time. The unconditional notion lumps
together all $d$ dimensions of the representation and considers
$\sum_{j=1}^d
\beta_j\cdot f_j(\mbX)$ as a whole. In this sense, the conditional
\gls{PNS} is finer-grained notion than the (unconditional)~\gls{PNS}.


\begin{algorithm}[t]
\small{
\SetKwData{Left}{left}\SetKwData{This}{this}\SetKwData{Up}{up}
\SetKwFunction{Union}{Union}\SetKwFunction{FindCompress}{FindCompress}
\SetKwInOut{Input}{input}\SetKwInOut{Output}{output}
\Input{The observational data and its label $\{\mbx_i,y_i\}_{i=1}^n$;
representation function $f(\cdot)$}; the probabilistic factor model
that generates the data $P(\mbX, \mbC)$.

\Output{The (lower bound of) efficiency and non-spuriousness of representation
$f(\mbX)$ for label $Y$}
\BlankLine
Fit a probabilistic factor model (\Cref{eq:factor-1,eq:factor-2}) and
infer $p(\mbc_i\g \mbx_i)$;

\If{Pinpointability holds, i.e., $p(\mbc_i\g \mbx_i)$ is close to a
point mass for all $i$}{
\If{Observability and positivity (\Cref{corollary:pns-final}) hold}{
	\ForEach{datapoint $i$}{
	Pinpoint the unobserved common cause: $\hat{\mbc}_i=h(\mbx_i)
	\approx
	\E{}{\mbc_i\g \mbx_i}$;}

	Calculate the conditional $P(Y\g f(\mbX), \mbC)$ by fitting a
	model to $\{\mbx_i,\hat{\mbc}_i, y_i\}_{i=1}^n$ (e.g.,
	\Cref{eq:y-model-C,eq:softmax-y-model-C});}

	Calculate the (lower bound of) efficiency and non-spuriousness
	$\underline{\gls{PNS}_n(f(\mbX), Y)}$ and
	$\underline{\gls{PNS}_n(f_j(\mbX), Y \g f_{-j}(\mbX))}, j=1,
	\ldots, d$ (\Cref{eq:pns-final,eq:pns-final-conditional}); }

\caption{Calculating the  (lower bound of) efficiency and non-spuriousness of a representation}
\label{alg:calculate-pns}}
\end{algorithm}

\parhead{A Linear example.} We illustrate \Cref{alg:calculate-pns} on
a toy rank-degenerate image dataset. (We analyze more complex and
higher-dimensional data in \Cref{sec:causalrep-empirical}.) Imagine we
collect a dataset of $n=1000$ images. Each image is characterized by
its values on $m=5$ chosen representative pixels, $\mbX=(X_1, \ldots,
X_5)$, accompanied by a label about the image brightness $Y$. Both the
pixel values $\mbX$ and the labels $Y$ are real-valued.

We simulate such a dataset of image pixels and labels. As the pixel
values are often highly correlated, we generate $\mbx_i=(x_{i1},
\ldots, x_{i5})$ from a multivariate Gaussian distribution with strong
(positive or negative) correlations---all pairwise correlations are >
0.8. As the brightness label only depends on a small number of pixels,
we simulate $y_i$ from a linear model that only uses two of the five
pixels
\begin{align} 
y_i = \beta^*_1 x_{i1} + \beta^*_2 x_{i2} + \epsilon_i, \qquad
\epsilon_i\sim \cN(0, \sigma^2),
\end{align}
where $\beta^*_1=0.5, \beta^*_2=1.0$. 

\sloppy
The goal is to evaluate the efficiency and non-spuriousness of a
representation $f(\mbX) = (X_2, 0.5\cdot X_1 + X_4)$. We apply
\Cref{alg:calculate-pns}: pinpoint the unobserved common cause $\mbC$,
estimate $P(Y\g f(\mbX), \mbC)$,~and finally calculate the lower bound
of efficiency and non-spuriousness $\underline{\gls{PNS}_n(f(\mbX),
Y)}$ and $\underline{\gls{PNS}_n(f_j(\mbX), Y \g f_{-j}(\mbX))}, j=1,
\ldots, d$.

We first pinpoint $\mbC$. Suppose it is known that the data $\mbX$ is
generated by a \gls{PPCA}. We fit a one-dimensional
\gls{PPCA} (\Cref{eq:ppca}) to $\{\mbx_i\}_{i=1}^n$; the latent
variable $c_i$ is thus a scalar. The fit leads to the posterior of
the unobserved common cause, $p(c_i\g
\mbx_i)$. We assess the pinpointability of $\mbC$ by calculating
$\mathrm{Var}(c_i\g \mbx_i)$ for all $i$. We find the variance is
smaller than 0.01 for all $i$, implying that $p(c_i\g
\mbx_i)$ is fairly close to a point mass and pinpointability is
approximately satisfied. (The threshold of 0.01 is a subjective
choice.) We then calculate $\hat{c}_i=\E{}{c_i\g \mbx_i}$ for
all $i$.

We next estimate $P(Y\g f(\mbX), \mbC)$ by fitting a linear model to
the dataset $\{(y_i, f(\mbx_i),
\hat{c}_i)\}_{i=1}^n$ with $f(\mbx_i) = (x_{i2}, 0.5\cdot x_{i1} +
x_{i4})$,
\begin{align}
y_i \sim \cN(\beta_0 + \beta_1 \cdot x_{i2} + \beta_2 \cdot (0.5\cdot
x_{i1} + x_{i4}) + \gamma\cdot \hat{c}_i, \sigma^2).
\end{align}
Fitting this linear model returns the regression coefficients
$\{\hat{\beta}_0, \hat{\beta}_1, \hat{\beta}_2, \hat{\gamma}\}$.

Finally, plugging in these regression coefficients to
\Cref{eq:pns-lower-last-step,eq:pns-conditional-lower-last-step} gives
the efficiency and non-spuriousness, with $f_1(\mbx_i) = x_{i2}$ and
$f_2(\mbx_i) = 0.5\cdot x_{i1} + x_{i4}$.

\subsection{CAUSAL-REP: Learning Efficient and Non-spurious
Representations}

\label{subsec:pns-learn}

We have developed a strategy to evaluate the efficiency and
non-spuriousness of a given representation in \Cref{subsec:pns-id}.
Here we utilize this strategy to improve representation learning in
both supervised and unsupervised settings.

\subsubsection{Representation learning as finding necessary and
sufficient causes}

As representation learning requires efficient and non-spurious
representations, we formulate representation learning as a task of
\emph{finding necessary and sufficient causes} of the label.

\parhead{Representation learning as finding necessary and sufficient
causes.} To operationalize this formulation, we search for a
representation that maximizes the \emph{conditional} efficiency and
non-spuriousness for a given dataset, following \Cref{defn:pns} and
\Cref{eq:pns-conditional-dataset}. Thus we view an ideal
representation as one in which each dimension of the representation
captures features that are essential and non-spurious given all other
dimensions.

We perform representation learning by maximizing the sum of log
\gls{PNS} across all dimensions of the representation,
\begin{align}
\label{eq:max-pns}
\max_{f} \sum_{j=1}^d \log \gls{PNS}_n(f_j(\mbX), Y \g f_{-j}(\mbX)),
\end{align}
where $\gls{PNS}_n(\cdot, Y\g \cdot)$ measures the conditional
efficiency and non-spuriousness of $f_j(\mbX)$ as in
\Cref{eq:pns-conditional-dataset}. 

\parhead{Classes of representation functions.} To perform this
maximization over representations in practice, we consider
parameterized classes of representation functions $f(\cdot)$. One
option is to consider a class of neural network functions with a fixed
architecture, subject to the constraint that each output dimension has
zero mean and unit standard deviation, e.g.,
\begin{multline}
\label{eq:nn-function}
\{f: \text{multilayer perceptrons
$\mathbb{R}^m\rightarrow\mathbb{R}^d$ with two hidden layers of size
512 }\\
\text{s.t. $\{f(\mbx_i)\}_{i=1}^n$ has sample mean 0 and standard error 1}\}.
\end{multline}

Another option is to consider representations that are convex
combinations of the $m$-dimensional data,
\begin{align}
\label{eq:convex-function}
\{f: f(\mbX) = \mbX_{1\times m} \mbW_{m\times d}, \sum_{l=1}^m W_{lj}
= 1, W_{lj}\geq 0, \forall l\in \{1, \ldots, m\}, j\in\{1, \ldots,
d\}\}.
\end{align}
Each dimension of such representations is a convex combination of
$X_1, \ldots, X_m$, that is, $f_j(\mbX) = \sum_{l=1}^m W_{lj} X_l$,
$j=1,\ldots, d$.

A third option is to consider representations that select relevant
subsets of the $m$-dimensional data,
\begin{align}
\label{eq:binary-function}
\{f: f(\mbX) = \mbX_{1\times m} \mbW_{m\times d}, W_{lj} \in \{0,1\},
\sum_{l=1}^m W_{lj} \leq 1, \forall l\in \{1, \ldots, m\}, j\in\{1,
\ldots, d\}\}.
\end{align}
Each dimension of such representations selects one (or none) of $X_1,
\ldots, X_m$. The $d$-dimensional representation $f(\mbX)$ thus
selects at most $d$ features from $X_1, \ldots, X_m$.

\parhead{CAUSAL-REP: Maximizing \gls{PNS} in practice.} Solving
\Cref{eq:max-pns} involves calculating counterfactual quantities such as
conditional \gls{PNS}. As is discussed in \Cref{subsec:pns-id}, these
quantities are not directly calculable without access to the causal
structural equations.

To employ \Cref{eq:max-pns} in practice, we employ the lower bound of
\gls{PNS} derived in \Cref{eq:pns-id}. Specifically, we restrict our
attention to representations whose lower bound of
\gls{PNS} is \emph{identifiable} via \Cref{corollary:pns-final}, and
find the representation that maximizes this lower bound:
\begin{align}
\label{eq:causalrep-final-obb}
\textbf{(CAUSAL-REP objective)  }\quad \max_{f} \sum_{j=1}^d \log
\underline{\gls{PNS}_n(f_j(\mbX), Y \g f_{-j}(\mbX))} + \lambda\cdot
R(f\s \{\mbx_i, y_i, \mbc_i\}_{i=1}^n),
\end{align}
where $\underline{\gls{PNS}_n(f_j(\mbX), Y \g f_{-j}(\mbX))}$ is the
\gls{PNS} lower bound (\Cref{eq:pns-final-conditional}) in
\Cref{corollary:pns-final}. The parameter $\lambda\geq 0$ indicates
regularization strength. The term $R(f\s \{\mbx_i, y_i,
c_i\}_{i=1}^n)$ is a regularization penalty,
\begin{align}
\label{eq:causalrep-reg}
R(f\s \{\mbx_i, y_i, \mbc_i\}_{i=1}^n)\triangleq
\frac{1}{d}\sum_{j=1}^d\log (1-\mathrm{Rsq}(\{f_j(\mbx_i)\s
\mbc_i\}_{i=1}^n)) - \alpha\cdot
\left|\left|W(f)\right|\right|_2^2,
\end{align}
where $\mathrm{Rsq}(\{f_j(\mbx_i)\s \mbc_i\}_{i=1}^n)$ is the
R-squared of regressing the $j$th-dimension of the representation
$f_j(\mbx_i)$ against the unobserved common cause $\mbc_i$. The
quantity $||W(f)||_2$ represents the $L_2$-norm of the $f$ function's
parameters, e.g., all the weight parameters of a neural network $f$;
$\alpha$ is the relative weight of the two regularization penalties.

The first part of the objective (\Cref{eq:causalrep-final-obb}) is the
lower bound of conditional \gls{PNS} developed in
\Cref{corollary:pns-final}. Following the practical discussions in
\Cref{subsec:pns-id}, its calculation requires pinpointing the
unobserved common cause $\mbC=h(\mbX)$ and
fitting a model to $P(Y\g f(\mbX), \mbC)$ (e.g., \Cref{eq:y-model-C}).

The second part of the objective is the regularization penalty. It
encourages representations whose lower bound of \gls{PNS} is
identifiable from observational data. Specifically, these
regularization penalties aim to enforce the positivity and
observability conditions in \Cref{thm:do-prob-highdim}. (The
pinpointability condition is enforced in a separate step.)

In more detail, the first penalty $\frac{1}{d}\sum_{j=1}^d\log
(1-\mathrm{Rsq}(\{f_j(\mbx_i)\s \mbc_i\}_{i=1}^n))$ in
\Cref{eq:causalrep-reg} encourages the positivity of $f(\mbX)$ given
$\mbC$ in \Cref{thm:do-prob-highdim}. The sample R-squared
$\mathrm{Rsq}(\{f_j(\mbx_{i})\s\mbc_i\}_{i=1}^n)$ evaluates the
variation of $f_j(\mbX)$ explainable by $\mbC$. When the R-squared is
equal to one, then the positivity condition is violated. Accordingly,
a close-to-one value of R-squared implies that $f_j(\mbX)$ nearly violates the
positivity condition. The penalty $\frac{1}{d}\sum_{j=1}^d\log
(1-\mathrm{Rsq}(\{f_j(\mbx_i)\s
\mbc_i\}_{i=1}^n))$ takes a large negative value if any dimension of
the representation $f(\mbX)$ nearly violates the positivity condition.

The second penalty $-\alpha\cdot ||W(f)||_2^2$ encourages
representations that satisfy the observability condition.
Specifically, it penalizes the coefficients in front of $(X_1, \ldots,
X_m)$ in the representation function $f$. Imposing a large
regularization parameter $\alpha>0$ leads to representations $f(\mbX)$
that only effectively depend on a small subset of $(X_1,
\ldots, X_m)$; i.e., $f(\mbX)=\tilde{f}((X_j)_{j\in S})$ for some
function $\tilde{f}$, with $S\subseteq \{1,\ldots, m\}$ being a set
with only a few elements. Such representations are more likely to
satisfy the observability condition---namely, $P((X_j)_{j\in S})>0$
for all values of $(X_j)_{j\in S}$---because lower-dimensional
$(X_j)_{j\in S}$ are more likely to have full rank~\citep{udell2019}.

\parhead{Out-of-distribution prediction with CAUSAL-REP.} Given a
representation returned by CAUSAL-REP, how do we make predictions,
especially on out-of-distribution data?

To make predictions, we train a prediction function that maps the
CAUSAL-REP representation to the labels. We note that the predictions
are only made using the CAUSAL-REP representation; it does not involve
the unobserved common cause $\mbC$. This prediction function is
different from the prediction model fitted for $P(Y\g f(\mbX), \mbC)$
(cf.\ \Cref{eq:y-model-C,eq:softmax-y-model-C}). The rationale is that
CAUSAL-REP encourages non-spurious representations; it implies that
the relationship between the representation $f(\mbX)$ and the label
$Y$ should generalize to out-of-distribution data. However, CAUSAL-REP
places no constraints on the relationship between $\mbC$ and $Y$.

Specifically, suppose CAUSAL-REP returns $\hat{f}(\cdot)$ as the
representation function; i.e., it maximizes
\Cref{eq:causalrep-final-obb}. Then we posit a model for $P(Y\g
\hat{f}(\mbX))$; e.g., a linear model:
\begin{align}
\label{eq:pred-linear}
P(Y\g \hat{f}(\mbX)) = \cN(\beta_0^{\mathrm{pred}} + (\boldsymbol{\beta}^{\mathrm{pred}})^\top\hat{f}(\mbX), (\sigma^{\mathrm{pred}})^2),
\end{align}
or a flexible exponential family model:
\begin{align}
\label{eq:pred-nonlinear}
P(Y\g \hat{f}(\mbX)) = \mathrm{EF}(g^{\mathrm{pred}}(\hat{f}(\mbX))),
\end{align}
where $\mathrm{EF}$ indicates an exponential family distribution and
$g^{\mathrm{pred}}$ indicates a link function. We fit this model
to the training data $\{(\hat{f}(\mbx_i), y_i)_{i=1}^n\}$ using
maximum likelihood estimation, and then use this fitted model to
predict on out-of-distribution data; e.g., we calculate $\E{}{(Y\g
\hat{f}(\mbX_{\mathrm{test}}))}.$

\parhead{CAUSAL-REP with perfectly correlated spurious and
non-spurious features.} We have described CAUSAL-REP, a representation
learning algorithm that targets non-spurious and efficient
representations. A reader might ask: What if spurious and non-spurious
features are perfectly correlated in the training dataset? That is,
what if the spurious feature is on if and only if the non-spurious
feature is also on? Intuitively, we can not hope to tease apart
spurious and non-spurious representations in this case. How would
CAUSAL-REP fare?

In such a perfect correlation case, the CAUSAL-REP algorithm would
capture neither the spurious feature nor the non-spurious one. The
reason is that both features are excluded from the feasible set of
representations considered by CAUSAL-REP. Their \gls{PNS} lower bounds
are both not identifiable, but CAUSAL-REP only considers
representations whose \gls{PNS} lower bound is identifiable.

In more detail, suppose $f_s((X_j)_{j\in T_s})$ captures the spurious
feature and $f_n((X_j)_{j\in T_n})$ captures the non-spurious
features, where $T_s, T_n \subset \{1, \ldots, m\}$. They depend on
disjoint subsets of $\mbX$, $T_s \cap T_n = \emptyset$, but they are
perfectly correlated $f_s((X_j)_{j\in T_s}) = f_n((X_j)_{j\in T_n})$.
Then both features must be measurable with respect to the unobserved
common cause $\mbC$: $f_s((X_j)_{j\in T_s}) = f_n((X_j)_{j\in T_n}) =
q(\mbC)$ for some function $q$.\footnote{The reason is that $\mbC$
must render all $X_j$'s conditionally independent due to the causal
graph (\Cref{fig:causalrep-scm}). It must also render the spurious and
non-spurious features conditionally independent, because they depend
on disjoint subsets of $\mbX$.  As the two features are perfectly
correlated, the only way to render them conditionally independent is
to make them measurable with respect to $\mbC$. } This measurability
makes their \gls{PNS} lower bounds non-identifiable; it violates the
positivity assumption required by \Cref{corollary:pns-final}.

\subsubsection{CAUSAL-REP and the linear example continued}

We described each step of CAUSAL-REP in the previous section.
\Cref{alg:causal-rep} summarizes these steps. The key step is to
maximize the CAUSAL-REP objective (\Cref{eq:causalrep-final-obb}).
This step involves a nested loop of optimization: an outer loop with respect to the representation parameters $W(f)$, and
an inner loop to estimate $P(Y\g f(\mbX),
\mbC)$. This nesting is required because calculating the
CAUSAL-REP objective involves estimating $P(Y\g f(\mbX),
\mbC)$ by maximum likelihood.

To maximize the CAUSAL-REP objective, one can use standard
gradient-based algorithms to handle this nested loop. When the inner
loop---the maximum likelihood estimation of $P(Y\g f(\mbX),
\mbC)$---has a closed-form solution (e.g., under the linear model
in \Cref{eq:y-model-C}), we can plug in the solution to the
CAUSAL-REP objective and directly apply a gradient-based method for
optimization. When the inner loop does not admit a closed form, we can alternate between the gradient updates of the two
optimizations, e.g., alternating between (1) one gradient update step
to maximize the CAUSAL-REP objective, and (2) multiple gradient update
steps until convergence to estimate $P(Y\g f(\mbX), \mbC)$ with maximum
likelihood.


\begin{algorithm}[t]
\small{
\SetKwData{Left}{left}\SetKwData{This}{this}\SetKwData{Up}{up}
\SetKwFunction{Union}{Union}\SetKwFunction{FindCompress}{FindCompress}
\SetKwInOut{Input}{input}\SetKwInOut{Output}{output}
\Input{The training data and its label
$\{\mbx^\mathrm{train}_i,y^\mathrm{train}_i\}_{i=1}^{n_\mathrm{train}}$; the
(out-of-distribution) test data
$\{\mbx^\mathrm{test}_i\}_{i=1}^{n_\mathrm{test}}$; the probabilistic
factor model that generates the training data $P(\mbX, \mbC)$}

\Output{CAUSAL-REP representation function $\hat{f}(\cdot)$; predictions on the test data $\hat{y}_i, i=1, \ldots, n_\mathrm{test}$}
\BlankLine
\# Representation learning with CAUSAL-REP

Fit a probabilistic factor model (\Cref{eq:factor-1,eq:factor-2}) to
the training data and infer $p(\mbc_i\g \mbx^\mathrm{train}_i), i=1,
\ldots, n_\mathrm{train}$;

\If{Pinpointability holds, i.e. $p(\mbc_i\g \mbx^\mathrm{train}_i)$ is
close to a point mass for all $i$}{
\ForEach{training datapoint $i$}{
Pinpoint the unobserved common cause $\mbC$:
$\mbc_i=h(\mbx^\mathrm{train}_i)
\triangleq
\E{}{\mbc_i\g \mbx^\mathrm{train}_i}$ for $i=1,
\ldots, n_\mathrm{train}$;}

Maximize \Cref{eq:causalrep-final-obb} to obtain the CAUSAL-REP
representation $\hat{f}$;}
\BlankLine
\# Out-of-distribution prediction with the CAUSAL-REP representation

Estimate $P(Y\g \hat{f}(\mbX))$ by fitting a model (e.g.
\Cref{eq:pred-linear,eq:pred-nonlinear}) to
$\{\hat{f}(\mbx^\mathrm{train}_i),
y^\mathrm{train}_i\}_{i=1}^{n_\mathrm{train}}$

Predict on test data using the fitted model: $\hat{y}_i=\E{}{Y\g
\hat{f}(\mbx^\mathrm{test}_i)}, i=1, \ldots, n_\mathrm{test}$
\caption{CAUSAL-REP (Supervised)}
\label{alg:causal-rep}}
\end{algorithm}

\parhead{The linear example continued.} We illustrate CAUSAL-REP (\Cref{alg:causal-rep}) with the linear
example in \Cref{subsubsec:pns_practical}: pinpoint the unobserved
common cause $\mbC$ from the training data
$\{\mbx^\mathrm{train}_i,y^\mathrm{train}_i\}_{i=1}^{n_\mathrm{train}}$,
obtain the CAUSAL-REP objective $\hat{f}$, and predict on some
out-of-distribution test data
$\{\mbx^\mathrm{test}_i\}_{i=1}^{n_\mathrm{test}}$. We focus on
learning two-dimensional representations, $f(\mbX) = (f_1(\mbX),
f_2(\mbX))$.

We first perform the same pinpointability step to obtain $\hat{c}_i$
as in \Cref{subsubsec:pns_practical}.

We next maximize the CAUSAL-REP objective
(\Cref{eq:causalrep-final-obb}). We calculate the CAUSAL-REP objective
by plugging in the calculation of $\underline{\gls{PNS}_n(f_j(\mbX), Y
\g f_{-j}(\mbX))}$ in \Cref{eq:pns-conditional-lower-last-step},
where we adopt a linear model for $P(Y\g f(\mbX), \mbC)$. In more
detail, to optimize the CAUSAL-REP objective, we need to write (the
parameters of) the fitted model of $P(Y\g f(\mbX),
\mbC)$ as a function of the representation function $f$. We use the
closed-form estimates of these linear model parameters when fitted to
the training data $\{\mbx_i^{\mathrm{train}},
y_i^{\mathrm{train}}\}_{i=1}^{n^{\mathrm{train}}}$:
\begin{align*}
(\hat{\beta}_0, \hat{\beta}_1, \hat{\beta}_2, \hat{\gamma}) 
&= \argmin
\sum_{i=1}^n(y^{\mathrm{train}}_i - \beta_0 - \beta_1 \cdot f_1(\mbx^{\mathrm{train}}_i) - \beta_2 \cdot f_2(\mbx^{\mathrm{train}}_i)
- \gamma \cdot \hat{c}_i)^2\\
&=(\tilde{\mbX}^\top \tilde{\mbX})^{-1} (\tilde{\mbX}^\top \tilde{Y}),
\end{align*}
where $\tilde{\mbX}$ is an $n^{\mathrm{train}}\times 4$ matrix with
$\tilde{\mbX}_{i1}=1$, $\tilde{\mbX}_{i2} =
f_1(\mbx^{\mathrm{train}}_i)$, $\tilde{\mbX}_{i2} =
f_2(\mbx^{\mathrm{train}}_i)$, $\tilde{\mbX}_{i2} =
\hat{c}_i$, and $\tilde{Y}$ is an $n\times 1$ vector with
$\tilde{Y}_{i1} = y^{\mathrm{train}}_i$, $i=1, \ldots, n$. (When
closed-form solutions are not available, we need to perform
gradient-based optimization to obtain the fitted model parameters.)

Plugging in these terms, we maximize the CAUSAL-REP objective via
gradient descent with respect to the parameters of the representation
function $W(f)$. The optimal $\hat{W}(f)$ give the CAUSAL-REP
representation function $\hat{f}$.

Finally, we perform out-of-distribution prediction with the CAUSAL-REP
representation $\hat{f}$. We train a prediction function by fitting a
linear model to $P(Y\g \hat{f}(\mbX))$,
\begin{align}
y^{\mathrm{train}}_i = \beta_0^{\mathrm{pred}} +
\beta_1^{\mathrm{pred}} \cdot\hat{f}_1(\mbx^{\mathrm{train}}_i) +
\beta_2^{\mathrm{pred}}\cdot \hat{f}_2(\mbx^{\mathrm{train}}_i) + \epsilon,
\epsilon\sim \cN(0,(\sigma^{\mathrm{pred}})^2).
\end{align}
We obtain the estimated regression coefficients
$\{\hat{\beta}_0^{\mathrm{pred}}, \hat{\beta}_1^{\mathrm{pred}},
\hat{\beta}_2^{\mathrm{pred}}, \sigma^{\mathrm{pred}}\}$ by maximum
likelihood. They allow us to make predictions on the test data
$\{\mbx^\mathrm{test}_i\}_{i=1}^{n_\mathrm{test}}$,
\begin{align}
\hat{y}_i = \E{}{Y\g
\hat{f}(\mbx^\mathrm{test}_i)}=\hat{\beta}_0^{\mathrm{pred}} +
\hat{\beta}_1^{\mathrm{pred}} \cdot\hat{f}_1(\mbx^{\mathrm{test}}_i) +
\hat{\beta}_2^{\mathrm{pred}}\cdot \hat{f}_2(\mbx^{\mathrm{test}}_i).
\end{align}

\subsubsection{Extending CAUSAL-REP to unsupervised settings}

We extend CAUSAL-REP to unsupervised settings where labels are not
available. We focus on the task of instance discrimination in
unsupervised representation
learning~\citep{hadsell2006dimensionality}, where the goal is to find
representations that can distinguish the different subjects.

\parhead{The unsupervised CAUSAL-REP.} We focus on the unsupervised
setting where some form of data augmentation is available. That is,
the dataset contains one raw observation for each subject; this raw
observation is then augmented multiple times, which leads to multiple
observations per subject. Given this unsupervised dataset, which attach the subject ID to each observation as a supervised label.  We can thereby formulate the
instance discrimination problem as a supervised task, one that finds representations that are informative of the subject IDs and extend CAUSAL-REP to an unsupervised setting. The resulting algorithm turns out to be closely related to contrastive learning~\citep{chen2020simple}.

In more detail, we begin with a dataset with $n$ i.i.d.\ samples,
$(\mbx_1,
\ldots, \mbx_n)$. Assume that each data point is associated with a different
subject~$i$. Next, we augment each data point with $U-1$ augmentations,
producing an augmented dataset, $\{\{\mbx_{i}^u\}_{u=1}^U\}_{i=1}^n$, $U\geq 2$.
We then create an $n$-dimensional label, $\mby_{i}^u = (y_{i1}^u,
\ldots, y_{in}^u)$, for each data point in the augmented dataset, where
\begin{align}
\label{eq:unsupervised-aug-label}
y_{is}^u =
\mathbb{I}\{\text{$\mbx_{i}^u$ belongs to subject $s$}\}, s=1, \ldots, n.
\end{align}
This labeling turns the unsupervised problem into a supervised one with respect to
the augmented labeled dataset $\{\{\mbx_{i}^u,
\mby_{i}^u\}_{u=1}^U\}_{i=1}^n$. We then employ the
CAUSAL-REP objective to find representations that are necessary and
sufficient causes of the $n$ outcomes $(Y_1,\ldots, Y_n)$:
\begin{align}
\label{eq:causalrep-unsupervised}
\max_{f}
\sum_{s=1}^n\sum_{j=1}^m \log
\underline{\gls{PNS}_{n\cdot U}(f_j(\mbX), Y_s \g f_{-j}(\mbX))} + \lambda\cdot
R(f\s \{\{\mbx_i^u, y_{is}^u, \mbc^u_i\}_{u=1}^U\}_{i=1}^n),
\end{align}
where $\underline{\gls{PNS}_{n\cdot U}(f_j(\mbX), Y_s \g
f_{-j}(\mbX))}$ is the conditional efficiency and spuriousness
(\Cref{eq:pns-conditional-dataset}) for the $s$th outcome; it is
calculated on the augmented dataset $\{\{\mbx_{i}^u,
\mby_{i}^u\}_{u=1}^U\}_{i=1}^n$. The penalty term $R(\cdot)$ is
the same penalty as in the CAUSAL-REP objective
(\Cref{eq:causalrep-final-obb,eq:causalrep-reg}). Solving
\Cref{eq:causalrep-unsupervised} produces non-spurious and efficient
representations of $\mbX$ for instance discrimination. (We summarize
the algorithm in \Cref{alg:causal-rep-unsupervised}.)

The optimization objective in \Cref{eq:causalrep-unsupervised}
considers the $n$ outcomes $\{Y_s\}$ one at a time and then averages
over them. It focuses on the contrast of one-vs-all, finding
representations that distinguish each subject from the rest. This
objective does not consider the $n$ outcomes as a single
$n$-dimensional outcome $(Y_1, \ldots, Y_n)$ and calculate its
\gls{PNS}. 

\parhead{Unsupervised CAUSAL-REP and contrastive learning.} The
unsupervised CAUSAL-REP objective (\Cref{eq:causalrep-unsupervised})
is closely related to contrastive learning~\citep{chen2020simple},
whose learning objective is to maximize
\begin{align}
\sum_{u_1=1}^U\sum_{u_2=1}^U -\log \frac{\exp(d(f(\mbx_i^{u_1}), f(\mbx_i^{u_2})))}{\sum_{i'\ne i}\sum_{u_*=1}^U\exp(d(f(\mbx_i^{u_1}), f(\mbx_{i'}^{u_*})))},
\end{align}
where $d(\cdot, \cdot)$ is some distance function. This contrastive
learning objective targets the same instance discrimination task as
the unsupervised CAUSAL-REP does; it aims for representations that can
determine whether two data points come from the same subject.
Contrastive learning thus is also able to predict whether each data
point belongs to subject $s$, for $s=1, \ldots, n$, as is the
unsupervised CAUSAL-REP.

That said, the two algorithms only coincide when different dimensions
of $\mbX$ are independent---they do not share an unobserved common
cause as in \Cref{fig:causalrep-scm}. The two algorithms behave very
differently when such a common cause exists, especially when the
common cause induces some spurious feature (e.g., image color).
Concretely, there may be a subset of $(X_1, \ldots, X_m)$ that is
highly correlated with the label $\mbY$ but cannot causally determine
$\mbY$. In such cases, contrastive learning may pick up these spurious
features if no augmentations are performed for these features (e.g.,
there is no augmentation that randomly changes the color of images).
In contrast, the unsupervised CAUSAL-REP would not capture these
spurious features.

This difference illustrates how CAUSAL-REP learns non-spurious
representations in a different way than data-augmentation-based
methods (including contrastive learning). CAUSAL-REP capitalizes on a
causal perspective and its treatment of high-dimensional $\mbX$.
Data-augmentation-based methods instead wipe out the correlation
between spurious features and the label by performing data
augmentation. We will illustrate this difference further in
\Cref{subsubsec:causalrep-unsupervised}.


\subsection{Empirical Studies of CAUSAL-REP}
\label{sec:causalrep-empirical}

We study CAUSAL-REP in both image and text datasets. We find that
\glsreset{POC}\gls{POC} are effective in distinguishing
efficient/inefficient and non-spurious/spurious representations.
Moreover, CAUSAL-REP finds non-spurious features in both supervised
and unsupervised settings; it also outperforms existing unsupervised
representation learning algorithms in downstream out-of-distribution
prediction.

\glsreset{POC}

\subsubsection{How well do probabilities of causation measure
efficiency and non-spuriousness of features?}
\label{sec:pns-distinguish}

We first study the correspondence between probabilities of causation
(\gls{PS}, \gls{PN}, \gls{PNS}) and the efficiency/non-spuriousness of
representations. We generate features with known efficiency and
non-spuriousness properties; we find that the (lower bound of)
probabilities of causation in \Cref{thm:pns-id} are consistent with
these properties.

\parhead{The simulated data.} We simulate two binary features $Z_1,
Z_2$ and two binary outcomes~$Y_1, Y_2$:
\begin{align*}
Z_1&\sim\mathrm{Bernoulli}(0.4),\\
Z_2&=Z_1 \oplus \mathrm{Bernoulli}(p),  p \in \{0, 0.1, \ldots, 0.9, 1\},\\
Y_1&=Z_1 \oplus \mathrm{Bernoulli}(0.2), \\
Y_2&=(Z_1 \& Z_2) \oplus\mathrm{Bernoulli}(0.2),
\end{align*}
where $\oplus$ indicates the XOR operator. We vary the parameter $p$;
a small $p$ implies a high correlation between $Z_1$ and $Z_2$. The
generative process of $Y_1, Y_2$ implies that (1) $Z_1$ is necessary
(a.k.a. efficient) and sufficient (a.k.a. non-spurious) for $Y_1$, but
is necessary and insufficient for $Y_2$; (2) $Z_2$ is neither
necessary nor sufficient for $Y_1$, but is necessary and insufficient
for $Y_2$; (3) $Z_1 \& Z_2$ is necessary and sufficient for $Y_2$.
Moreover, when the correlation between $Z_1$ and $Z_2$ increases, then
using both features as a representation become increasingly
inefficient; the conditional necessity decreases. We calculate the
lower bound of \gls{PS}, \gls{PN}, \gls{PNS} of $Z_1, Z_2$ for both
outcomes $Y_1, Y_2$; the exact value is not identifiable from data.

\parhead{Results.}
\Cref{fig:pns-measurement,fig:pns-measurement-supp-1,fig:pns-measurement-supp-2}
present the probabilities of causation of the features.
\Cref{fig:pns-measurement-condpns} shows that the (lower bound of)
conditional \gls{PNS} can correctly signal the non-spuriousness of
features: $Z_2$ is spurious for $Y_1$, but $Z_1$ is non-spurious for
$Y_1$, and both $Z_1, Z_2$ are non-spurious for $Y_2$. Moreover, we
study how \gls{POC} fare when the correlation between $Z_1$ and $Z_2$
increases. \Cref{fig:pns-measurement} shows that the (lower bounds of)
\gls{POC} of $Z_1$ for $Y_2$ increase as the correlation increases,
which is consistent with the intuition that $Z_2$ becomes less
necessary for $Y_2$ given $Z_1$ given increasingly higher
correlations. Similarly,
\Cref{fig:pns-measurement-supp-1,fig:pns-measurement-supp-2} show that
(the lower bounds of) the unconditional \gls{POC} of $Z_2$ for $Y_1$
also increase. It is also consistent with the intuition: $Z_2$ is an
increasingly better surrogate of $Z_1$ for $Y_1$ under higher
correlations between $Z_1$ and $Z_2$.

\begin{figure}[t]
  \centering
\captionsetup[subfigure]{justification=centering}
\begin{subfigure}{0.20\textwidth}
  \centering
    \includegraphics[width=\linewidth]{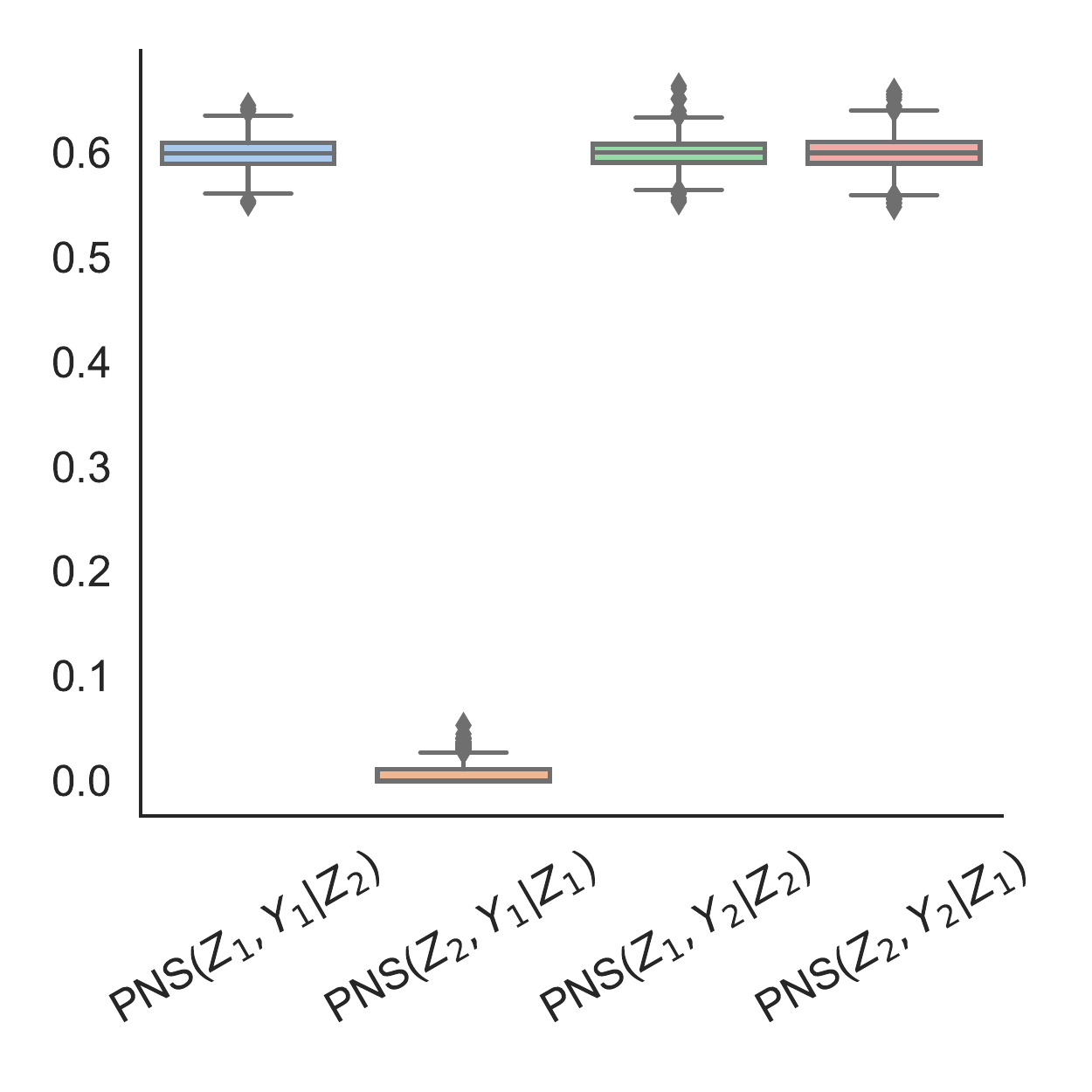}
  \caption{Conditional \gls{PNS}\label{fig:pns-measurement-condpns}}
\end{subfigure}
\begin{subfigure}{0.73\textwidth}
  \centering
    \includegraphics[width=\linewidth]{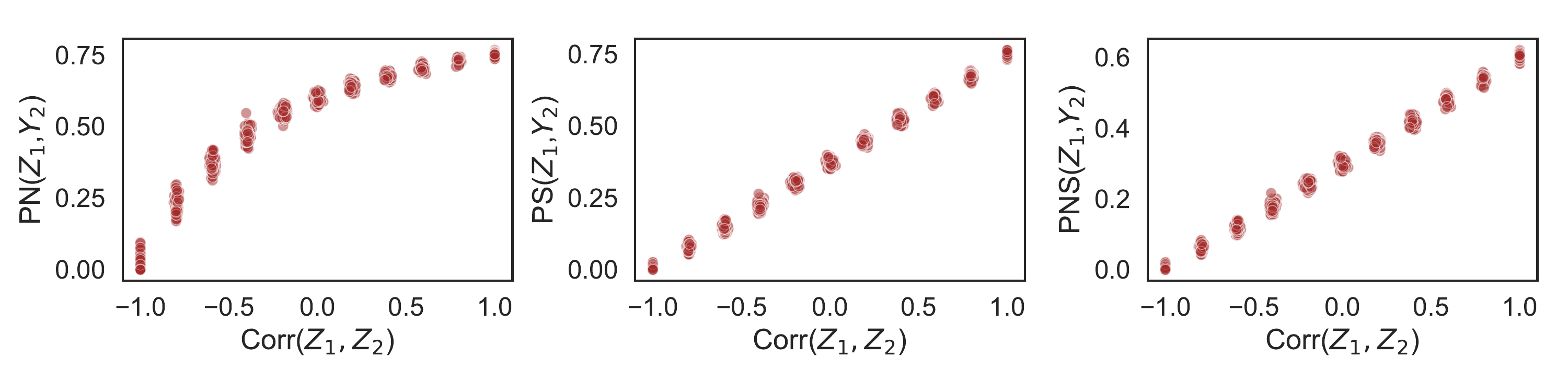}
  \caption{\glsreset{POC}\Gls{POC} of $Z_1$ for $Y_2$ \label{fig:pns-measurement-poc}}
\end{subfigure}
\caption{Probabilities of causation can distinguish
spurious/non-spurious and efficient/inefficient features. (a)
Conditional \gls{PNS} signals that $Z_2$ is spurious for $Y_1$. (b)
Probabilities of causation (\gls{PNS}, \gls{PN}, \gls{PS}) of $Z_1$
for $Y_2$ increase as $Z_1$ and $Z_2$ become increasingly highly
correlated. \label{fig:pns-measurement}}
\end{figure}

\subsubsection{Does the supervised CAUSAL-REP pick up spurious features in
synthetic data?}

\label{subsubsec:linear-synthetic}

We next study the performance of the supervised CAUSAL-REP in a toy
synthetic dataset. We use \gls{PPCA} as the pinpointing factor model
and linear functions as the class of representation functions.

\parhead{The simulated data.} We generate a dataset of core and
spurious features $(X_1, \ldots, X_5)$, who are highly correlated in
the training set but not so in the test set,
\begin{align*}
\bar{X}_1^\mathrm{train}, \ldots, \bar{X}_5^\mathrm{train} &\sim \cN(0, 0.05\cdot I_5 + 0.95 \cdot J_5),\\
\bar{X}_1^\mathrm{test}, \ldots, \bar{X}_5^\mathrm{test} &\sim \cN(0, 0.05\cdot I_5 + 0.05 \cdot J_5),
\end{align*}
where $I_5$ is a $5\times 5$ identity matrix and $J_5$ is a $5\times
5$ all-ones matrix. For both training and test sets, we inject noise
to features to lower the signal-to-noise ratio of the problem: $X_j
\sim \cN(\bar{X}_j, 0.4^2), j=1, 2; X_j \sim \cN(\bar{X}_j, 0.3^2),
j=3,4,5.$ Finally, we generate an outcome $Y$ that only depends on the
core features $Y \sim \cN(\beta_1 X_1 + \beta_2 X_2, 1)$, where the
coefficients are drawn from a uniform $\beta_j
\sim \mathrm{Unif}[0,10], j=1, 2,$.

\parhead{Results.} We compare the performance of CAUSAL-REP, linear
regression, and an oracle algorithm; the oracle is equipped with the
knowledge of which features are spurious and only performs linear
regression against the non-spurious features. We study whether these
algorithms pick up spurious features by evaluating their the \gls{OOD}
predictive performances. If an algorithm picks up spurious features,
then its \gls{OOD} predictions will suffer, because spurious features
are much less predictive in test data than in training data.
\Cref{fig:causal-rep-linear} presents the result: CAUSAL-REP
outperforms linear regression in \gls{OOD} predictive R-squared; its
predictive performance is not far from the oracle algorithm.

\subsubsection{Does the supervised CAUSAL-REP produce non-spurious
representations for image data? A study on Colored MNIST and CelebA}

\label{subsubsec:image}

\begin{figure}[t]
  \centering
  \captionsetup[subfigure]{justification=centering}
\begin{subfigure}{0.48\textwidth}
  \centering
    \includegraphics[height=0.7\linewidth]{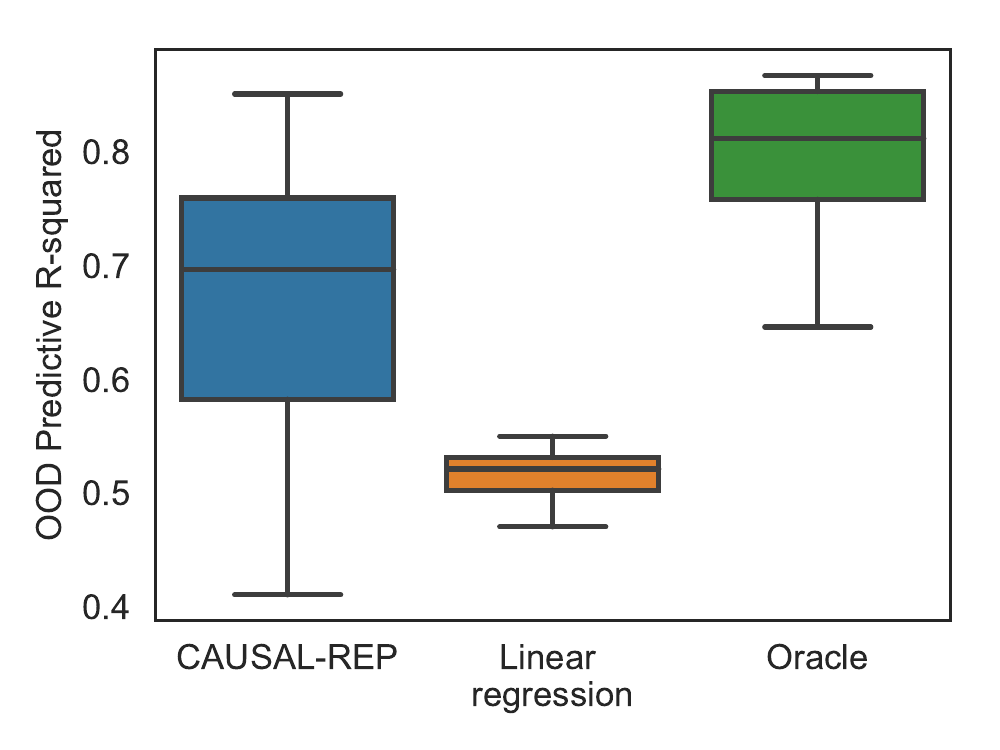}
  \caption{Supervised CAUSAL-REP: \\Toy synthetic data \label{fig:causal-rep-linear}}
\end{subfigure}
\begin{subfigure}{0.48\textwidth}
  \centering
    \includegraphics[height=0.7\linewidth]{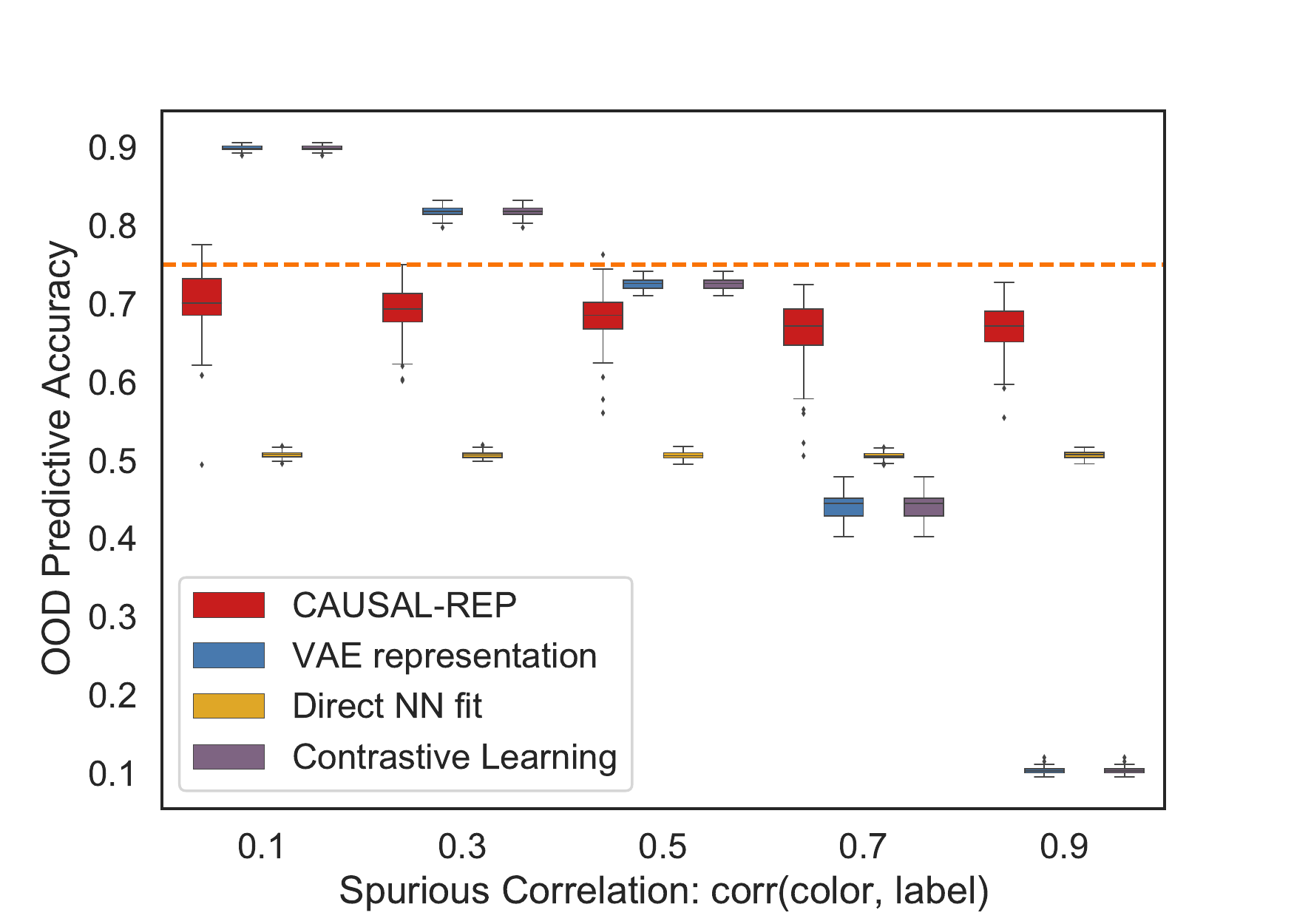}
  \caption{Unsupervised CAUSAL-REP: \\Colored and Shifted MNIST \label{fig:coloredmnist_unsupervised_poscorr}}
\end{subfigure}
\caption{CAUSAL-REP learns non-spurious representations in both
supervised and unsupervised settings. (a) CAUSAL-REP outperforms
linear regression in \gls{OOD} prediction in toy synthetic data. (b)
CAUSAL-REP outperforms baseline representation learning algorithms
(e.g. directly fitting neural networks, or performing contrastive
learning, or adopting \gls{VAE} representations) in the colored and
shifted MNIST dataset. The dashed yellow line indicates the
theoretical maximum of \gls{OOD} predictive accuracy. (Higher is
better.)}
\label{fig:Causal-Rep-supervised}
\end{figure}

\begin{figure}[t]
  \centering
  \captionsetup[subfigure]{justification=centering}
\begin{subfigure}{0.48\textwidth}
  \centering
    \includegraphics[width=\linewidth]{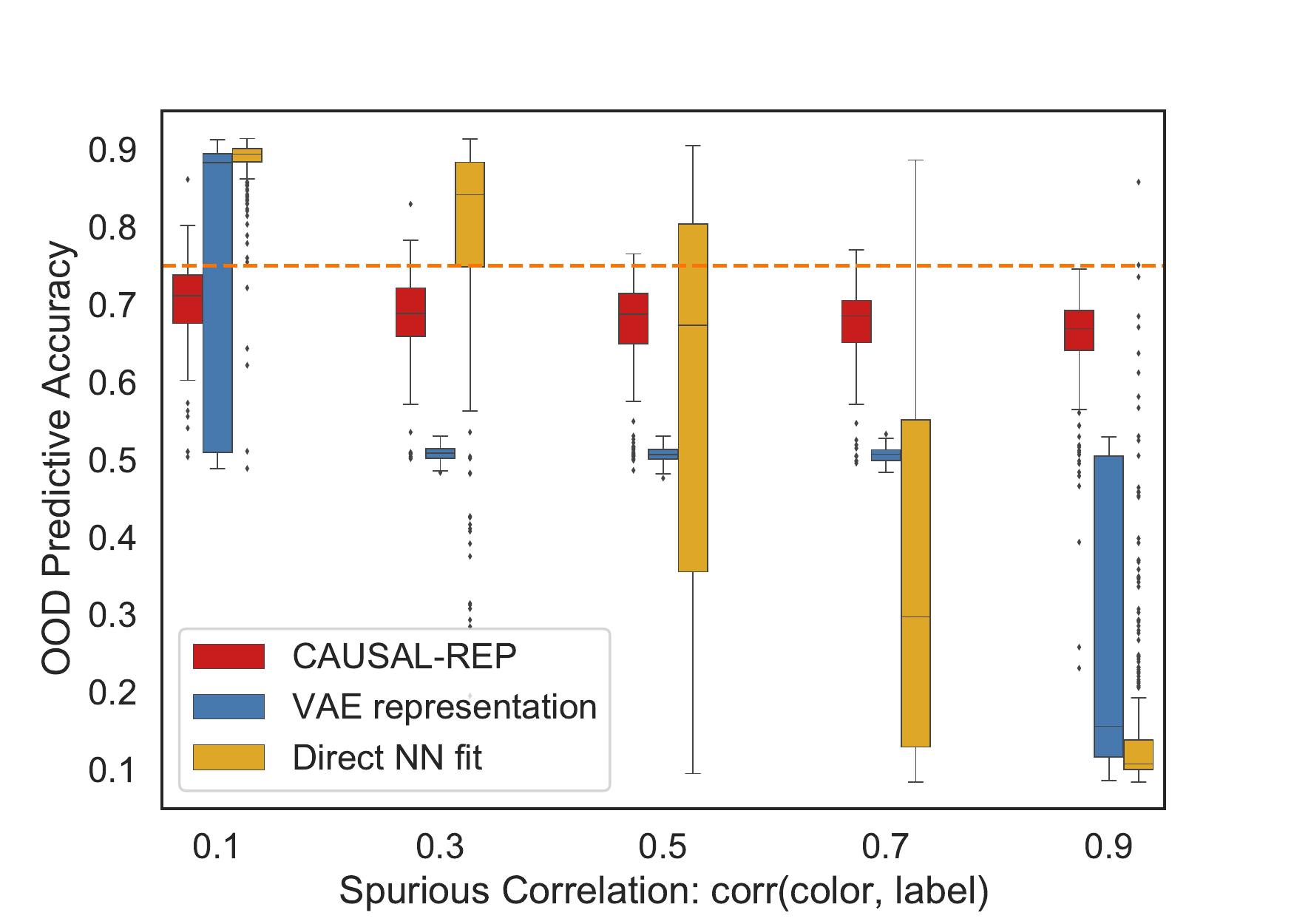}
  \caption{\label{fig:coloredmnist_spurious}}
\end{subfigure}
\begin{subfigure}{0.48\textwidth}
  \centering
    \includegraphics[width=\linewidth]{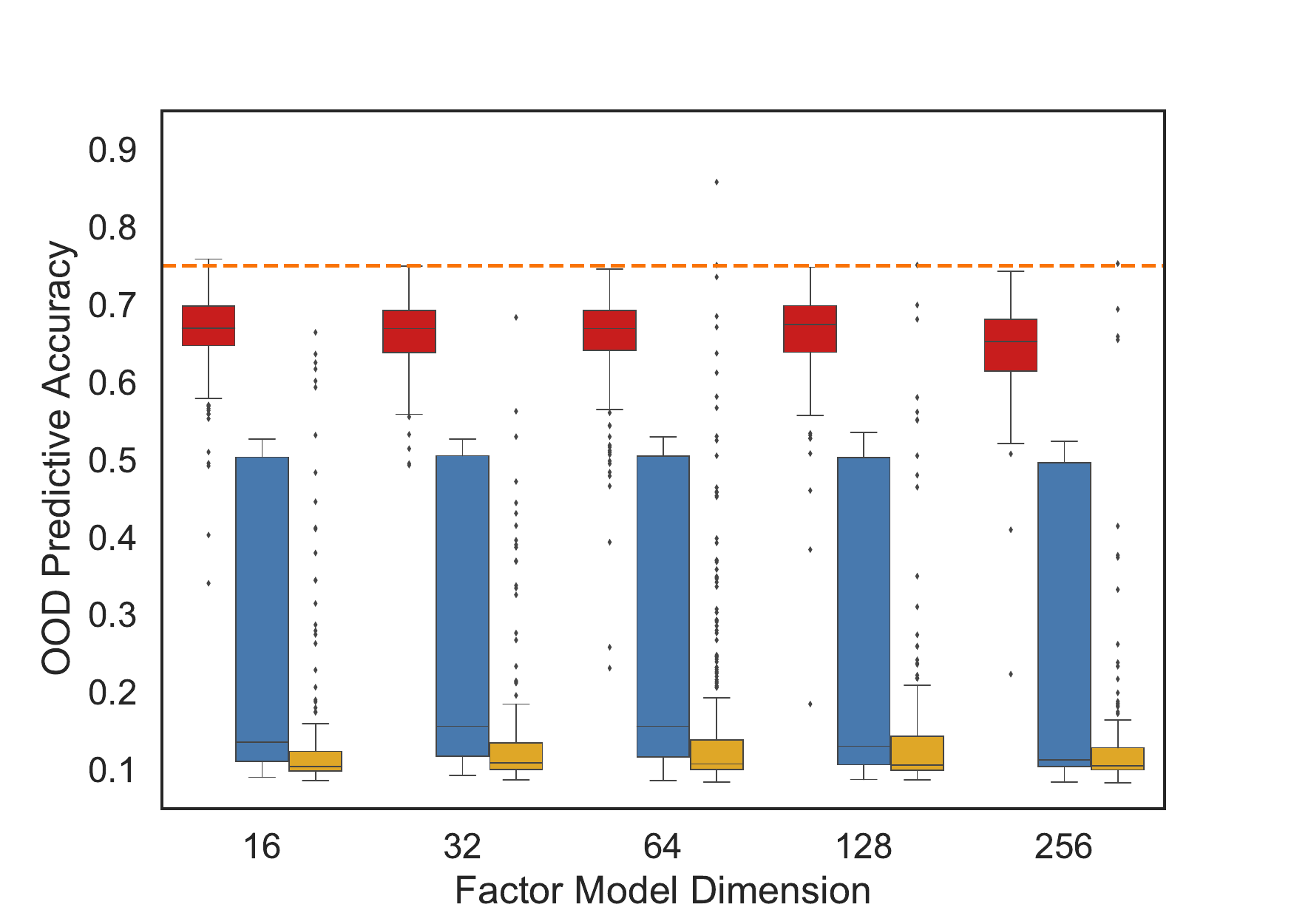}
  \caption{\label{fig:coloredmnist_dim64}}
\end{subfigure}
\caption{(a) The supervised CAUSAL-REP learns non-spurious
representations in colored MNIST and outperforms baseline
representation learning algorithms (e.g. directly fitting neural
networks, or adopting \gls{VAE} representations)in \gls{OOD}
prediction. (b) The performance of CAUSAL-REP is robust to the choice
of the latent dimensionality of probabilistic factor models. The
dashed yellow line indicates the theoretical maximum of \gls{OOD}
predictive accuracy. (Higher is better.)
\label{fig:colored-mnist-supervised}}
\end{figure}

We next study the supervised CAUSAL-REP in image datasets: Colored
MNIST~\citep{arjovsky2019invariant} and
CelebA~\citep{liu2015faceattributes}. We consider the version of
supervised CAUSAL-REP algorithm which adopts a \gls{VAE} with 64
latent dimensions as the probabilistic factor model for pinpointing.
For representation functions, we consider a two-layer neural network
with 20-dimensional outputs. We again evaluate the non-spuriousness of
the representations by the \gls{OOD} predictive accuracy; if the
learned representation captures spurious features, then it will suffer
in \gls{OOD} prediction.

\parhead{Competing methods.} We compare CAUSAL-REP with a baseline
representation learning algorithm that fits a neural network to the
labels and extracts its penultimate layer as the
representation~\citep{bengio2013representation}. As we use \gls{VAE}
for pinpointing in CAUSAL-REP in this study, we also compare with
directly adopting the \gls{VAE} representation for \gls{OOD}
prediction.

\parhead{The Colored MNIST study.} We focus on the colored MNIST data
with the digits `3' and `8' and colors `red' and `green'; see
\Cref{subsec:colored-mnist-supp} for the detailed experimental setup.

\Cref{fig:coloredmnist_spurious} presents the \gls{OOD} predictive
accuracy of CAUSAL-REP across different levels of spurious
correlations. CAUSAL-REP outperforms other baseline methods in
\gls{OOD} predictive accuracy when the spurious correlation is larger
than 0.5. Its performance remains close to the theoretical maximum
(the yellow dashed line) across different levels of the spurious
correlation, suggesting that CAUSAL-REP does not pick up spurious
features. In contrast, the \gls{OOD} performance of representation
learning by directly fitting neural networks quickly degrades as the
spurious correlation increases; so does \gls{OOD} prediction with the
\gls{VAE} representations.

\Cref{fig:coloredmnist_dim64} evaluates the robustness of CAUSAL-REP
against the latent dimensionality choice of the pinpointing factor
model, fixing the spurious correlation at 0.9. We find that the
performance of CAUSAL-REP does not change much as we vary the latent
dimensionality of pinpointing \gls{VAE}. The relative performance of
CAUSAL-REP and the competing methods also stay stable.

\parhead{The CelebA study.} We next study CAUSAL-REP on the CelebA
dataset~\citep{liu2015faceattributes}. We create training and test
sets, each containing 5,000 data points, by subsampling the CelebA
datasets. We focus on face attributes with a relatively balanced
distribution in the raw CelebA dataset. We designate pairs of target
attributes and spurious attributes, and subsample such that the two
are highly correlated in the training set and not as correlated in the
test set. We then perform representation learning and \gls{OOD}
prediction for the target labels, using CAUSAL-REP and other competing
methods.

\Cref{tab:celeba-results} presents the results for different pairs of
target and spurious face attributes. Though the spurious label and the
target label are highly correlated, CAUSAL-REP can pick up
non-spurious features that inform the target label and outperform the
baseline algorithm that directly fits a neural network. In most
settings, CAUSAL-REP also outperforms \gls{OOD} prediction with
\gls{VAE} representations. The exception is that \gls{VAE}
representations can outperform CAUSAL-REP representations when the
spurious correlations climb up to 0.9, though the performance of
CAUSAL-REP is still competitive. The spurious features in these
settings violate the positivity condition: Mustache and goatee
exclusively appear on males in the subsampled dataset.


\begin{table}[t]
    \centering
\begin{tabular}{p{2.5cm}p{2.3cm}p{1.2cm}p{1.2cm}p{2.2cm}p{2.2cm}p{2.2cm}}
\toprule
         &      &  spurious corr. (train) &  spurious corr. (test) &  CAUSAL-REP &  Direct NN fit & VAE rep. \\
target & spurious &      &      &        &       &    \\
\midrule
Arched Brows & Eye Bags  &  0.784 & -0.797  &  \bfseries{0.539(0.036)} &  0.514(0.029)&  0.499(0.012) \\
Arched Brows & Earrings &  0.799 &-0.791 &  \bfseries{0.521(0.025)}&  0.504(0.022)&  0.494(0.008) \\
Attractive & Necklace &  0.793 & -0.787  &  \bfseries{0.537(0.022)}&  0.505(0.030)&  0.485(0.018) \\
Black Hair & Mouth Open &  0.791 &-0.795  &  \bfseries{0.594(0.033)}&  0.566(0.060)&  0.505(0.010) \\
Goatee & Male &                     0.889 &  0.053  &  \bfseries{0.728(0.073)}&  0.566(0.102)&  \bfseries{0.867(0.165)} \\
Mustache & Black Hair &             0.764 & -0.778  &  0.525(0.023)&  0.512(0.037)&  \bfseries{0.540(0.005)} \\
Mustache & Male &                   0.892 &  0.088 &  \bfseries{0.787(0.066)}&  0.555(0.097)&  \bfseries{0.855(0.212)} \\
\bottomrule
\end{tabular}
    \caption{CAUSAL-REP learns non-spurious representations in the
    CelebA dataset and outperforms baseline representation learning
    algorithms in \gls{OOD} prediction. \label{tab:celeba-results}}
\end{table}

\subsubsection{Does the supervised CAUSAL-REP produce non-spurious
representations for text data? A study on reviews corpora and
sentiment analysis}

We next study CAUSAL-REP on text datasets: the
Amazon~\citep{wang2011latent,wang2010latent},
Tripadvisor,\footnote{http://times.cs.uiuc.edu/~wang296/Data/} and
Yelp\footnote{https://www.yelp.com/dataset/documentation/main} reviews
corpora, and the IMDB-L, IMDB-S, and Kindle dataset as is processed in
\citet{wang2020identifying,wang2020robustness}. In these studies, we
convert these corpora into bags of words. We use \gls{PPCA} as a
pinpointing factor model for CAUSAL-REP and consider representations
whose each dimension is a convex combination of words.

\parhead{The reviews corpora study.} We begin with the raw reviews
datasets from Amazon, Tripadvisor, and Yelp. We create a binary label
for each review by converting 4 and 5 stars to a positive label and 1
and 2 stars to a negative label. We then inject irrelevant words to the
training dataset as spurious features by randomly adding in ``as'',
``also'', ``am'', ``an'' to reviews with positive labels; the
resulting spurious correlation is around 0.9. We create two test
datasets: one is in-distribution with the spurious words present as in
the training set; the other is out-of-distribution without the
randomly added spurious words.

\Cref{fig:reviews-text} presents the predictive accuracy of CAUSAL-REP
in both test sets. We find that the predictive performance of
CAUSAL-REP is stable across in-distribution and out-of-distribution
test sets, suggesting that it learns non-spurious representations. In
contrast, logistic regression predicts well in the in-distribution
test set but not in the out-of-distribution test set, suggesting that
it picks up spurious features that only exist in the training set.

\Cref{tab:reviews-top-words} presents the most informative words of
the (positive or negative) ratings, suggested by the CAUSAL-REP
representation and the logistic regression coefficients. Across three
reviews corpora, logistic regression returns the spurious words
``as'', ``also'', ``am'', ``an'' as the top words. In contrast,
CAUSAL-REP extracts words that are more relevant for the ratings.

\begin{figure}[t]
  \centering
  \captionsetup[subfigure]{justification=centering}
\begin{subfigure}{0.31\textwidth}
  \centering
    \includegraphics[width=\linewidth]{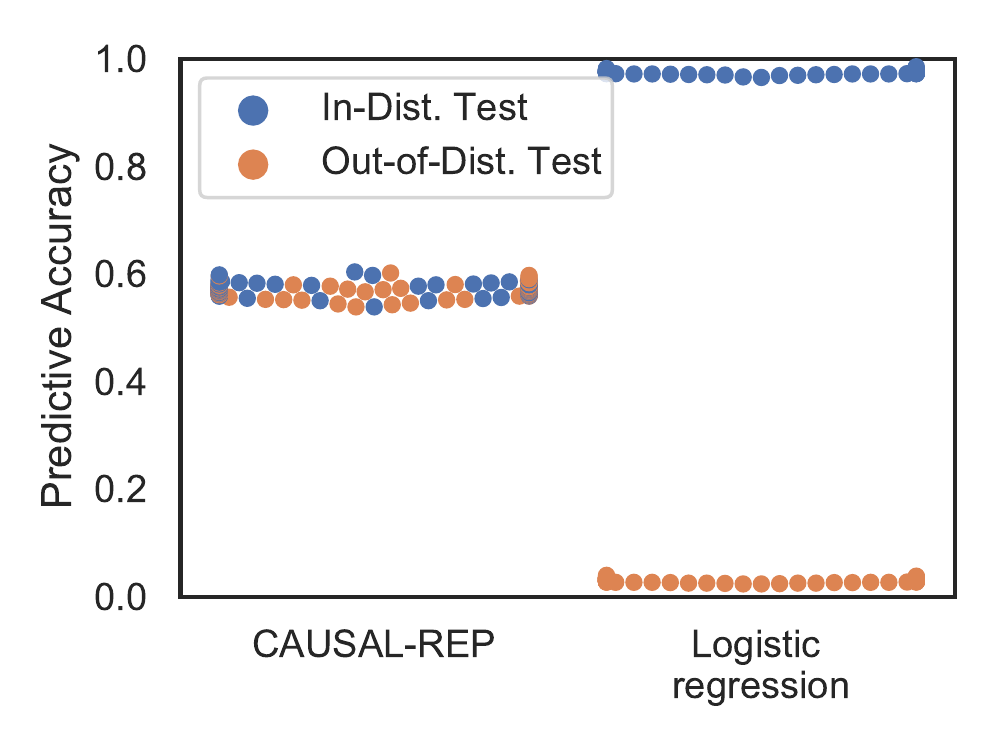}
  \caption{Amazon reviews}
\end{subfigure}
\begin{subfigure}{0.31\textwidth}
  \centering
    \includegraphics[width=\linewidth]{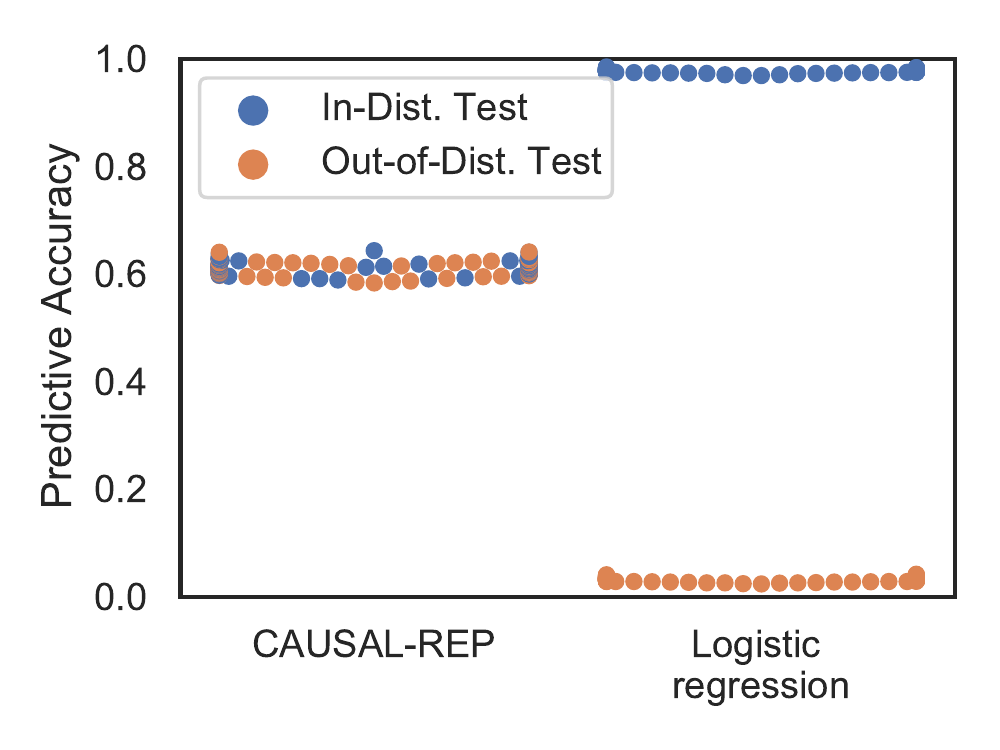}
  \caption{Tripadvisor reviews}
\end{subfigure}
\begin{subfigure}{0.31\textwidth}
  \centering
    \includegraphics[width=\linewidth]{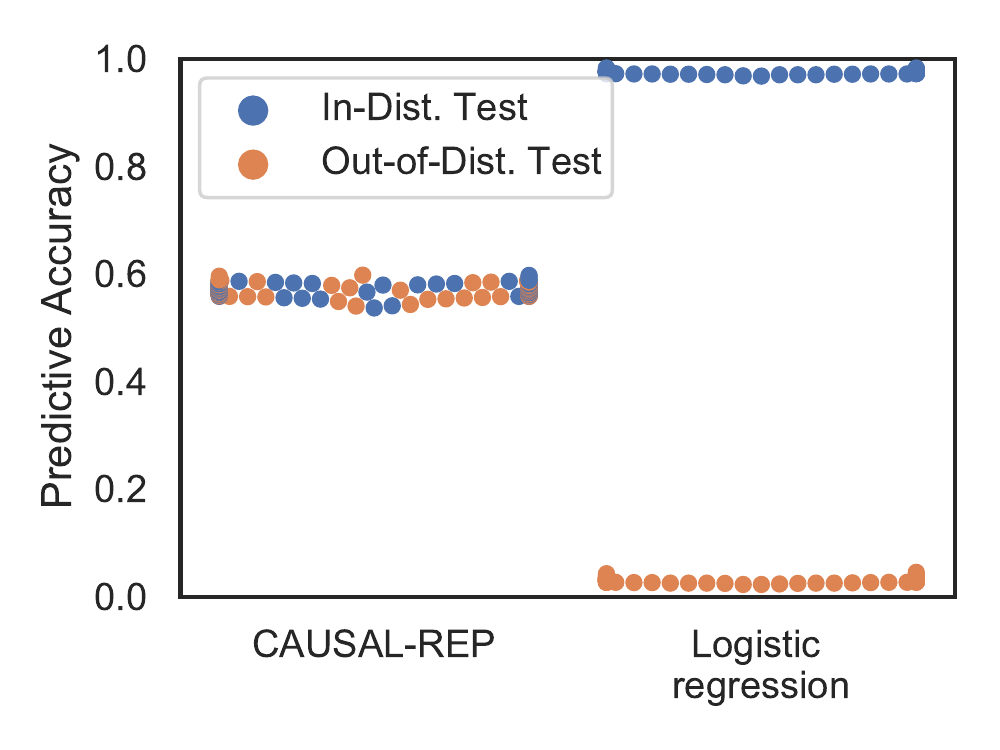}
  \caption{Yelp reviews}
\end{subfigure}
\caption{CAUSAL-REP learns non-spurious representations across reviews
text copura; its predictive performance is stable across
in-distribution and out-of-distribution test
sets.\label{fig:reviews-text}}
\end{figure}

\parhead{The sentiment analysis study.} We next employ CAUSAL-REP for
sentiment classification on the IMDB-L, IMDB-S, and Kindle
datasets~\citep{wang2020identifying,wang2020robustness}. We evaluate
CAUSAL-REP on their raw observational and counterfactual test sets,
without employing additional data augmentations as in
\citet{wang2020identifying,wang2020robustness}. The observational test
sets are in-distribution test sets, and the counterfactual test sets
are out-of-distribution test sets.

\Cref{tab:text-real} presents the predictive performance of CAUSAL-REP
compared with logistic regression. While logistic regression
outperforms in in-distribution predictive accuracy, CAUSAL-REP
outperforms in out-of-distribution prediction. Moreover, the
predictive performance of CAUSAL-REP is similar across in-distribution
and out-of-distribution test sets, suggesting that CAUSAL-REP produces
non-spurious representations.


\begin{table}[t]
\centering
\begin{tabular}{lrrrr}
\toprule
& \multicolumn{2}{c}{Observational test set} & \multicolumn{2}{c}{Counterfactual test set}\\
& Logistic Regression &   CAUSAL-REP &  Logistic Regression &   CAUSAL-REP\\
\midrule
IMDB-L&    \bfseries{0.669} & 0.645 & 0.591 & \bfseries{0.642}\\
IMDB-S &    \bfseries{0.836} & 0.682 & 0.570 & \bfseries{0.621}\\
Kindle&    \bfseries{0.850} & 0.618 & 0.468 & \bfseries{0.572}\\
\bottomrule
\end{tabular}
\caption{CAUSAL-REP outperforms naive representation learning
algorithms in predicting on counterfactual test sets.
\label{tab:text-real}}
\end{table}

\subsubsection{How well does unsupervised CAUSAL-REP perform on
instance discrimination? A study on colored and shifted MNIST}

\label{subsubsec:causalrep-unsupervised}

Finally, we study CAUSAL-REP in the unsupervised setting. We focus on
image datasets in the unsupervised setting because non-spurious
features that distinguish different images are more readily defined in
the image domain than in the text domain. We evaluate the
non-spuriousness of the unsupervised CAUSAL-REP again by \gls{OOD}
prediction. Given the unsupervised CAUSAL-REP representation, we fit a
prediction model to the target label and test its predictive
performance on \gls{OOD} test sets.

\parhead{Competing methods.} We compare CAUSAL-REP with baseline
representation learning algorithms that fit a neural network to the
subject ID label in \Cref{subsubsec:causalrep-unsupervised}. We also
compare with contrastive learning~\citep{chen2020simple} and the
\gls{VAE} representation.

\parhead{The colored and shifted MNIST study.} We construct the
colored and shifted MNIST dataset by coloring and shifting digits in a
way that is highly correlated with the digit labels; see
\Cref{subsec:colored-shifted-mnist-supp} for details. (Representation
learning with the unsupervised CAUSAL-REP does not make use of the
digit labels; these labels are only used in producing a prediction
function from the representations to the labels.)

\Cref{fig:coloredmnist_unsupervised_poscorr} presents the \gls{OOD}
predictive accuracy of the unsupervised CAUSAL-REP on colored and
shifted MNIST. CAUSAL-REP outperforms baseline representation learning
algorithms when the spurious correlation is high, and their predictive
accuracy stays close to the theoretical maximum (the yellow dashed
line). Comparing CAUSAL-REP with contrastive learning with shift
augmentation, we find that the \gls{OOD} predictive accuracy of
contrastive learning degrades when the spurious correlations increase
over 0.7. It is because the shift augmentation can not rule out the
spurious color feature; it can only avoid picking up the spurious
shift feature. This observation echoes the discussion in
\Cref{subsubsec:causalrep-unsupervised}, illustrating how the
unsupervised CAUSAL-REP relies on a different mechanism to produce
non-spurious representations than contrastive learning and data
augmentation.


\section{Unsupervised Representation Learning:\\
Disentanglement}

\label{section:disentanglement}

In this section we study the desideratum of disentanglement in
unsupervised representation learning. Unsupervised representation
learning aims to find a low-dimensional representation for
high-dimensional objects without the help of labels. Given an
$m$-dimensional object, $\mbX=(X_1, \ldots, X_m)$, the goal is again
to find a $d$-dimensional representation, $\mbZ = (Z_1, \ldots, Z_d)
\triangleq (f_1(\mbX), \ldots, f_d(\mbX))$.

We focus on a causal definition of disentanglement: different
dimensions of the representation shall encode features that do not
causally affect each other~\citep{Suter2019}. An absence of
causal relationships among different dimensions of the representation
can be viewed as independent manipulability of each dimension. Other
related criteria include the notion of independent
mechanisms~\citep{Suter2019,parascandolo2018learning,scholkopf2012causal}
and the concept of independently controllable
factors~\citep{thomas2017independently,thomas2018disentangling}.
Based on this causal definition, we will develop unsupervised metrics
and algorithms for unsupervised disentanglement.

\begin{figure}[t!]
\centering
\begin{subfigure}{0.3\textwidth}
  \centering
        \includegraphics[width=0.7\linewidth]{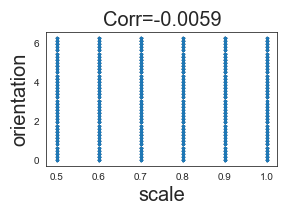} 
  \includegraphics[width=0.7\linewidth]{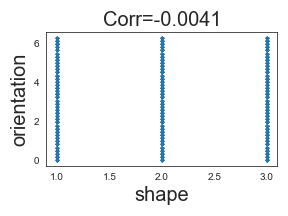}
  \includegraphics[width=0.7\linewidth]{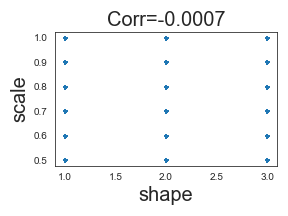}
  \caption{Disentangled and \\ uncorrelated (\gls{IOSS}=0.13)\label{fig:uncorr_ys}}
  \label{fig:uncorr_ys}
\end{subfigure}
\begin{subfigure}{0.3\textwidth}
  \centering
        \includegraphics[width=0.7\linewidth]{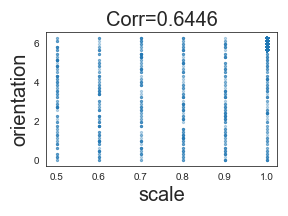}
  \includegraphics[width=0.7\linewidth]{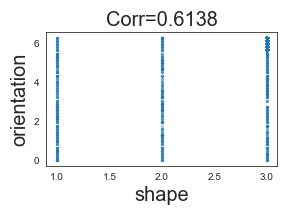}
  \includegraphics[width=0.7\linewidth]{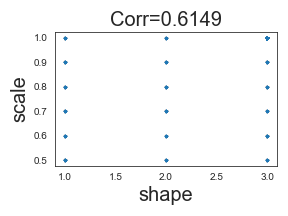}
  \caption{Disentangled but highly \\correlated (\gls{IOSS}=0.14)\label{fig:corr_ys}}
  \label{fig:corr_ys}
\end{subfigure}
\begin{subfigure}{0.3\textwidth}
  \centering
        \includegraphics[width=0.7\linewidth]{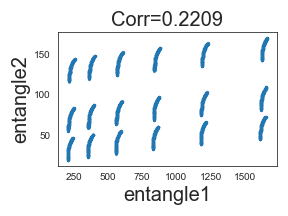} 
  \includegraphics[width=0.7\linewidth]{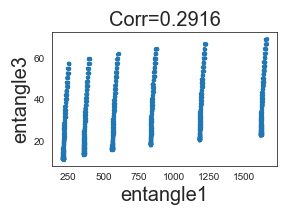}
  \includegraphics[width=0.7\linewidth]{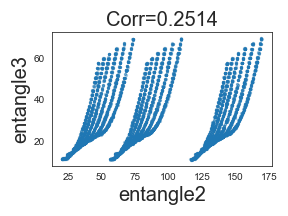}
  \caption{Entangled but with \\low correlations (\gls{IOSS}=0.27)\label{fig:entangle_ys}}
  \label{fig:entangle_ys}
\end{subfigure}
\caption{Disentangled features have independent support even though
they may be correlated. Moreover, the independence of support can
distinguish disentangled and entangled features. This figure
illustrates how entangled and disentangled features differ using
pairwise scatter plots.  \Cref{fig:uncorr_ys} considers the ground
truth features (shape, scale, orientation) of the dSprites dataset.
These features are disentangled. They also have independent support;
e.g., conditional on `scale', the set of values that `orientation' can
take does not change with `scale.' Visually, these  disentangled
features have scatter plots that occupy rectangular (or
hyperrectangular) regions. \Cref{fig:corr_ys} considers the same
features but in a subset of the dSprites dataset where the features
are correlated. These features, though correlated, are still
disentangled; they also have independent support.
\Cref{fig:entangle_ys} considers three entangled features, each of
which is a nonlinear transformation of the three ground-truth
features. These features are not disentangled. Their supports are also
not independent. Conditional on `entangle1', the possible values
`entangle2' can take depends on the value of `entangle1.'}
\label{fig:disentanglement-truefeatures}
\end{figure}

\subsection{The Causal Definition of Disentanglement}

\label{subsec:disentangle-def}

We begin with reviewing the causal definition of disentanglement. A
representation is (causally) \emph{disentangled} if (1) it represents
features that can generate the original data and (2) its 
dimensions correspond to features that do not causally affect each
other~\citep{Suter2019}. \Cref{defn:causal-disentangle} casts these
requirements in terms of a \gls{SCM}~\citep{Pearl2011}.

\begin{defn}[(Causal) Disentanglement~\citep{Suter2019}]
\label{defn:causal-disentangle} A
$d$-dimensional representation, $\mbZ=(Z_1,\ldots, Z_d)$, is 
disentangled if it represents features that generate the object of
interest $\mbX=(X_1,
\ldots, X_m)$ according to the following causal
\gls{SCM}:
\begin{align}
\mbC &\leftarrow u_{\mbC}, \nonumber\\
Z_j&\leftarrow f^z_j(\mbC, u_{z,j}), &j=1,\ldots,d,\label{eq:no-causal-g}\\
X_l &\leftarrow f^x_l(\mbZ, u_{\mbX}), &l=1,\ldots,m,\label{eq:g-generate-x}
\end{align}
where $\mbC=(C_1, \ldots, C_K)$ denotes an $K$-dimensional unobserved
confounder that can affect $\mbZ$ and hence induce correlations among
its components $(Z_1, \ldots, Z_d)$. \Cref{fig:causalrep-scm}
describes this \gls{SCM} using a causal DAG.\footnote{This causal
\gls{SCM} allows $Z_j$ to depend on only a subset of the $K$
confounder dimensions because the constraint imposed by a \gls{SCM}
resides in the absence of causal connections.}
\end{defn}

\begin{figure}
\centering
\begin{subfigure}{0.43\textwidth}
\centering
\begin{adjustbox}{height=6cm}
\includegraphics{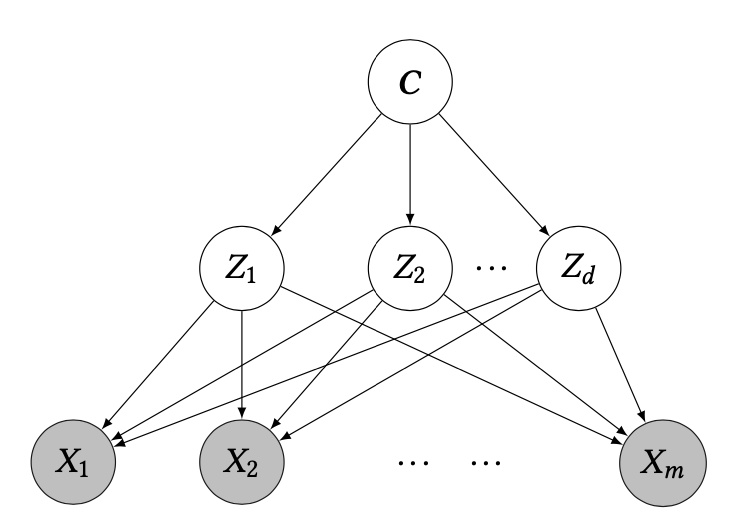}
\end{adjustbox}
\end{subfigure}\caption{The \gls{SCM} of unsupervised 
disentanglement~\citep{Suter2019}.\label{fig:disentanglement-scm}}
\end{figure}

As an example to illustrate this definition, we return to the dataset
of animal images in \Cref{sec:intro}. If we had access to the $d$
ground-truth features that generated the dataset (e.g., fur color,
number of legs), we could create a $d$-dimensional representation that
concatenates them. Such a representation would be disentangled, as the
ground-truth features do not causally affect each other. That said,
they may still be correlated due to some unobserved confounders
(a.k.a., common causes). For example, the animal's genetic information
may be an unobserved confounder; the genes can causally affect these
features, and thus induce correlations among them.

Notably, the causal definition of disentanglement opens up the
possibility to consider correlated features that are disentangled. We
work with this disentanglement definition in the rest of this section.

\subsection{Measuring Disentanglement with Observational Data}

\label{subsec:disentangle-metric}

How do we assess the disentanglement of a representation? In
the $\rmdo$ notation, a representation $\mbZ$ is  disentangled
if, for $j=1, \ldots, d,$
\begin{align}
\label{eq:disentangle-do-intervention}
P(Z_j\g \rmdo(\mbZ_{-j}=\mbz_{-j})) = P(Z_j), \forall
\mbz_{-j}\in \mathcal{Z}_{-j},
\end{align}
where $P(\cdot\g \rmdo(\cdot))$ denotes the intervention
distribution~\citep{Pearl2011}. We cannot, however, estimate $P(Z_j\g
\rmdo(\mbZ_{-j}))$ in practice. The reason is that we only have
access to an observational dataset of the representation
$\{\mbz_i\}_{i=1}^n = \{f(\mbx_i)\}_{i=1}^n$. This dataset is
generally insufficient for identifying $P(Z_j\g
\rmdo(\mbZ_{-j}))$ due to the unobserved
common cause~$\mbC$~\citep{Pearl2011,imbens2015causal}. This
difficulty of estimating $P(Z_j\g \rmdo(\mbZ_{-j}))$ is a core challenge
in assessing disentanglement.

To tackle this challenge, we consider the possibility of finding
observable implications of disentanglement. These observable
implications will be necessary conditions associated with
\Cref{eq:disentangle-do-intervention}. Though possibly insufficent,
they can still help us reject representations that are not
disentangled. Specifically, when a representation violates these
observable implications, it cannot be  disentangled.

\subsubsection{Observable implications of disentanglement: The
independent support \\condition}

We turn to the specific observable implications that underpin our
approach to disentanglement. We will prove that if a representation is
disentangled, then its different dimensions must have independent
support under a standard positivity condition.

\begin{thm}[Disentanglement $\Rightarrow$ Independent support]
\label{thm:disentanglement-support}
Assume the unobserved common cause $\mbC$ satisfies a positivity
condition: for all $j$, we have $P(Z_j\g \mbC) > 0$ iff $P(Z_j) >
0$. Then the support of the interventional distribution coincides with
that of the observational distribution:
\begin{align}
\label{eq:support-interventional-obs}
\supp(Z_j \g
\mathrm{do}(Z_{j'}=z_{j'})) = \supp(Z_j \g Z_{j'}=z_{j'}), 
\end{align}
where $j, j'\in\{1, \ldots, d\}$, $j\ne j'$, and the density at
$z_{j'}$ is nonzero, $p(z_{j'})>0$. As a consequence, different
dimensions of a disentangled representation $\mbZ=(Z_1, \ldots, Z_d)$
must have independent support:
\begin{align}
\supp(Z_1, \ldots, Z_d) = \supp(Z_1)\times \cdots \times \supp(Z_d),\label{eq:indep-support}\\
\supp(Z_j\g Z_\cS) = \supp(Z_j) \text{ for all
$\cS\subseteq \{1,\ldots, d\}\backslash j$}.\nonumber
\end{align} 
\end{thm}
\begin{proof}
We focus on proving \Cref{eq:support-interventional-obs}. (The remainder of
the proof is in \Cref{sec:disentanglement-support-proof}.)

First notice that the positivity condition on $\mbC$ guarantees that
$\supp(\mbC)=\supp(\mbC\g Z_j=z_j)$ for all $p(z_j)>0$:
\begin{align}
\mathbb{I}\{P(\mbC\g Z_j)>0\} 
&= \mathbb{I}\left\{\frac{P(\mbC)P(Z_j\g\mbC)}{P(Z_j)}>0\right\}\nonumber\\
&=\mathbb{I}\left\{P(\mbC)>0\right\}\mathbb{I}\left\{P(Z_j\g\mbC)>0\right\}\nonumber\\
&=\mathbb{I}\left\{P(\mbC)>0\right\}\mathbb{I}\left\{P(Z_j)>0\right\}\nonumber\\
&=\mathbb{I}\left\{P(\mbC)>0\right\},\label{eq:positivity-C}
\end{align}
where the third equality is due to the positivity condition.
Therefore, \Cref{eq:support-interventional-obs} holds because
\begin{align*}
&\mathbb{I}\{P(Z_j\g \rmdo(Z_{j'}))>0\} \\
&=\mathbb{I}\left\{\int P(Z_j\g Z_{j'}, \mbC) P(\mbC)\dif \mbC>0\right\} \\
&=\int \mathbb{I}\left\{P(Z_j\g Z_{j'}, \mbC)>0\right\} \mathbb{I}\left\{P(\mbC)>0\right\}\dif \mbC \\
&=\int \mathbb{I}\left\{P(Z_j\g Z_{j'}, \mbC)>0\right\} \mathbb{I}\left\{P(\mbC\g Z_{j'})>0\right\}\dif \mbC \\
&=\mathbb{I}\left\{\int P(Z_j\g Z_{j'}, \mbC) P(\mbC\g Z_{j'})\dif \mbC>0\right\} \\
&=\mathbb{I}\{P(Z_j\g Z_{j'})>0\},
\end{align*}
where the third equality is due to \Cref{eq:positivity-C}.
\end{proof}

The intuition behind \Cref{thm:disentanglement-support} is that two
variables can have their support depend on each other in only two
ways: (1) the causal connection between them, and (2) their common
cause induces a dependence among their supports. Therefore, two
causally disconnected variables must have independent support under
the positivity condition, which guarantees the common cause does not
causally affect their support. Positivity is also known as the overlap
condition and it is a standard assumption in causal inference with
observational data~\citep{imbens2015causal,pearl1995causal}. It
implies that the unobserved common cause $\mbC$ can induce
correlations among $Z_j$'s, though it can not affect their support.
Any combination of $Z_j$'s values must be possible.

\Cref{thm:disentanglement-support} shows that the absence of causal
arrows between two variables implies that their support must be
independent, even in the presence of unobserved confounders. In other
words, the relationship between the support of the variables is robust
to unobserved confounding. It implies that any  disentangled
$\mbZ$ must satisfy the independent support condition despite of the
unobserved confounder $\mbC$.

Consider compactly supported representations: each dimension of the
representations must be supported on a closed and bounded interval.
Then ``independent support'' implies a hyperrectangular joint support. As an
example, \Cref{fig:uncorr_ys,fig:corr_ys} illustrate the scatterplots
of two disentangled discrete features; they occupy rectangular
regions, regardless of their correlations. \Cref{fig:entangle_ys}
illustrate the scatter plots of two entangled features, which occupy
curvilinear, non-rectangular regions.

\Cref{thm:disentanglement-support} is most suitable for detecting
entangled representations whose different dimensions causally affect
each other's support. For example, in a dataset of animals, the
feature ``the animal being a rabbit'' can affect the support of the
feature ``the background being the sea'' because rabbits can not swim.
These two features are entangled; they also violate the independent
support condition. In the next sections, we will leverage these
observations stemming from \Cref{thm:disentanglement-support} to
develop a disentanglement metric.

\glsreset{IOSS}

\subsubsection{Assessing disentanglement with the IOSS}

\glsreset{IOSS}

To assess disentanglement, we build on
\Cref{thm:disentanglement-support} to develop the \gls{IOSS}. It
quantifies the extent to which a representation violates the
independent support condition.

\glsreset{IOSS}

\begin{defn}[\Gls{IOSS}]
Suppose a representation $\mbZ$ has bounded support
and~$\sup~Z_j~-~\inf~Z_j~>~0, j=1,\ldots, d$. Then the \gls{IOSS} of
$\mbZ$ is the Hausdorff distance between the joint support of $(Z_1,
\ldots, Z_d)$ and
the product of each individual's support:
\begin{align*}
&\gls{IOSS}(Z_1, \ldots, Z_d) \\
&\triangleq d_H(\supp(\bar{Z}_1, \ldots, \bar{Z}_d), \supp(\bar{Z}_1)\times
\cdots \supp(\bar{Z}_d))\\
&=d(\supp(\bar{Z}_1)\times\cdots \supp(\bar{Z}_d), \supp(\bar{Z}_1,
\ldots, \bar{Z}_d)),
\end{align*}
where $\bar{Z}_j = (Z_j-\inf Z_j)/(\sup Z_j - \inf Z_j)$ is the
standardized $Z_j$, and $d_H(\cdot,\cdot)$ is the Hausdorff
distance.\footnote{The Hausdorff distance between sets $\cX$ and $\cY$
is
\begin{align*}
d_H(\cX, \cY) = \max\{d(X, Y), d(Y,X)\},
\end{align*}
where
\begin{align*}
d(x,Y) &= \inf\{d(x,y)\g y\in Y\},\\
d(X, Y)&=\sup\{d(x,Y)\g x\in X\}.
\end{align*}
} The second equality is due to $\supp(Z_1, \ldots, Z_d)\subseteq
\supp(\bar{Z}_1)\times\cdots \times \supp(\bar{Z}_d).$
\end{defn}

\gls{IOSS} calculates the distance between the current support of the
representation $\mbZ$ and the (fictitious) support of $\mbZ$ if it
were disentangled. The larger the \gls{IOSS} is, the more entangled
the representation is. When $\mbZ$ is  disentangled,
$\gls{IOSS}(\mbZ)=0$.

\gls{IOSS} focuses on representations with bounded support, which
allows us to standardize each $Z_j$ to have $\inf \bar{Z}_j=0$ and
$\sup \bar{Z}_j=1$. This standardization step makes \gls{IOSS}
scale-invariant; it also makes the different $Z_j$ equivalent in their
importance in \gls{IOSS}.

To compute \gls{IOSS} in practice, we can directly compute the support
of discrete-valued representations and hence their \gls{IOSS}. For
continuous-valued representations, we compute the sample \gls{IOSS}. 
Given $n$ samples, $\{\mbZ_i\}_{i=1}^n =
\{(Z_{i1},
\ldots, Z_{id})\}_{i=1}^n$, we rescale the representation to the unit
hypercube, 
\[\bar{Z}_{ij} = \frac{Z_{ij} -
\min_{i=1, \ldots, n} Z_{ij}}{\max_{i=1, \ldots, n} Z_{ij} -
\min_{i=1, \ldots, n} Z_{ij}},\]
and compute the sample \gls{IOSS} 
as follows:
\begin{align*}
\widehat{\gls{IOSS}}(\{\mbZ_i\}_{i=1}^n) = \max_{k=1,\ldots, K}\min_{i=1,\ldots, n}(U_k - Z_{ij})^2,
\end{align*}
where $U_k\stackrel{iid}{\sim}\mathrm{Unif}[0,1]^d, k=1, \ldots, K,$
are $K$ uniform draws from $\prod_{j}\supp(\bar{Z}_{ij})$. To improve
the robustness of \gls{IOSS} in practice, one may also replace the
$\max, \min$ with the $(100-\alpha)$th and $\alpha$th quantile, where
$\alpha$ is a small positive number.

The sample \gls{IOSS} is comparable for representations with a fixed
sample size $n$, dimension $d$, and uniform draws $K$. As
$n\rightarrow \infty$ and $K/n\rightarrow \infty$, the sample
$\gls{IOSS}$ approaches
\gls{IOSS} almost surely,
$\widehat{\gls{IOSS}}(\{\mbZ_i\}_{i=1}^n)\rightarrow
\gls{IOSS}(\mbZ).$

To see \gls{IOSS} in action, we compute the sample \gls{IOSS} for the
three representations in \Cref{fig:disentanglement-truefeatures} with
$K=10^d\cdot n$. The sample \gls{IOSS} of the entangled features in
\Cref{fig:entangle_ys} is substantially higher than that of the
disentangled features in \Cref{fig:corr_ys,fig:uncorr_ys}. In
\Cref{sec:ioss-empirical}, we will show that \gls{IOSS} better
distinguishes entangled and disentangled representations than existing
disentanglement metrics.

\subsection{Learning Disentangled Representations with \gls{IOSS}}
\label{subsec:ioss-learn}

To learn disentangled representations, we propose to use the sample
\gls{IOSS} as a regularizer in unsupervised representation learning.

Before presenting the algorithm, we first establish the
identifiability of representations with independent support. This
result will license the use of \gls{IOSS} as a regularizer for
unsupervised disentanglement.
 \begin{thm}[Identifiability of representations with independent
support]\footnote{We are grateful to Shubhangi Ghosh, Luigi Gresele,
Julius von Kugelgen, Michel Besserve, and Bernhard Sch\"{o}lkopf for
pointing out a counterexample to a previous version of the theorem. }
\label{thm:indep-support-id} For any two $d$-dimensional
representations, $\mbZ=f(\mbX)=(Z_{1},\ldots, Z_{d})$ and
$\mbZ'=f'(\mbX) = (Z'_{1},
\ldots, Z'_{d})$, such that 
\begin{enumerate}
    \item The representations $\mbZ,\mbZ'$ both satisfy the independent support condition and have compact support in $\mathbb{R}^d$, i.e. 
    \begin{align}
        \supp(Z_1,\ldots, Z_d) &= [a_1, b_1]\times \cdots \times [a_d, b_d],\\
        \supp(Z'_1,\ldots, Z'_d) &= [a'_1, b'_1]\times \cdots [a'_d, b'_d],
    \end{align}
    where $a_j\leq b_j, a'_j\leq b'_j, j=1, \ldots, d$,
    \item The functions $f, f'$ are differentiable, and $f =
    \mathbf{h} \circ f'$ with
    $\mathbf{h}:\mathbb{R}^d\rightarrow\mathbb{R}^d$ for some
    bijective differentiable $\mathbf{h} = (h_1, \ldots, h_d)$ .
    Further, the dependency graph of $h_i(Z_1,
    \ldots, Z_d)$ on $Z_j$ (i.e. whether $h_i(Z_1, \ldots, Z_d)$
    nontrivially depends on $Z_j$) remains the same across all values
    of $\mbZ_{-j} = (Z_1, \ldots, Z_{j-1}, Z_{j+1}, \ldots, Z_d)$
    within the support, i.e.
    \begin{align}
    \label{eq:h-assumption}
        \mathbf{1}\left\{\frac{\partial h_i(Z_1, \ldots, Z_d)}{\partial Z_j} = 0 \right\} \text{is a constant function in $\mbZ_{-j}$} \qquad \forall i,j\in \{1, \ldots, d\}.
    \end{align}


    \item All components of $\mathbf{h}$ being invertible, i.e. $h_j(Z'_1, \ldots, Z'_d)$ is an invertible function with respect to each of $Z'_k, k\in \{1,\ldots, d\}$, for all values $\mathbf{Z'_{-k}}$ within the support and for all $j$,
    \item For each $i\in \{1, \ldots, d\}$, there exist an open set $\overline{\mathcal{Z}_{-i}}\subset \supp(\mbZ'_{-i})$, such that 
    \begin{align}
    \label{eq:trivial-h}
    \cap_{j \ne l} \{Z'_i: h_j(Z'_i\s \mbZ'_{-i}=\mbz'_{-i})\in [a_j, b_j]\}\not\subset \{Z'_i: h_l(Z'_i\s \mbZ'_{-i}=\mbz'_{-i})\in [a_l, b_l]\}, \text{ $\forall \mbz_{-i}\in\overline{\mathcal{Z}_{-i}}$},
    \end{align}
\end{enumerate}
then we have
\begin{align*}
Z_{1},\ldots, Z_{d} = \mathrm{perm}(q_1(Z'_{1}), \ldots, q_d(Z'_{d})),
\end{align*}
where the $q_{j}$ are continuous bijective function with a compact
domain in $\mathbb{R}$. (The proof is in
\Cref{sec:indep-support-id-proof}.)
\end{thm}


To understand the intuition behind \Cref{thm:indep-support-id}, we
consider a toy example of a two-dimensional compactly supported
representation $(Z_1, Z_2)$ with independent support: $Z_1\in [1,2]$,
$Z_2\in [0,2]$. Next consider an entanglement of this representation
$(Z'_1, Z'_2)$, which is a bijective transformation of $(Z_1, Z_2)$:
\[Z'_1 = Z_1 + Z_2, \qquad Z'_2 = Z_1 - Z_2.\] 
We will show that $(Z'_1,Z'_2)$ does not have independent support, then the
support of $Z_1-Z_2$ depends on the value of $Z_1 + Z_2$. To see why,
consider the case when $Z_1 + Z_2 = 4$, then we must have $Z_1 = Z_2 =
2$ due to the support constraints on $Z_1, Z_2$. Hence $Z_1-Z_2=0$,
thus the support of $Z_1-Z_2$ is $\{0\}$. Following a similar
argument, the support of $Z_1-Z_2$ is $\{1\}$ when $Z_1+Z_2=1$.
Therefore, the support of $Z_1-Z_2$ depends on values of $Z_1+Z_2$,
and hence they have dependent support.

The proof of \Cref{thm:indep-support-id} relies on a similar argument
as in the toy example. It builds on the following observation. Suppose
a variable $W\in \mathbb{R}$ has support $[U_W, L_W]$. Then the
support of $g(W)\in \mathbb{R}$ is $[\min\{g(U_W), g(L_W)\},
\max\{g(U_W), g(L_W)\}]$ if $g$ is continuous and bijective (subject
to other assumptions of \Cref{thm:indep-support-id}). In other words,
the end points of $g(W)$'s support are functions of end points of
$W$'s support. Therefore, any function $Z_{W'}$ that non-trivially
depends on a second variable $W'$ will violate the independent support
condition; the support of $Z_{W'}(W)$ will depend on $W'$. In this
case, the mappings between two representations can only involve
coordinate-wise mappings (e.g., permutation and coordinate-wise
bijective mappings); the mapping cannot involve multiple dimensions of
the representation.

This argument requires several assumptions in addition to independent
support. Assumption 2 posit that the dependency graph of $\mathbf{h}$
does not change across values. To study the dependency structure of
$\mathbf{h}$, it thus is sufficient to study the boundary of the
support of $\mbZ, \mbZ'$. Loosely, this assumption guarantees that
``$h_i$ nontrivially depends on two or more of its arguments'' implies
that ``$h_i$ nontrivially depends on two or more of its arguments at
the boundary of the support.'' Assumption 3 assumes the invertibility
of the $\mathbf{h}$ functions, which facilitates the discussion on the
dependency structure of $\mathbf{h}$, namely how $Z'_j$ depends on
$Z_{j'}$ for $j,j'\in \{1, \ldots, d\}$. Assumption 4 ensures that the
dependency structure of $\mathbf{h}(\cdot)$ has observable
implications on the support boundary of $\mbZ, \mbZ'$; no support
boundary is vacuous in the sense that it does not depend on any
components of the $\mathbf{h}(\cdot)$ function.

\Cref{thm:indep-support-id} establishes the identifiability of
representations with independent support. It focuses on
representations that generate the same $\sigma$-algebra; two
representations $\mbZ, \mbZ'$ satisfy $\sigma(\mbZ)=\sigma(\mbZ')$ if
there exists a bijective function $L$ such that $\mbZ = L(\mbZ')$.
Representations with the same $\sigma$-algebra are
``information-equivalent''; they capture the same information about
the raw data $\mbX$. Among these information-equivalent
representations, there is a unique representation (up to
coordinate-wise bijective transformations) that has independent
support. Together with ``disentanglement $\Rightarrow$ independent
support'' (\Cref{thm:disentanglement-support}), this unique
representation must also be  disentangled (if~it~exists).

We now illustrate disentangled representation learning with
\gls{IOSS}. Begin with an unsupervised representation learning
objective $L(\{x_i\}_{i=1}^n\s
\{\mbz_i\}_{i=1}^n)$, and use the sample \gls{IOSS} as a
regularizer:
\begin{align}
L_{\gls{IOSS}}(&\{x_i\}_{i=1}^n\s
\{\mbz_i\}_{i=1}^n)
=-\underbrace{L(\{x_i\}_{i=1}^n\s
\{\mbz_i\}_{i=1}^n)}_{\text{representation learning loss}} + \lambda\cdot \widehat{\gls{IOSS}}(\{\mbz_i\}_{i=1}^n),\label{eq:ioss-obj}
\end{align}
As a concrete example, the representation learning objective can be
the \gls{VAE} objective, a popular generative modeling approach to
unsupervised representation learning \citep{kingma2014auto}:
\[L(\{x_i\}_{i=1}^n\s \{\mbz_i\}_{i=1}^n)= -\E{q_\theta(\mbz_i\g \mbx_i)}{\log p_\theta(\mbz_i, \mbx_i)-\log q_\theta(\mbz_i\g \mbx_i)}\]
where we adopt the (approximate) posterior mean of the latent
variables as the representation $\mbz_i = \E{q(\mbz_i\g
\mbx_i)}{\mbz_i\g
\mbx_i}$, and where $\lambda>0$ is a regularization parameter.

To perform stochastic optimization on the $L_{\gls{IOSS}}$ objective,
we subsample batches $\{\mbz_i\}_{i\in \mathcal{B}}$ of the data and
update the parameters with the gradients of
$L_{\gls{IOSS}}(\{\mbz_i\}_{i\in \mathcal{B}})$. \Cref{alg:ioss}
summarizes this algorithm for learning disentangled representations
with
\gls{IOSS}. In \Cref{sec:ioss-empirical}, we will show that it can
produce representations that are more disentangled.

\begin{algorithm}
\small{
\SetKwData{Left}{left}\SetKwData{This}{this}\SetKwData{Up}{up}
\SetKwFunction{Union}{Union}\SetKwFunction{FindCompress}{FindCompress}
\SetKwInOut{Input}{input}\SetKwInOut{Output}{output}
\Input{
Training data $\{\mbx_i\}_{i=1}^n$; Test data
$\{\tilde{\mbx}_i\}_{i=1}^n$}
\Output{ Representations of the test data $\{\tilde{\mbz}_i\}_{i=1}^{n'}$}
\BlankLine
Minimize \Cref{eq:ioss-obj} to obtain $q_\theta$;
\BlankLine
Compute representations for the test data: $\tilde{\mbz}_i = \E{q_\theta(\mbz_i\g \mbx_i)}{\mbz_i\g
\mbx_i}.$
\caption{Disentangled representation learning with \gls{IOSS}}
\label{alg:ioss}}
\end{algorithm}


\subsection{Empirical Studies of IOSS}
\label{sec:ioss-empirical}

We study \gls{IOSS} in unsupervised image datasets. We find that
\gls{IOSS} is more effective at distinguishing between disentangled
and entangled representations than other unsupervised disentanglement
metrics. Unsupervised representation learning with the \gls{IOSS}
penalty results in representations with better disentanglement.

\subsubsection{Can \gls{IOSS} distinguish entangled and disentangled
representations?}

\label{subsec:disentangle-measure}

We first study whether \gls{IOSS} can distinguish entangled and
disentangled representations.

\parhead{Generating entangled/disentangled representations.} Focusing
on datasets with correlated ground truth features, we subsample each
dataset so that the correlation among the ground-truth features is
$>0.8$. Given the subsampled dataset, we take the ground-truth
features as disentangled representations. We then generate entangled
representations by applying nonlinear transformations to the
ground-truth features. Given ground-truth features $Z_1, \ldots, Z_d$,
we construct entangled representations by setting $Z'_i= f(Z_1,
\ldots, Z_d\s \theta_i) + Z_i$ with $\theta_i\sim
\mathrm{Unif}([-2.5,2.5])$, where $f(\cdot\s\theta)$ is a third-order
polynomial with coefficients $\theta$.

\parhead{Evaluation metrics.} For each pair of generated entangled and
disentangled representations, we calculate different disentanglement
scores for each. We report the percentage of entangled and disentangled
pairs that are correctly labeled by the metric. For example, if
\gls{IOSS} returns a smaller value for the disentangled representation
than the entangled one in the pair, then it is deemed correct
labeling. Similar procedures apply to other unsupervised disentangled
metrics. A metric can effectively distinguish entanglement and
disentanglement if this percentage is high.

\parhead{Competing methods.} We compare \gls{IOSS} with existing
unsupervised disentanglement metrics that do not rely on ground truth
features: total correlation~\citep{chen2018isolating} and Wasserstein
dependency~\citep{xiao2019disentangled}. We also compared with an
oracle supervised disentanglement metric: the intervention robustness
score~\citep{Suter2019}; it targets the same causal disentanglement
definition as \gls{IOSS} but relies on ground truth features.

\parhead{Results.} \Cref{tab:tab_disentangle_measure} presents the
results. The \gls{IOSS} outperforms baseline unsupervised
disentanglement metrics in distinguishing disentanglement from
entanglement.
\Cref{fig:disentangle-measure-1,fig:disentangle-measure-2,fig:disentangle-measure-3,fig:disentangle-measure-4}
also show that \gls{IOSS} can better separate disentangled and
entangled representations than existing unsupervised disentanglement
metrics.


\begin{table*}[t]
\begin{center}
\begin{tabular}{lcccc}
\toprule
{} &  mpi3d &  smallnorb &  dsprites &  cars3d \\
\midrule
IOSS (this paper)          &  \bfseries{0.998(0.045)} &      \bfseries{0.968(0.176)}&     \bfseries{0.980(0.140)} &  0.892(0.311) \\
Total Correlation          &  0.858(0.349) &      0.070(0.255) &     0.162(0.369) &   0.090(0.286) \\
Wasserstein Dependency     &  0.956(0.205) &      0.478(0.500) &     0.310(0.463) &   \bfseries{0.964(0.186)} \\
Intervention Robustness (oracle)               &  1.000(0.000) &      1.000(0.000) &     1.000(0.000) &   1.000(0.000) \\
\bottomrule
\end{tabular}
\caption{\gls{IOSS} outperforms existing disentanglement metrics in
distinguishing entangled and disentangled features. The table presents
the proportion of disentangled/entangled pairs that the metric
correctly distinguish. Intervention Robustness Score is an oracle
metric as it makes use of ground-truth features.(Higher is better.)
\label{tab:tab_disentangle_measure}}
\end{center}
\end{table*}


\begin{figure}[t]
  \centering
  \captionsetup[subfigure]{justification=centering}
\begin{subfigure}{0.31\textwidth}
  \centering
    \includegraphics[width=\linewidth]{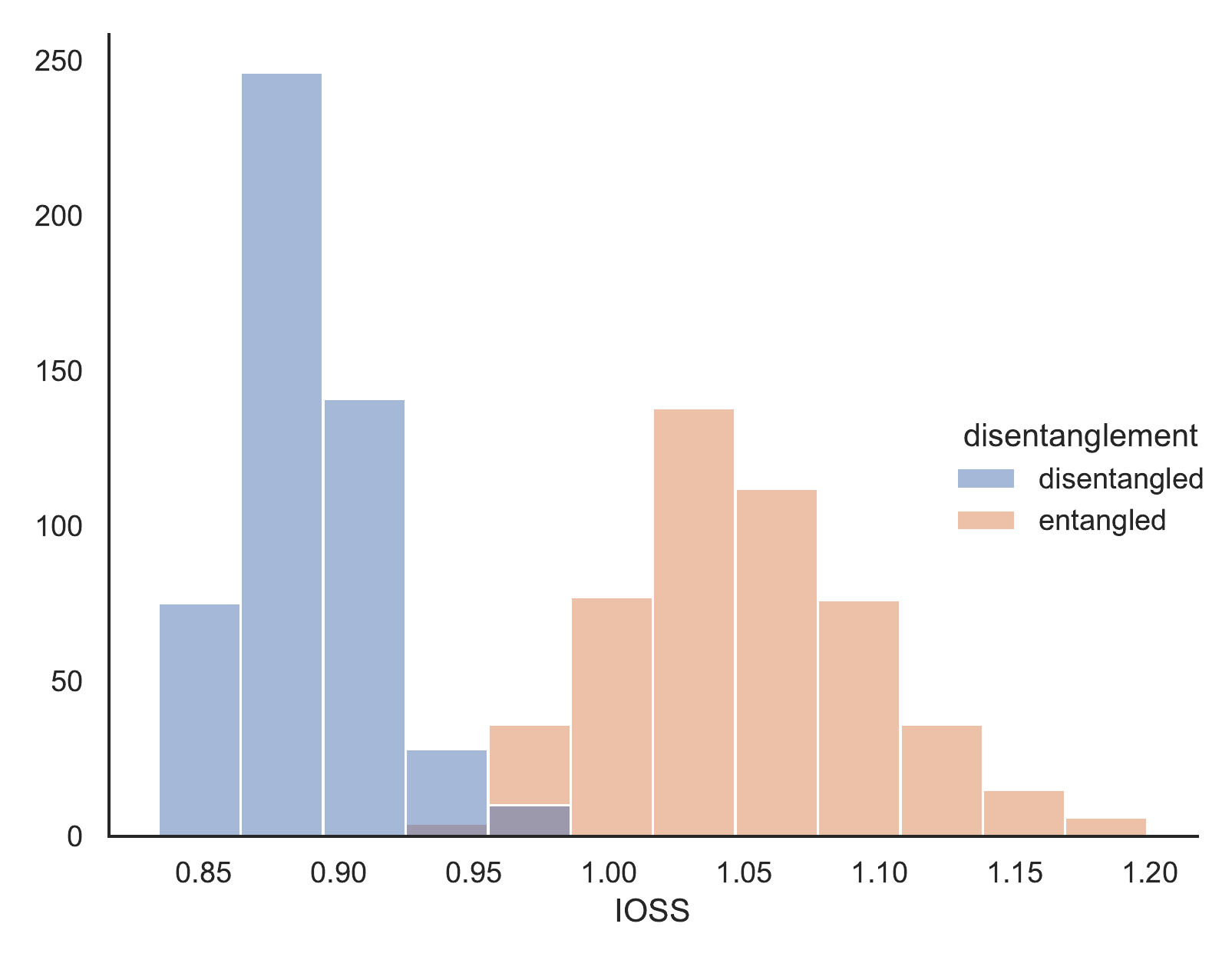}
  \caption{\gls{IOSS}}
\end{subfigure}
\begin{subfigure}{0.31\textwidth}
  \centering
    \includegraphics[width=\linewidth]{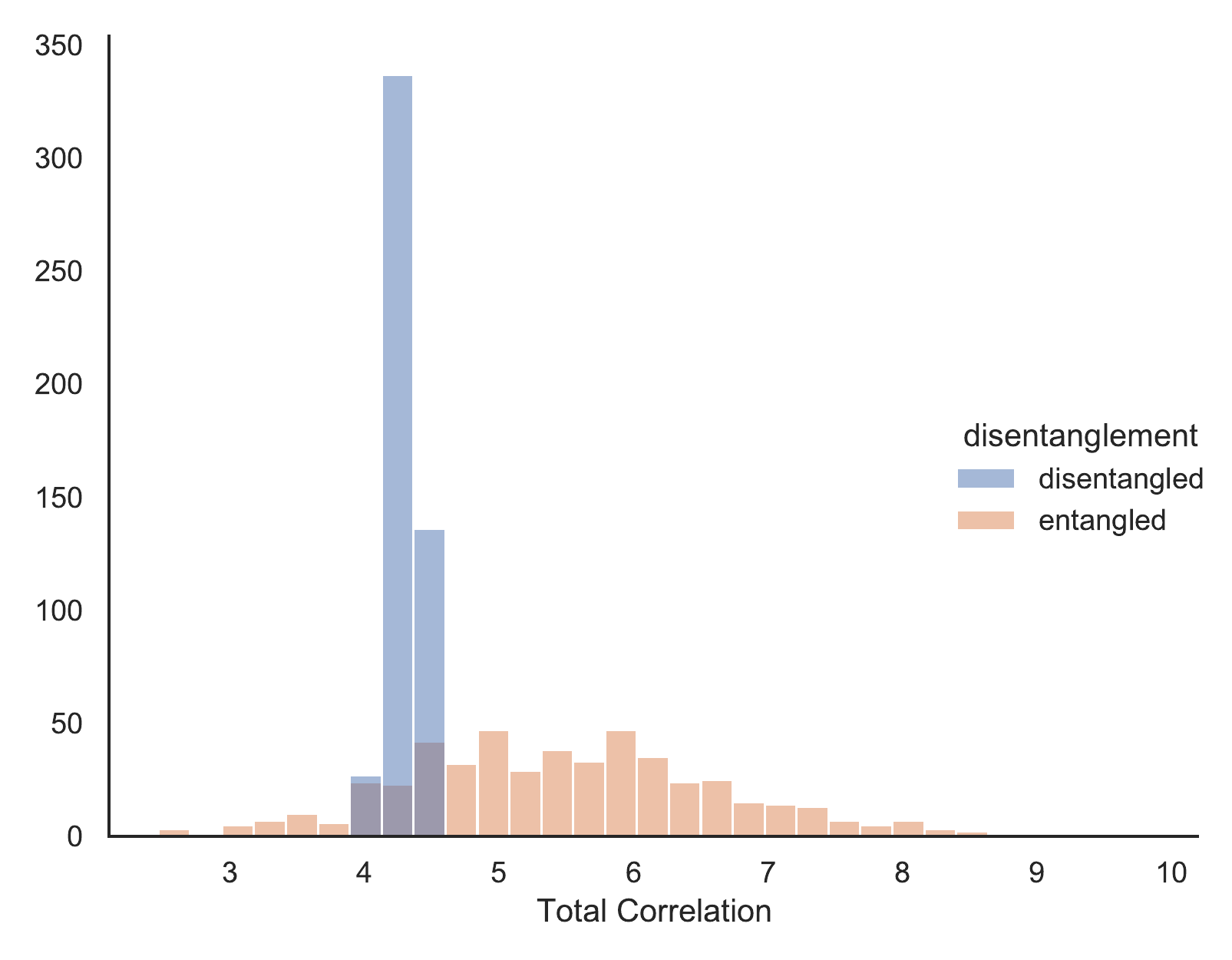}
  \caption{Total Correlation}
\end{subfigure}
\begin{subfigure}{0.31\textwidth}
  \centering
    \includegraphics[width=\linewidth]{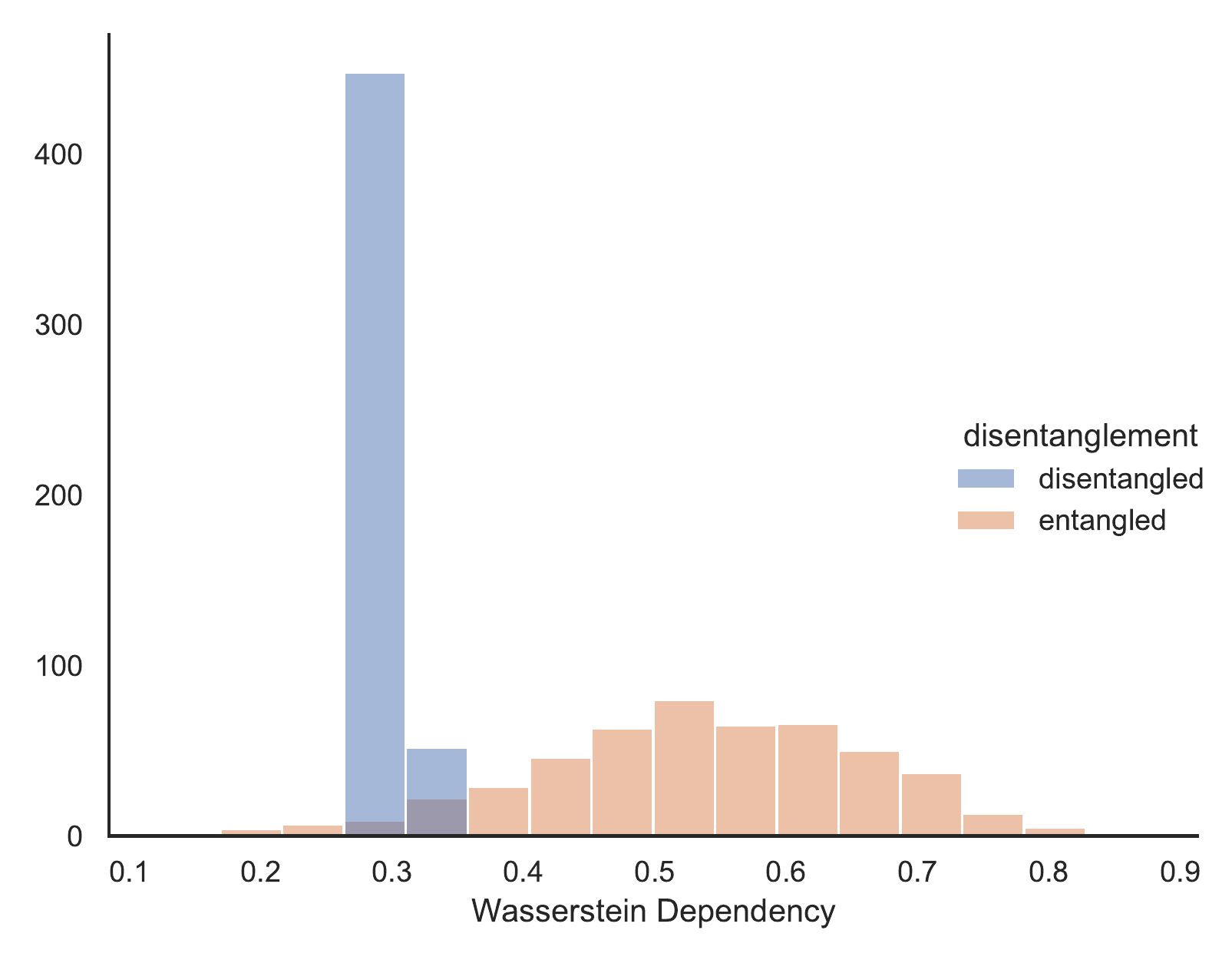}
  \caption{Wasserstein Dependency}
\end{subfigure}
\caption{IOSS can better distinguish entangled and disentangled
representations than existing unsupervised disentanglement metrics on
the mpi3d dataset. \label{fig:disentangle-measure-1}}
\end{figure}

\subsubsection{Does the \gls{IOSS} regularizer encourage disentangled
representations?}

\label{subsubsec:disentangle-learn}

We next apply the \gls{IOSS} penalty to learn disentangled
representations via \gls{VAE}s across all fours datasets: mpi3d,
smallnorb, dsprites, and cars3d. We work with subsampled datasets such
that the ground truth features are highly correlated (with correlation
$\approx 0.8$) following a similar process as in
\Cref{subsec:disentangle-measure}. We then fit \gls{VAE}s with an
increasingly strong regularization with the \gls{IOSS} penalty. We
evaluate the disentanglement of the learned representation using the
supervised disentanglement score of intervention robustness
score~\citep{Suter2019}; a higher score implies better
disentanglement.

\parhead{Competing methods.} We compare \gls{VAE}+\gls{IOSS} with
other unsupervised disentanglement algorithms including classical
\gls{VAE}~\citep{kingma2014auto},
FactorVAE~\citep{kim2018disentangling},
betaVAE~\citep{higgins2017beta}, and
betaTCVAE~\citep{chen2018isolating}.

\parhead{Disentanglement of \gls{VAE}+\gls{IOSS}-learned
representations.} \Cref{fig:ioss-rep-irs} shows that
\gls{VAE}+\gls{IOSS} produce more causally disentangled
representations than existing unsupervised disentanglement algorithms
when ground truth features are highly correlated. These results
corroborate \Cref{fig:entangle_ys,fig:corr_ys}. When features are
highly correlated, they may still be disentangled. Hence enforcing
statistical independence like beta-TCVAE and FactorVAE in this setting
does not perform disentanglement. This setting violates the core
assumption of beta-TCVAE and FactorVAE that the disentangled factors
are independent; i.e., their joint distribution is a product of their
marginals. Thus beta-TCVAE and FactorVAE do not perform as well in
disentanglement with correlated features. In contrast, VAE with
\gls{IOSS} regularization can accommodate correlated factors, as the
causal graph \Cref{fig:causalrep-scm} allows an unobserved common
cause between the factors.

Moreover, \Cref{fig:regularization-ioss} shows that, across datasets,
the \gls{VAE}+\gls{IOSS}-learned representation becomes increasingly
disentangled as we regularize with the \gls{IOSS} penalty. It implies
that the \gls{IOSS} penalty indeed encourages learning causally
disentangled representations. \Cref{fig:ioss-support} also illustrates
that regularizing for \gls{IOSS} indeed encourages independent support
across different dimensions of the learned representations.

\parhead{Tradeoff between informativeness and disentanglement.} We
finally evaluate the tradeoff between informativeness and
disentanglement in representation learning. We increase the
regularization strengths of \gls{IOSS} and evaluate the log-likelihood
of the fitted \gls{VAE}. \Cref{fig:ioss-tradeoff} shows that the
log-likelihood stays stable despite the increasing regularization for
disentanglement. While increasing the regularization strengths of
\gls{IOSS} encourages disentanglement, it does not compromise the
informativeness of the learned representations in \gls{VAE}. This
result suggests the orthogonality between the disentanglement
desiderata and other informativeness-related desiderata.


\begin{figure}[t]
\centering
\begin{subfigure}{0.3\textwidth}
\includegraphics[width=\linewidth]{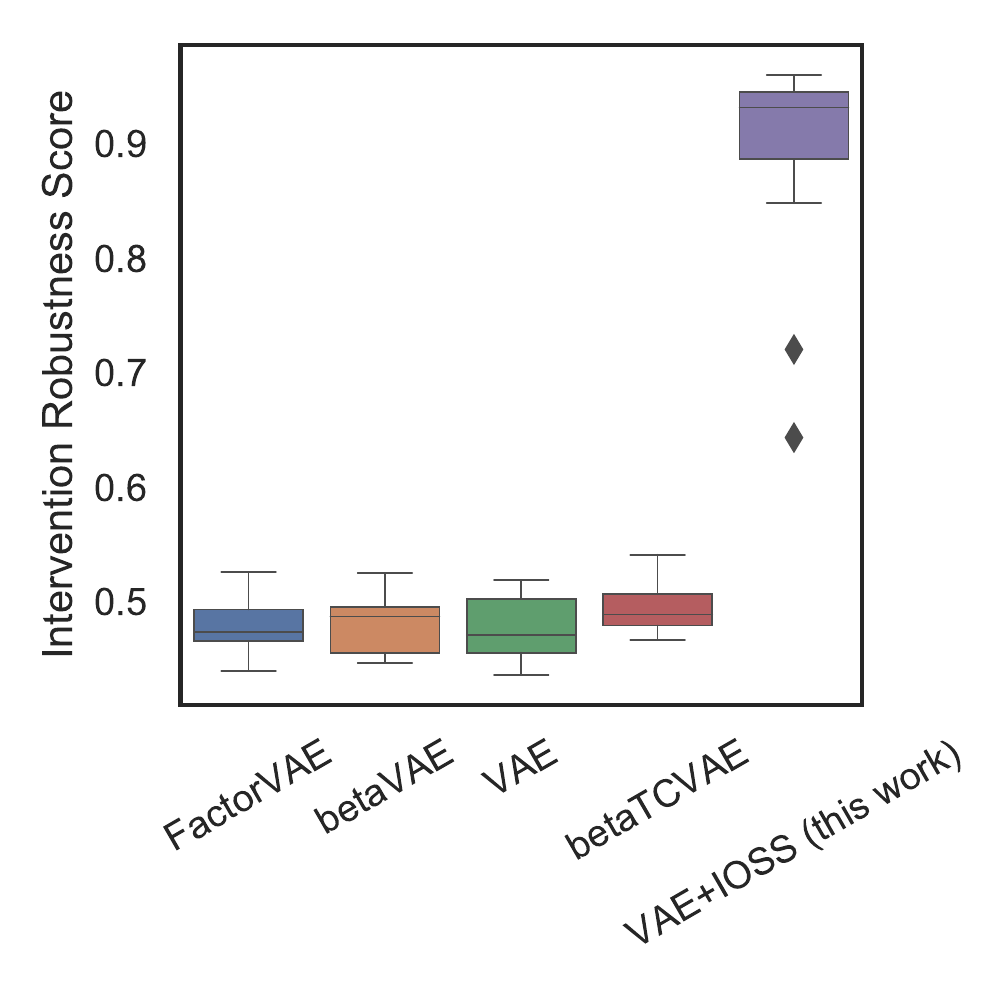}
\caption{Disentanglement of \gls{IOSS} learned representations
  \label{fig:ioss-rep-irs}}
  \end{subfigure}
\begin{subfigure}{0.3\textwidth}
\centering
\includegraphics[width=\linewidth]{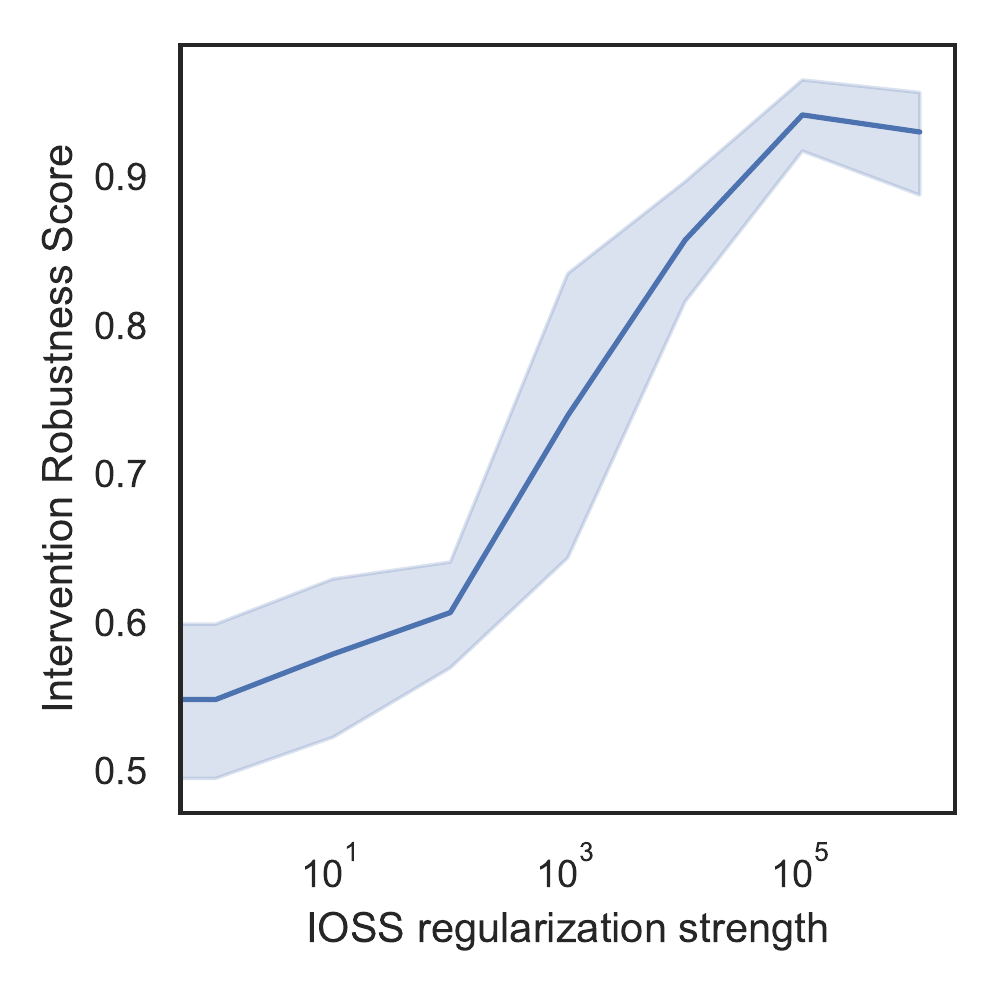} 
  \caption{Regularization with \gls{IOSS}
  penalty\label{fig:regularization-ioss}}
\end{subfigure}
\begin{subfigure}{0.3\textwidth}
\centering
\includegraphics[width=\linewidth]{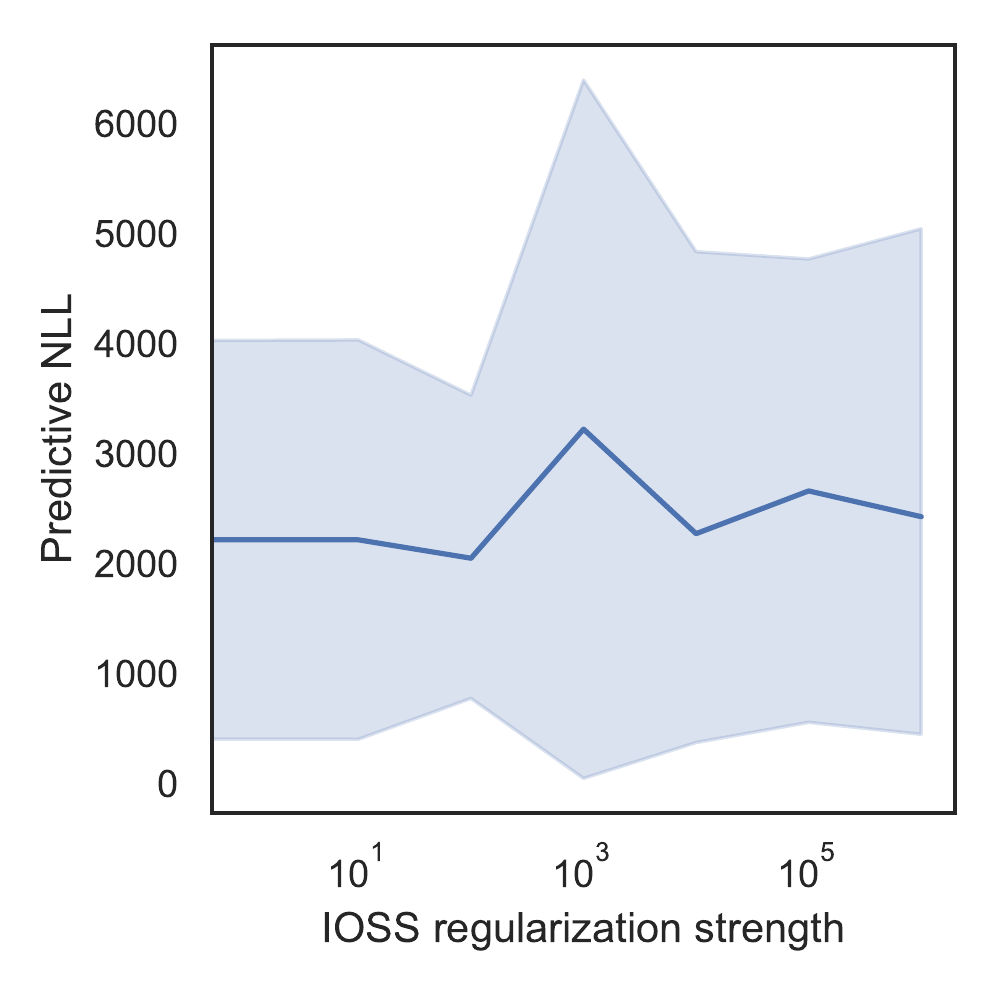} 
  \caption{Informativeness \& disentanglement tradeoff
  \label{fig:ioss-tradeoff}}
\end{subfigure}
\caption{(a) \gls{VAE}+\gls{IOSS} outperforms existing
unsupervised disentanglement algorithms in producing representations
that are more causally disentangled. (b) Regularizing for \gls{IOSS}
encourages learning disentangled representations. The intervention
robustness score of the representation increases as we increase the
regularization strengths of \gls{IOSS}. (c) There is no obvious
tradeoff between informativeness and disentanglement. The model fit
stays stable despite we increase the regularization strengths of
\gls{IOSS}.}
\end{figure}

\section{Discussion}
\label{sec:conclusion}
\glsreset{PNS}

We have shown that desiderata for representation learning, often
discussed informally but rarely defined formally, can be usefully
formalized from the perspective of causal inference. By studying the
observable implications of our causal definitions, we develop metrics
that evaluate representational desiderata and algorithms that enforce
these desiderata.

We have focused on two examples of representation learning desiderata:
(1) efficiency and non-spuriousness in supervised representation
learning, and (2) disentanglement in unsupervised representation
learning.  Our overall workflow can be summarized as follows:
``desiderata''$\rightarrow$``causal definitions'' $\rightarrow$
``observable implications'' $\rightarrow$ ``metrics and algorithms.''
We have shown that this workflow can lead to practical evaluation
metrics and representation learning algorithms.

Targeting the desiderata of efficiency and non-spuriousness in the
supervised setting, we view the representation as a cause of the label
and formalize these desiderata as the \gls{PNS} of the cause. Studying
the observable implications of \gls{PNS} enables us to measure
efficiency and non-spuriousness with observational data. It also
allows us to formulate representation learning as a task of finding
necessary and sufficient causes. We operationalize this task by
developing the CAUSAL-REP algorithm. En route, we develop
identification results for the \gls{PNS} given high-dimensional
(rank-degenerate) data.

To encourage unsupervised disentanglement, we again begin with a
causal definition, and we obtain metrics for disentanglement that
again rely on studying its observable implications. We capitalize on
the observation that causally disentangled variables---which have no
causal connections among each other---must have independent support
under a standard positivity condition. Based on this insight, we
develop the \glsreset{IOSS} \gls{IOSS} for assessing causal
disentanglement. We further establish the identifiability of
representations with independent support, which enables representation
learning with a \gls{IOSS} penalty.

For future work, one promising direction is to identify other
desiderata in representation learning (or machine learning tasks in
general) that may be formalized using causal notions. One may follow
the workflow to obtain evaluation metrics and learning algorithms for
the desiderata.

Another direction is to extend the metrics and algorithms developed in
this work. One may extend the CAUSAL-REP algorithm to general
anti-causal learning, without requiring the assumptions (e.g.,
pinpointability) required in \Cref{thm:do-prob-highdim}. Such an
extension would enable the treatment of reverse causal inference via
probabilities of causation. On the disentanglement side, \gls{IOSS}
can also be extended to learn primitives for compositional
generalization. Primitives share similar properties as disentangled
features; they cannot causally affect each other. In this sense,
\gls{IOSS}-type ideas can potentially offer new approaches to
compositional generalization.

\vspace{20pt}

\parhead{Acknowledgments. } We thank Shubhangi Ghosh, Luigi Gresele,
Julius von Kugelgen, Michel Besserve, and Bernhard Sch\"{o}lkopf for
useful discussions and for pointing out a counterexample to the
Theorem 11 in a previous version of the manuscript.

\clearpage
\putbib[BIB1]
\end{bibunit}

\appendix

\onecolumn

\textbf{\Large{Supplementary Materials}}
\appendix

\section{Proof of \Cref{thm:pns-id}}

\label{sec:thm-pns-id-proof}

\begin{proof}
The proof generalizes the proof of Theorem 9.2.10 in
\citet{Pearl2011}.

We first notice the following:
\begin{align}
P(Y(\mbZ=\mbz) = y) = P(Y(\mbZ=\mbz)=y, Y(\mbZ\ne\mbz) =y) + P(Y(\mbZ=\mbz)=y, Y(\mbZ\ne\mbz) \ne y)\label{eq:pns-po}
\end{align}

Next denote $\epsilon\triangleq P(Y(\mbZ=\mbz)\ne y, Y(\mbZ\ne\mbz)
=y)\geq 0$. Then we have
\begin{align}
P(Y(\mbZ\ne\mbz) =y) &= P(Y(\mbZ=\mbz)=y, Y(\mbZ\ne\mbz) =y) + P(Y(\mbZ=\mbz)\ne y, Y(\mbZ\ne\mbz) =y)\\
&= P(Y(\mbZ=\mbz)=y, Y(\mbZ\ne\mbz) =y) +\epsilon. \label{eq:pns-ne}
\end{align}

Substituting \Cref{eq:pns-ne} into \Cref{eq:pns-po} implies that
\begin{align}
P(Y(\mbZ=\mbz) = y) =P(Y(\mbZ\ne\mbz) =y)  + P(Y(\mbZ=\mbz)=y, Y(\mbZ\ne\mbz) \ne y) -\epsilon,
\end{align}
which implies
\begin{align}
P(Y(\mbZ=\mbz)=y, Y(\mbZ\ne\mbz) \ne y) 
= &P(Y(\mbZ=\mbz) = y) - P(Y(\mbZ\ne\mbz) =y) + \epsilon\\
= & P(Y=y\g \rmdo(\mbZ=\mbz)) - P(Y=y\g \rmdo(\mbZ\ne\mbz)) + \epsilon\\
\geq & P(Y=y\g \rmdo(\mbZ=\mbz)) - P(Y=y\g \rmdo(\mbZ\ne\mbz)).
\end{align}

The inequality becomes equality when $\epsilon=0$ under the
monotonicity condition in \Cref{thm:pns-id}.

\end{proof}

\section{The definition of functional interventions recovers backdoor
adjustment}

\label{sec:functional-interventions-supp}

The definition of functional
interventions~(\Cref{defn:functional-intervention}) recovers the
standard backdoor adjustment as special cases when the function
$f(\mbX)$ returns $\mbX$ or its subset (assuming the causal graph
\Cref{fig:causalrep-scm}).

For example, when $f(\mbX) = \mbX$, then
\begin{align*}
&P(Y\g\rmdo(f(\mbX)=\mbz)) = P(Y\g\rmdo(\mbX=\mbz)) \\
&= \int P(Y\g \rmdo(\mbX), \mbC) P(\mbX\g \mbC, \mbX=\mbz) P(\mbC)\dif \mbX\dif \mbC\\
&=\int P(Y\g \mbX) \cdot \delta_{\mbX=\mbz} \cdot P(\mbC)\dif \mbX\dif \mbC\\
&= P(Y\g \mbX=\mbz),
\end{align*}
where the third inequality is due to no unobserved confounding
between $\mbX$ and $Y$ in \Cref{fig:causalrep-scm}. 

As another special case, suppose the function $f(\mbX)$ returns the
subset of $\mbX=(X_1, \ldots, X_m)$ except $X_1$, i.e., $f(\mbX) =
(X_2, \ldots, X_m)$. Then \Cref{eq:functional-intervention-def}
recovers backdoor adjustment:
\begin{align*}
&P(Y\g\rmdo(f(\mbX)=\mbz)) = P(Y\g\rmdo((X_2, \ldots, X_m)=\mbz)) \\
&= \int P(Y\g \rmdo(\mbX), \mbC) P(\mbX\g \mbC, (X_2, \ldots, X_m)=\mbz) P(\mbC)\dif \mbX\dif \mbC\\
&=\int P(Y\g \mbX) \cdot \delta_{(X_2, \ldots, X_m)=\mbz} \cdot P(X_1\g \mbC) P(\mbC)\dif \mbX\dif \mbC\\
&=\int P(Y\g \mbX_1, (X_2, \ldots, X_m)=\mbz) \cdot P(X_1\g \mbC) P(\mbC)\dif X_1\dif \mbC\\
&=\int P(Y\g \mbX_1, (X_2, \ldots, X_m)=\mbz, \mbC) \cdot P(X_1\g (X_2, \ldots, X_m)=\mbz, \mbC) P(\mbC)\dif X_1\dif \mbC\\
&=\int P(Y,\mbX_1\g (X_2, \ldots, X_m)=\mbz, \mbC) P(\mbC)\dif X_1\dif \mbC\\
&=\int P(Y\g (X_2, \ldots, X_m)=\mbz, \mbC) P(\mbC)\dif \mbC,
\end{align*}
where the third inequality is due to no unobserved confounding between
$\mbX$ and $Y$ in \Cref{fig:causalrep-scm}, and the fifth inequality
is due to the conditional independence of $X_1, \ldots, X_m$ given
$\mbC$ in \Cref{fig:causalrep-scm}.

\section{Proof of \Cref{thm:do-prob-highdim}}

\label{sec:do-prob-highdim-proof}

\begin{proof}

We calculate the intervention distribution for functional
interventions $f(\mbX) = \tilde{f}((X_j)_{j\in S})$
\begin{align}
&P(Y\g \rmdo(f(\mbX)))\nonumber\\
& = \int P(Y\g (X_j)_{j\in S}, \mbC) P((X_j)_{j\in S}\g \mbC, f(\mbX)) P(\mbC)\dif (X_j)_{j\in S}\dif \mbC,\label{eq:id-do-req1}\\
& = \int P(Y\g (X_j)_{j\in S}, \mbC,  f(\mbX)) P((X_j)_{j\in S}\g \mbC, f(\mbX)) P(\mbC)\dif (X_j)_{j\in S}\dif \mbC,\label{eq:id-do-req1}\\
& = \int P(Y\g \mbC,  f(\mbX)) P(\mbC)\dif \mbC,\label{eq:id-do-req1}\\
& =\int P(Y\g f(\mbX), h(\mbX))\cdot P(h(\mbX))  \dif h(\mbX).
\end{align}

The first equation is due to the definition of functional
interventions on $f(\mbX)$~\citep{puli2020causal}. It is a soft
intervention on $\mbX$ with a stochastic policy conditional on the
parents of $\mbX$. The stochastic policy is $P(\mbX\g f(\mbX),
PA(\mbX))$ where $PA(\mbX)$ denotes the parents of $\mbX$.

The second equation is due to $P(Y\g (X_j)_{j\in S}, \mbC,
f(\mbX))=P(Y\g (X_j)_{j\in S}, \mbC)$.

\sloppy
The third equation is due to the observability and positivity
condition. Observability and positivity guarantee that $P((X_j)_{j\in
S}, \mbC)$ > 0 for all $(X_j)_{j\in S}, \mbC$. Thus we can calculate
$P(Y\g (X_j)_{j\in S}, \mbC,  f(\mbX))=P(Y\g (X_j)_{j\in S},
\mbC)$. Moreover, the two conditions imply that $P(f(\mbX),
\mbC)>0$ for all $f(\mbX),
\mbC$, and hence we can calculate $P((X_j)_{j\in S}\g \mbC, f(\mbX))$.

The fourth equation is due to the pinpointability condition.

\end{proof}

\section{Pinpointing the unobserved common cause $\mbC$}

\label{sec:pinpoint-C}

The CAUSAL-REP algorithm begins with a step of pinpointing the
unobserved common cause $\mbC$. In this step, we infer $\mbC$ from the
observational data $\mbX$, when $\mbC$ is low-dimensional and $\mbX$
is high-dimensional. Operationally, as $\mbC$ renders $\mbX=(X_1,
\ldots, X_m)$ conditionally independent, we infer $\mbC$ by fitting a
probabilistic factor model to $\mbX$,
\begin{align} p(\mbx_i, \mbc_i\s\phi) = p(x_{i1}, \ldots, x_{im},
\mbc_i\s \phi) = p(\mbc_i)\prod_{j=1}^m p(x_{ij}\g \mbc_i\s \phi).
\end{align}

Specifically, we consider \gls{VAE}, a flexible probabilistic factor
model~\citep{kingma2014auto},
\begin{align}
\mbC_i &\sim p(\mbc_i),\label{eq:vae-gen1}\\
\mbX_i \g \mbC_i &\sim p(\mbx_i\g \mbc_i\s \theta)=\mathrm{EF}(\mbx_i\g f_\theta(\mbc_i)\s \lambda_\theta),\label{eq:vae-gen2}
\end{align} 
where $\mathrm{EF}$ is an exponential family distribution, $\theta =
(f_\theta, \lambda_\theta)$ are the parameters, and $f_\theta:
\mathcal{\mbC} \rightarrow \mathcal{\mbX}$ is a flexible neural
network.

Next we infer $p(\mbc_i\g \mbx_i)$ using variational approximation
$p(\mbc_i\g \mbx_i)\approx q_{\phi^*}(\mbc_i\g \mbx_i)$ where
$q_{\phi^*}(\mbc\g\mbx)$ maximizes the
\gls{ELBO} objective of \gls{VAE}~\citep{blei2017variational},
\begin{align*}
q_{\phi^*}=\argmax_{q_\phi}\sum_{i=1}^n\left[\log p(\mbx_i, \mbc_i) - \log
q_\phi(\mbc_i\g \mbx_i)\right],
\end{align*}
and $q_\phi(\cdot\g \mbx)$ parametrizes a flexible family of
distributions with parameters $\phi$. For example, we can have
$q_\phi(\cdot\g \mbx) = \cN(Z_{1,\phi}(\mbx), Z_{2,\phi}(\mbx)\cdot
I)$ or even more flexible non-Gaussian distributions via normalizing
flows.

Finally, we set
\begin{align}
\label{eq:pinpoint-c-supp}
\mbC_i=h(\mbX_i) \approx
\E{q_{\phi^*}(\mbc_i\g \mbx_i)}{\mbC_i\g \mbX_i}, i=1, \ldots, n.
\end{align}

Though we use variational approximation for $p(\mbc_i\g \mbx_i)$,
\Cref{eq:pinpoint-c-supp} can often give a good approximation of
$\mbC_i$ when $\mathrm{dim}(\mbX)\gg\mathrm{dim}(\mbC)$, or more
precisely when the pinpointability condition
holds~\citep{chen2020structured,wang2019frequentist}.

\section{Calculating \gls{PNS} lower bounds with linear models}
\label{sec:calculate-pns-linear}

\allowdisplaybreaks
We calculate the lower bound of the conditional
efficiency and non-spuriousness \\$\underline{\gls{PNS}_n(f_j(\mbX), Y
\g f_{-j}(\mbX))}$ with the linear model (\Cref{eq:y-model-C}):
\begin{align}
&\underline{\gls{PNS}_n(f_j(\mbX), Y \g f_{-j}(\mbX))}\nonumber\\
&=\prod_{i=1}^n \int
\left[P(Y=y_i \g f_j(\mbX) = f_j(\mbx_i), f_{-j}(\mbX) =
f_{-j}(\mbx_i), \mbC)\right.\nonumber\\
&\quad\qquad\left.- P(Y=y_i \g f_j(\mbX) \ne f_j(\mbx_i),
f_{-j}(\mbX) = f_{-j}(\mbx_i), \mbC)\right]\cdot P(\mbC) \dif \mbC\nonumber\\
&=\prod_{i=1}^n\int  \left[\cN(y_i\s \beta_0 + \beta_j f_j(\mbx_i) + \sum_{j'\ne j} \beta_{j'} f_{j'}(\mbx_i) + \boldsymbol{\gamma}^\top \mbC, \sigma^2) \right.\nonumber\\
&\quad\qquad \left.- \cN(y_i\s \beta_0 + \beta_j \E{}{f_j(\mbx_i)} + \sum_{j'\ne j} \beta_{j'} f_{j'}(\mbx_i) + \boldsymbol{\gamma}^\top \mbC, \sigma^2)\right]\cdot P(\mbC) \dif\mbC\nonumber\\
&= \prod_{i=1}^n\int \left[\exp\left(-\frac{(\gamma^\top (\mbc_i-\mbC) +\epsilon_i )^2}{2\sigma^2} \right) \right.\nonumber\\
&\quad\quad\left.-\exp\left(-\frac{(\beta_j\cdot (f_j(\mbx_i) -\E{}{f_j(\mbx_i)})+ \gamma^\top (\mbc_i-\mbC) +\epsilon_i)^2}{2\sigma^2}\right)\right]\cdot P(\mbC) \dif \mbC \times (2\pi \sigma^2)^{-\frac{n}{2}}\\
&\approx \prod_{i=1}^n\int \left[\left(1-\frac{(\gamma^\top (\mbc_i-\mbC) +\epsilon_i )^2}{2\sigma^2} \right) \right.\nonumber\\
&\quad\quad\left.-\left(1-\frac{(\beta_j\cdot (f_j(\mbx_i) -\E{}{f_j(\mbx_i)})+ \gamma^\top (\mbc_i-\mbC) +\epsilon_i)^2}{2\sigma^2}\right)\right]\cdot P(\mbC) \dif \mbC \times (2\pi \sigma^2)^{-\frac{n}{2}}\label{eq:taylor-1}\\
&= \prod_{i=1}^n\int \left(\frac{(\beta_j\cdot (f_j(\mbx_i) -\E{}{f_j(\mbx_i)}))^2 + 2\cdot \beta_j\cdot (f_j(\mbx_i) -\E{}{f_j(\mbx_i)}) \cdot  (\gamma^\top (\mbc_i-\mbC) +\epsilon_i)}{2\sigma^2}\right)\nonumber\\
&\quad\quad \cdot P(\mbC) \dif \mbC\times (2\pi \sigma^2)^{-\frac{n}{2}}\\
&= \prod_{i=1}^n \left(\frac{(\beta_j\cdot (f_j(\mbx_i) -\E{}{f_j(\mbx_i)}))^2 + 2\cdot \beta_j\cdot (f_j(\mbx_i) -\E{}{f_j(\mbx_i)}) \cdot  (\gamma^\top (\mbc_i-\E{}{\mbC}) +\epsilon_i)}{2\sigma^2}\right)\nonumber\\
&\quad\quad \times (2\pi \sigma^2)^{-\frac{n}{2}}\\
&= \prod_{i=1}^n \exp\left(\frac{(\beta_j (f_j(\mbx_i) -\E{}{f_j(\mbx_i)}))^2 + 2 \beta_j (f_j(\mbx_i) -\E{}{f_j(\mbx_i)})   (\gamma^\top (\mbc_i-\E{}{\mbC}) +\epsilon_i)}{2\sigma^2}-1\right)\nonumber\\
&\quad\quad \times (2\pi \sigma^2)^{-\frac{n}{2}}\label{eq:taylor-2}\\
&=  \exp\left(\frac{\sum_{i=1}^n(\beta_j (f_j(\mbx_i) -\E{}{f_j(\mbx_i)}))^2 + 2 \sum_{i=1}^n\beta_j (f_j(\mbx_i) -\E{}{f_j(\mbx_i)})   (\gamma^\top (\mbc_i-\E{}{\mbC}))}{2\sigma^2}-n\right)\nonumber\\
&\quad\quad \times (2\pi \sigma^2)^{-\frac{n}{2}},\label{eq:pns-lower-last-step-supp}
\end{align}
where $\cN(\cdot)$ denotes the Gaussian density, $\epsilon_i = y_i -
(\beta_0 + \boldsymbol{\beta}^\top f(\mbx_i)  +
\boldsymbol{\gamma}^\top \mbc_i)$ is the residual of the regression in
\Cref{eq:y-model-C}. \Cref{eq:taylor-1,eq:taylor-2} make use of Taylor
approximation $\exp(x) \approx 1 + x$.
\Cref{eq:pns-lower-last-step-supp} makes use of $\sum_{i=1}^n
\epsilon_i\cdot (f_j(\mbx_i) -\E{}{f_j(\mbx_i)}) \approx
\E{}{\epsilon \cdot (f_j(\mbX) -\E{}{f_j(\mbX)})}=0$ in the regression
model.

Finally, in CAUSAL-REP
\Cref{alg:causal-rep,alg:causal-rep-unsupervised}, we often impose an
R-squared penalty to encourage the possitivity of $f_j(\mbX)$ given
$\mbC$. In these cases, we often have $\sum_{i=1}^n\beta_j\cdot
(f_j(\mbx_i) -\E{}{f_j(\mbx_i)}) \cdot  (\gamma^\top
(\mbc_i-\E{}{\mbC})) \approx \beta_j\mathrm{Cov}(f_j(\mbX),
\gamma^\top\mbC)\approx 0$. Thus, we can further approximate the
\gls{PNS} by
\begin{align}
&\underline{\gls{PNS}_n(f_j(\mbX), Y \g f_{-j}(\mbX))}\nonumber\\
&\approx  \exp\left(\frac{\sum_{i=1}^n(\beta_j\cdot (f_j(\mbx_i) -\E{}{f_j(\mbx_i)}))^2}{2\sigma^2}-n\right)\times (2\pi \sigma^2)^{-\frac{n}{2}},
\end{align}
and thus,
\begin{align}
\log \underline{\gls{PNS}_n(f_j(\mbX), Y \g f_{-j}(\mbX))}\approx  \left(\frac{\sum_{i=1}^n(\beta_j\cdot (f_j(\mbx_i) -\E{}{f_j(\mbx_i)}))^2}{2\sigma^2}\right)+\mathrm{constant},
\end{align}
Similarly, we can obtain the lower bound of the (unconditional)
efficiency and non-spuriousness, similar to
\Cref{eq:pns-lower-last-step-supp},
\begin{align}
&\log \underline{\gls{PNS}_n(f(\mbX), Y)}\nonumber\\
&\approx \left(\frac{1}{2\sigma^2}\sum_{i=1}^n\left[ (\sum_{j=1}^d\beta_j\cdot (f_j(\mbx_i) -\E{}{f_j(\mbx_i)}))^2 \right.\right.\nonumber\\
&\left.\left.\qquad+ 2\cdot \sum_{j=1}^d\beta_j\cdot (f_j(\mbx_i) -\E{}{f_j(\mbx_i)}) \cdot  \gamma^\top (\mbc_i-\E{}{\mbC})\right]\right) + \mathrm{constant}.
\end{align}

\section{Details of the unsupervised CAUSAL-REP algorithm}

\label{sec:unsupervised-causalrep-details}

We summarize the unsupervised CAUSAL-REP in
\Cref{alg:causal-rep-unsupervised}.


\begin{algorithm}[t]
\small{
\SetKwData{Left}{left}\SetKwData{This}{this}\SetKwData{Up}{up}
\SetKwFunction{Union}{Union}\SetKwFunction{FindCompress}{FindCompress}
\SetKwInOut{Input}{input}\SetKwInOut{Output}{output}
\Input{The observational training data (without labels)
$\{\mbx_i\}_{i=1}^{n}$; the probabilistic factor model that generates
the training data $P(\mbX,
\mbC)$}

\Output{CAUSAL-REP representation function $\hat{f}(\cdot)$}
\BlankLine

Augment the unsupervised training dataset into a supervised one $\{\{\mbx_{i}^u,
\mby_{i}^u\}_{u=1}^U\}_{i=1}^n$ following
\Cref{eq:unsupervised-aug-label};

Fit a probabilistic factor model (\Cref{eq:factor-1,eq:factor-2}) and
infer $\{\{p(\mbc^u_i\g \mbx^u_i)\}_{u=1}^U\}_{i=1}^n$;

\If{Pinpointability holds, i.e. $p(\mbc^u_i\g \mbx^u_i)$ is
close to a point mass for all $i,u$}{
\ForEach{training datapoint $i$}{
Pinpoint the unobserved common cause $\mbC$:
$\mbc^u_i=h(\mbx^u_i)
\triangleq
\E{}{\mbc^u_i\g \mbx^u_i}$ for all $u=1,
\ldots, U$;}

Maximize \Cref{eq:causalrep-unsupervised} to obtain the CAUSAL-REP
representation $\hat{f}$;}

\caption{CAUSAL-REP (Unsupervised)}
\label{alg:causal-rep-unsupervised}}
\end{algorithm}

\section{Details of the empirical studies for CAUSAL-REP and
additional empirical results}

\label{sec:experiments-supp}

\subsection{Details of \Cref{sec:pns-distinguish}}

\Cref{fig:pns-measurement-supp-1,fig:pns-measurement-supp-2} present
additional results for \Cref{sec:pns-distinguish}. As $Z_1$ and $Z_2$
become increasingly correlated, the (lower bounds of) unconditional
\gls{POC} of $Z_2$ for $Y_1$ also increase. It is also consistent with
the intuition: $Z_2$ is an increasingly better surrogate of $Z_1$ for
$Y_1$ give higher correlations between $Z_1$ and $Z_2$.

\begin{figure}[t]
  \centering
\begin{subfigure}{\textwidth}
  \centering
    \includegraphics[width=\linewidth]{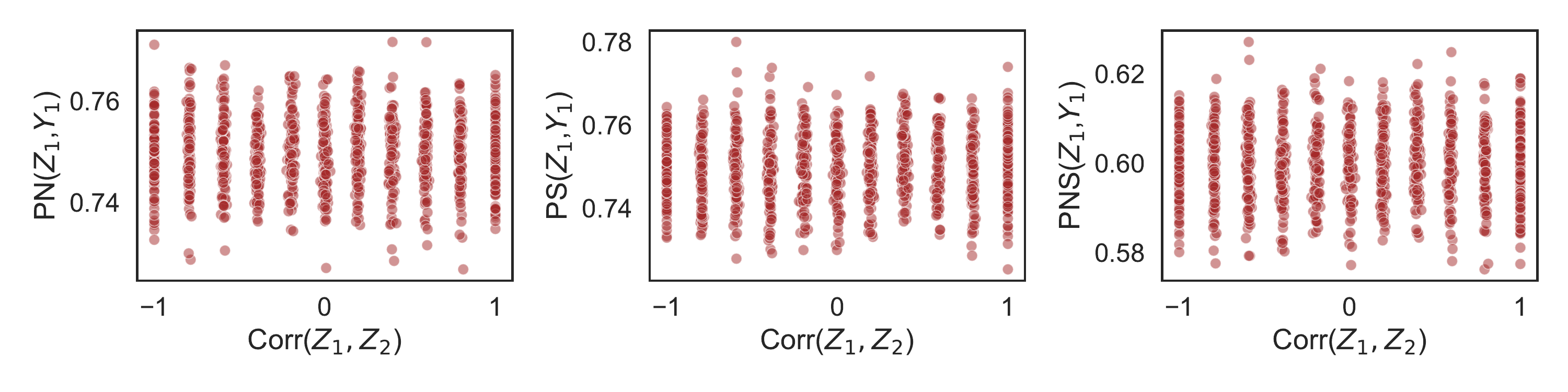}
  \caption{\gls{POC} of $Z_1$ for $Y_1$}
\end{subfigure}
\begin{subfigure}{\textwidth}
  \centering
    \includegraphics[width=\linewidth]{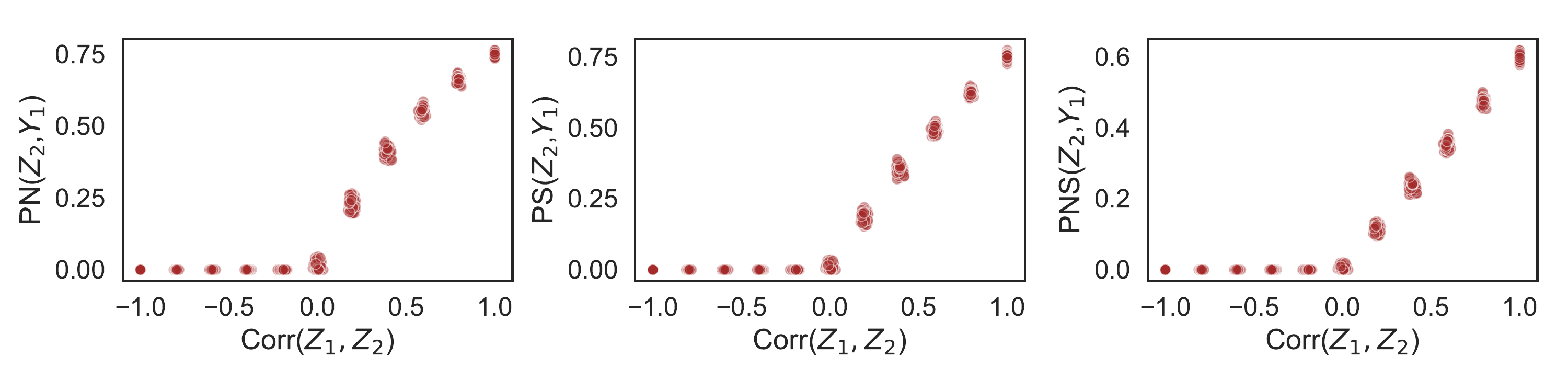}
  \caption{\gls{POC} of $Z_2$ for $Y_1$}
\end{subfigure}
\begin{subfigure}{\textwidth}
  \centering
    \includegraphics[width=\linewidth]{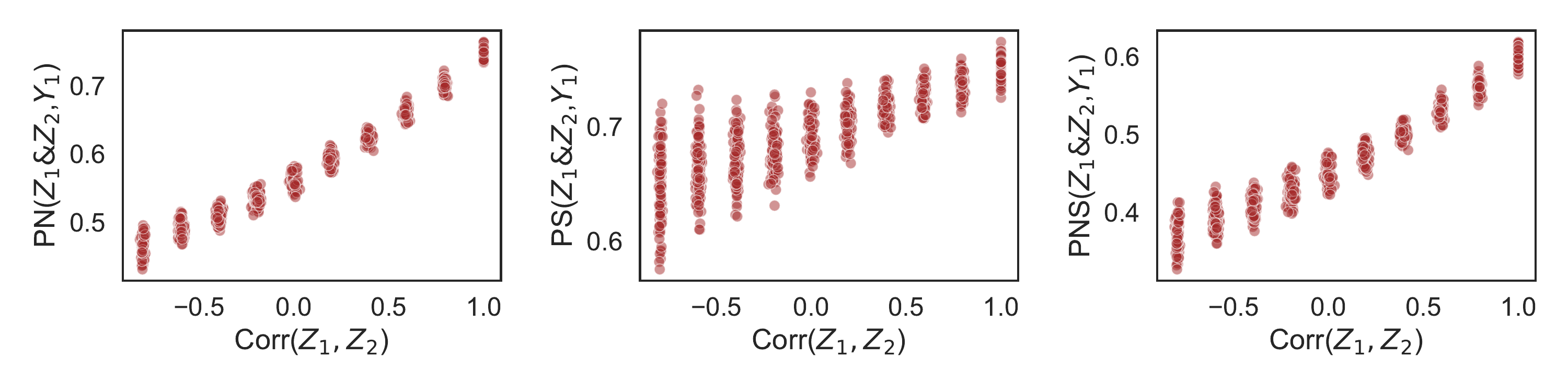}
  \caption{\gls{POC} of $Z_1 \& Z_2$ for $Y_1$}
\end{subfigure}
\caption{Lower bounds of probabilities of causation are consistent
with intuitive notions of feature necessity and sufficiency. As $Z_1$
and $Z_2$ become increasingly highly correlated, (a) the \gls{POC} of
$Z_1$ for $Y_1$ stays high, (b) the \gls{POC} of $Z_2$ for $Y_1$
starts to increase when the correlation turns positive, and (c) the
\gls{POC} of $Z_1\&Z_2$ increases.\label{fig:pns-measurement-supp-1}}
\end{figure}

\begin{figure}[t]
  \centering
\begin{subfigure}{\textwidth}
  \centering
    \includegraphics[width=\linewidth]{img/sec2-4-1-pns_distinguish/y2_z1.pdf}
  \caption{\gls{POC} of $Z_1$ for $Y_1$}
\end{subfigure}
\begin{subfigure}{\textwidth}
  \centering
    \includegraphics[width=\linewidth]{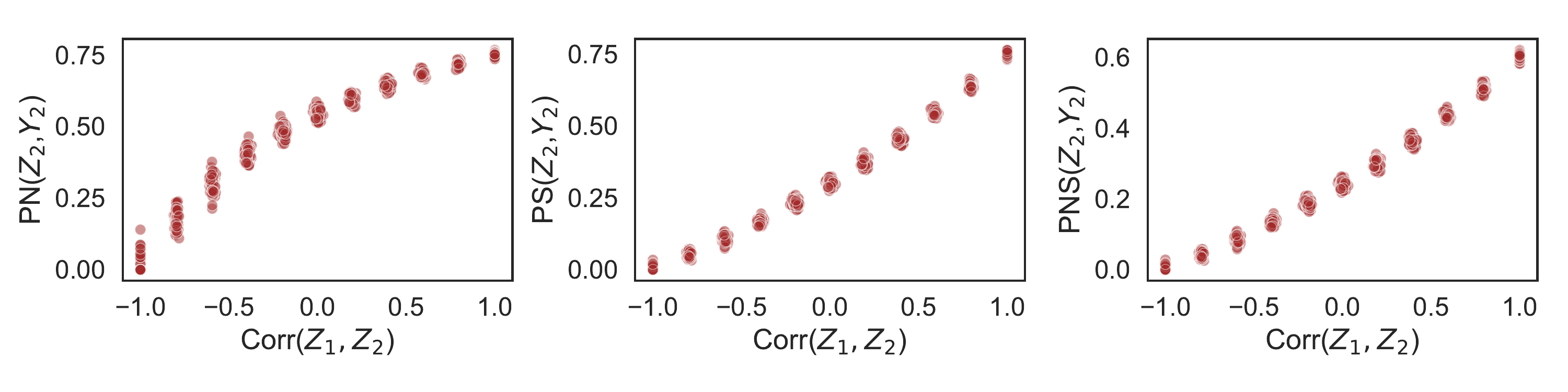}
  \caption{\gls{POC} of $Z_2$ for $Y_1$}
\end{subfigure}
\begin{subfigure}{\textwidth}
  \centering
    \includegraphics[width=\linewidth]{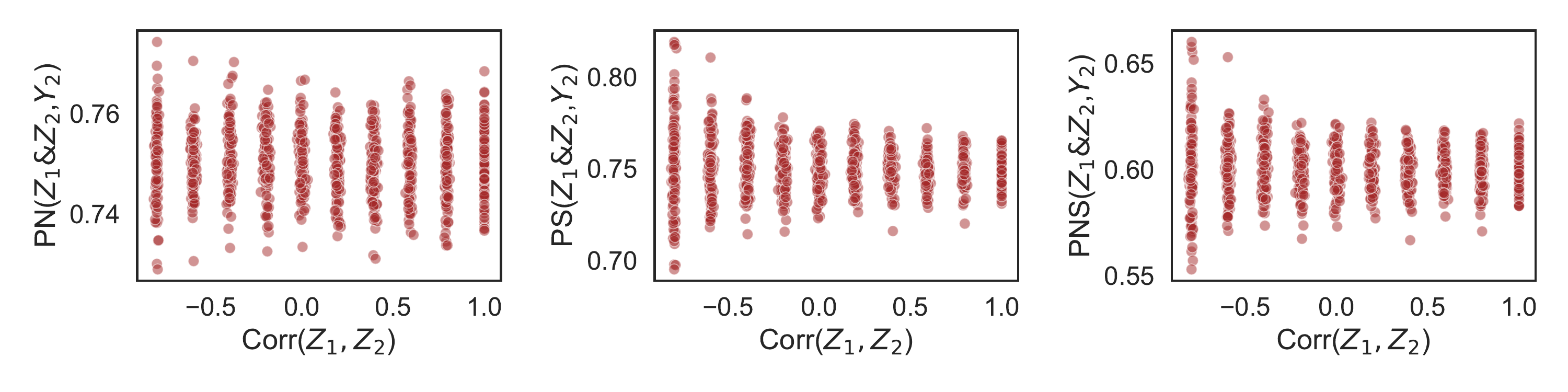}
  \caption{\gls{POC} of $Z_1 \& Z_2$ for $Y_1$}
\end{subfigure}
\caption{Lower bounds of probabilities of causation are consistent
with intuitive notions of feature necessity and sufficiency. As $Z_1$
and $Z_2$ become increasingly highly correlated, (a) the \gls{POC} of
$Z_1$ for $Y_2$ increases, (b) the \gls{POC} of $Z_2$ for $Y_2$
increases, and (c) the \gls{POC} of $Z_1\&Z_2$ stays high.
\label{fig:pns-measurement-supp-2}}
\end{figure}

\subsection{Details of the colored MNIST study}

\label{subsec:colored-mnist-supp}

We study CAUSAL-REP on the colored MNIST
dataset~\citep{arjovsky2019invariant}. The dataset builds on the
original MNIST data but color the image in a way that is highly
correlated with the digits label. We focus on the digits `3' and `8'
and colors `red' and `green.'

To create a training set, we color the `3' images in red with
probability $p$ and in green with probability $(1-p)$. Next, we color
the `8' images in red with probability $(1-p)$ and in green with
probability $p$. When $p\in[0,1]$ is large, then the color of the
image is highly correlated with the digit label in such a training
set. We further add noise to the ground truth digit label by randomly
flipping the labels with a probability of 0.25. The best possible
predictive accuracy is thus 0.75 (the yellow dashed line in
\Cref{fig:colored-mnist-supervised}).

The color of the images is a spurious feature in the training set; it
has a high correlation with the digit label but does not causally
determine the label. In contrast, the features of the digits
themselves are non-spurious features; they are highly correlated with
the digit label and can causally determine the label.

To create a test set, we color the images such that the color and
images are correlated oppositely. We color the `3' images in red with
probability $(1-p)$ and the `8' images in red with probability $p$. As
the color-image relationship is very different, a representation
learning algorithm will predict poorly in the test set if it only
captures color as a feature in the training set.

\subsection{Details of the reviews corpora study}

\Cref{tab:reviews-top-words} presents the most informative words of
the (positive or negative) ratings, suggested by the CAUSAL-REP
representation and the logistic regression coefficients. Across three
reviews corpora, logistic regression returns the spurious words
``as'', ``also'', ``am'', ``an'' as the top words. In contrast,
CAUSAL-REP extracts words that are more relevant for the ratings.


\allowdisplaybreaks
\begin{table}[t]
    \centering
    \begin{tabular}{p{0.1\linewidth}p{0.65\linewidth}p{0.2\linewidth}}
      Amazon &CAUSAL-REP  & Logistic Regression \\ \hline
      1& love\_this\_camera, recommend\_this\_camera, my\_first\_digital, great, best\_camera, camera\_if\_you, this\_camera\_and, camera\_have, excellent\_camera, camera\_bought\_this;
      & \multirow{5}{\linewidth}{am, an, also, as, love\_my, the\_tracfone, it\_real, which\_is, too, so\_much, is\_so\_much, which\_is\_pretty, nokia, ear, home, is\_must, for\_your, faster, must\_for, when\_use}\\
      2&this\_camera, camera, camera\_is, pictures, picture, the\_camera, digital, camera\_for, this\_camera\_is, digital\_camera;\\
      3 &really\_nice, hold\_the, excellent\_it, this\_one\_it, easy\_it, is\_superb, nice\_if, returning, too\_low, you\_need\_more; \\
      4 &with\_this, aa, took, came, yet, pictures\_of, camera\_in, computer, pictures\_in, for\_those;  \\
      5 &camera\_was, expect, the\_photos, by, camera\_are, blurry, sony, have\_an, had\_some, wife;  \\
    \end{tabular}

\vspace{5pt}

    \begin{tabular}{p{0.1\linewidth}p{0.65\linewidth}p{0.2\linewidth}}
      Tripadvisor &CAUSAL-REP  & Logistic Regression \\ \hline
      1& front\_door, typical, lady, room\_is, sized, yogurt, of\_italian, in\_we, was\_at, front\_desk\_to, was\_easy;
      & \multirow{5}{\linewidth}{am, as, an, also, we\_have, consideration, and\_nice, real, stay\_here\_again, good}\\
      2&lobby, man, directions, door, open, tried, seemed, with\_the, by\_the, almost;\\
      3 &man\_at, the\_price, is\_in, above\_and\_beyond, we\_checked\_out, in\_central\_rome, colliseum, great\_for, covered, place\_to\_sit; \\
      4 &the\_front, front, front\_desk, desk, the\_front\_desk, at\_the\_front, desk\_staff, front\_desk\_staff, desk\_was, front\_of;  \\
      5 &the\_people\_at, the\_people, stayed\_nights, people\_at, was\_very\_helpful, would\_not\_recommend, lovely\_and, great\_location, staff\_at, of\_my, pricey;  \\
    \end{tabular}

\vspace{5pt}
    \begin{tabular}{p{0.1\linewidth}p{0.65\linewidth}p{0.2\linewidth}}
      Yelp &CAUSAL-REP  & Logistic Regression \\ \hline
      1& but\_if, looking, what\_you, you\_go, but\_if\_you, that\_you, to\_do, lot, youll, try, own, do\_not;
      & \multirow{5}{\linewidth}{an, as, am, also, japanese\_food, because\_its, sure\_to\_get, ive, ive\_been, at\_night, up\_for\_it}\\
      2 & even\_if, your\_place, this\_restaurant, for\_you, you\_should, thank\_you, or\_if, thank, lamb, fabulous, is\_awesome;  \\
      3 & want\_the, then\_you, ahead, hollywood, dont\_want, suggest, with, please, check\_this\_place, all\_that;  \\
      4 & are, you\_like, if\_you\_like, here\_if, are\_looking, here, are\_in, is\_the\_place, you\_need, are\_looking\_for;  \\
      5 & if\_you, if, you, want, you\_want, you\_are, if\_you\_are, want\_to, if\_you\_want, your, you\_want\_to, you\_dont, dont;  \\
    \end{tabular}
    \caption{Across the Amazon, Tripadvisor, and Yelp reviews corpa,
    CAUSAL-REP learns representation that does not rely on the
    injected words that are spuriously correlated with the sentiment.
    The table shows the top 12 words identified by the
    five-dimensional representation from CAUSAL-REP; the five
    dimensions are not ordered. The representation obtained from
    logistic regression relies heavily on the spurious words ``am'',
    ``an'', ``also'', ``as.'' \label{tab:reviews-top-words}}
  
\end{table}

\subsection{Details of the colored and shifted MNIST study}

\label{subsec:colored-shifted-mnist-supp}

The colored and shifted MNIST dataset is created similarly as in the
colored MNIST. We consider four shifts: $(dx, dy) = (0,1), (1,0),
(0,1), (1,1)$---labeled 0,1,2,3---and shift the images such that the
shift label is highly correlated with the digit label in the training
set.

We evaluate the non-spuriousness of the representations using a
downstream prediction task with domain shift. Given a labeled training
set where the digit features are no longer correlated with the
spurious features, we learn a mapping from the representations to the
digit label. A representation can only predict the digit label well if
it captured the non-spurious digit features in unsupervised
representation learning.

\section{Details of \Cref{fig:disentanglement-truefeatures}}

Here we include the full pairwise scatter plot
(\Cref{fig:supp-disentanglement-truefeatures}) of the three sets of
features in \Cref{fig:disentanglement-truefeatures}.

\begin{figure}[t]
\centering
\begin{subfigure}{0.3\textwidth}
  \centering
    \includegraphics[width=\linewidth]{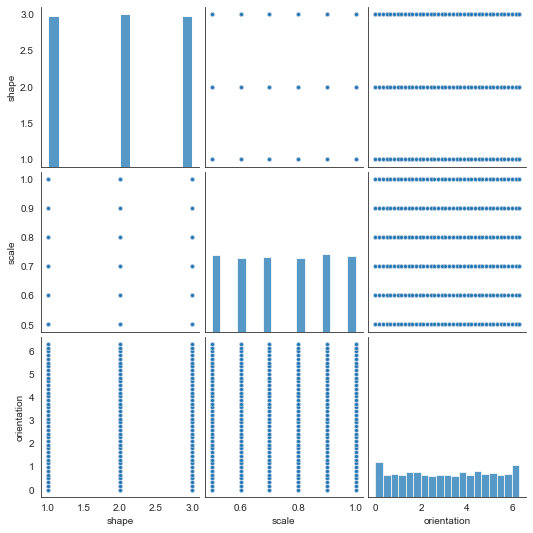} 
  \includegraphics[width=\linewidth]{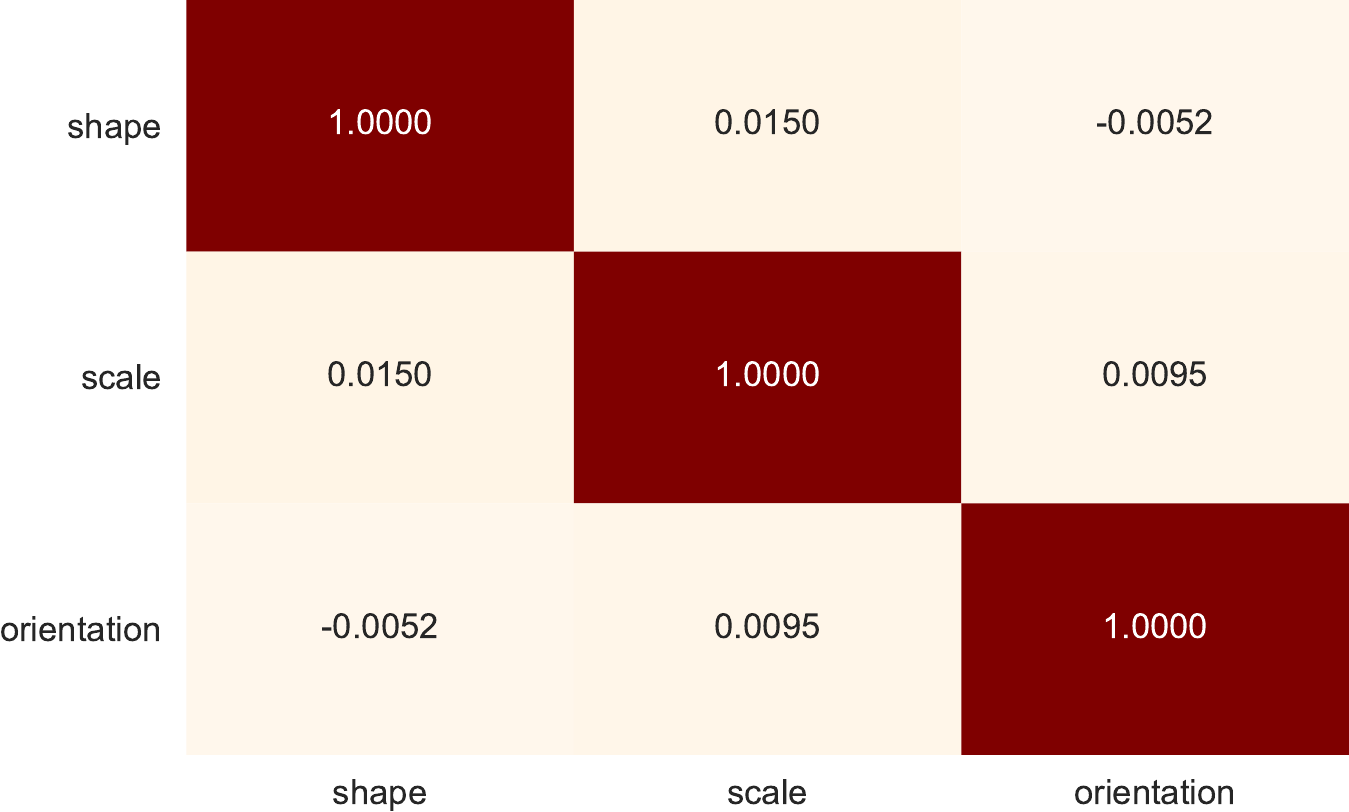} 
        \caption{Disentangled and \\uncorrelated(\gls{IOSS}=0.13)\label{fig:supp-uncorr_ys}}
  \label{fig:supp-uncorr_ys}
\end{subfigure}
\begin{subfigure}{0.3\textwidth}
  \centering
    \includegraphics[width=\linewidth]{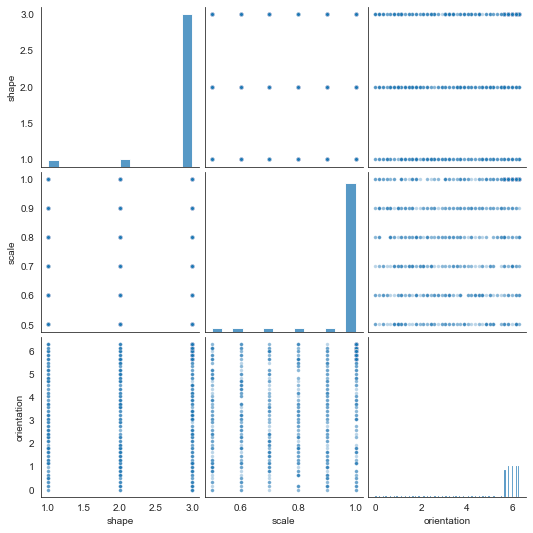}
  \includegraphics[width=\linewidth]{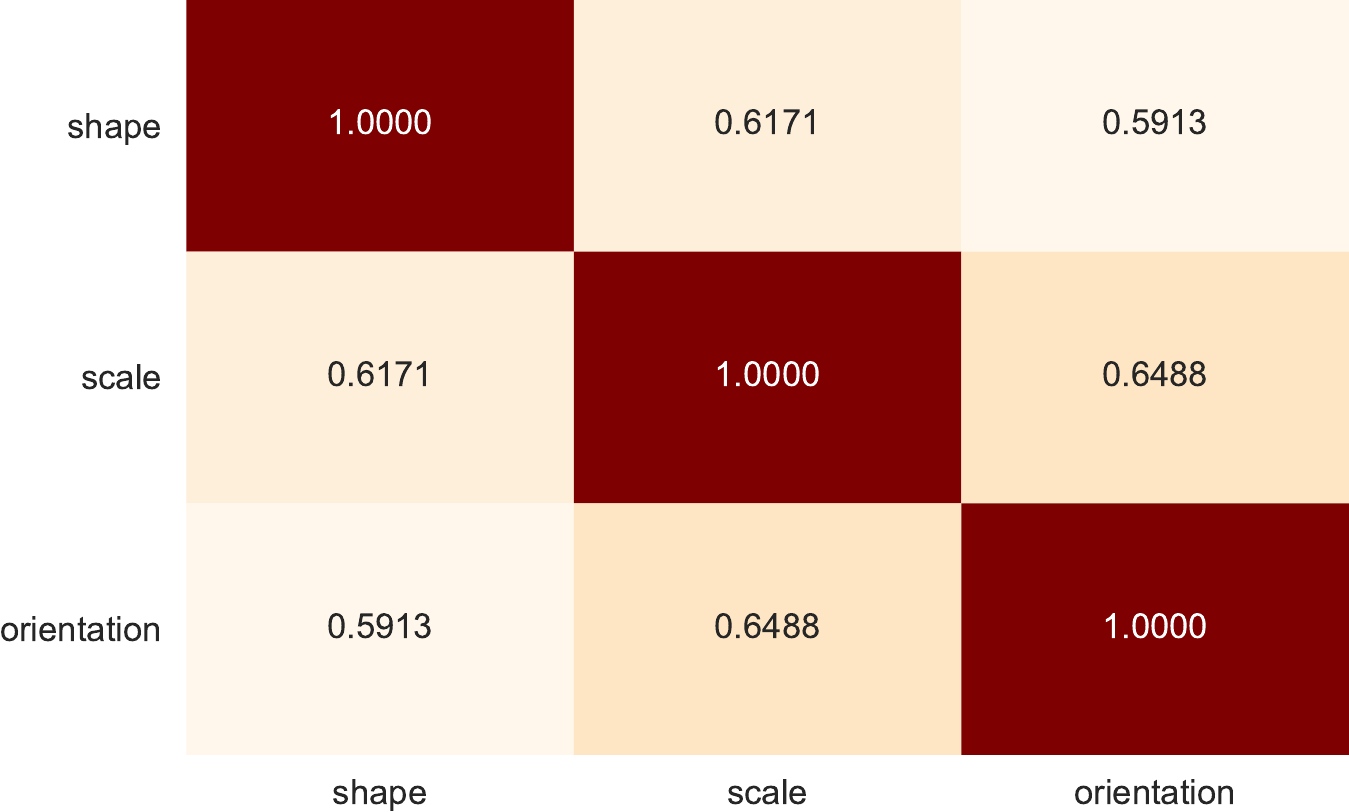}
        \caption{Disentangled but highly \\correlated(\gls{IOSS}=0.14)\label{fig:supp-corr_ys}}
  \label{fig:supp-corr_ys}
\end{subfigure}
\begin{subfigure}{0.3\textwidth}
  \centering
    \includegraphics[width=\linewidth]{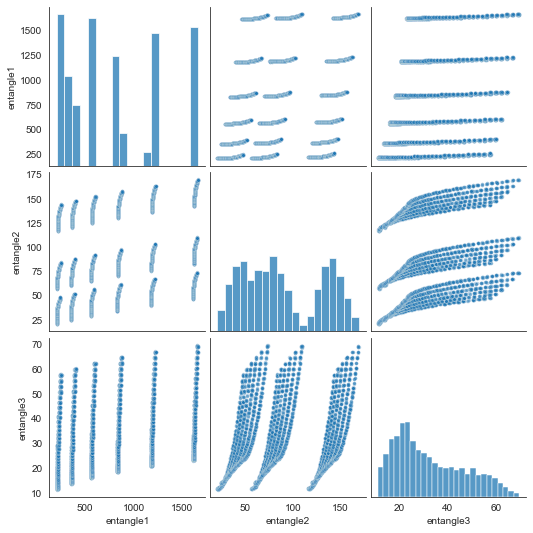} 
  \includegraphics[width=\linewidth]{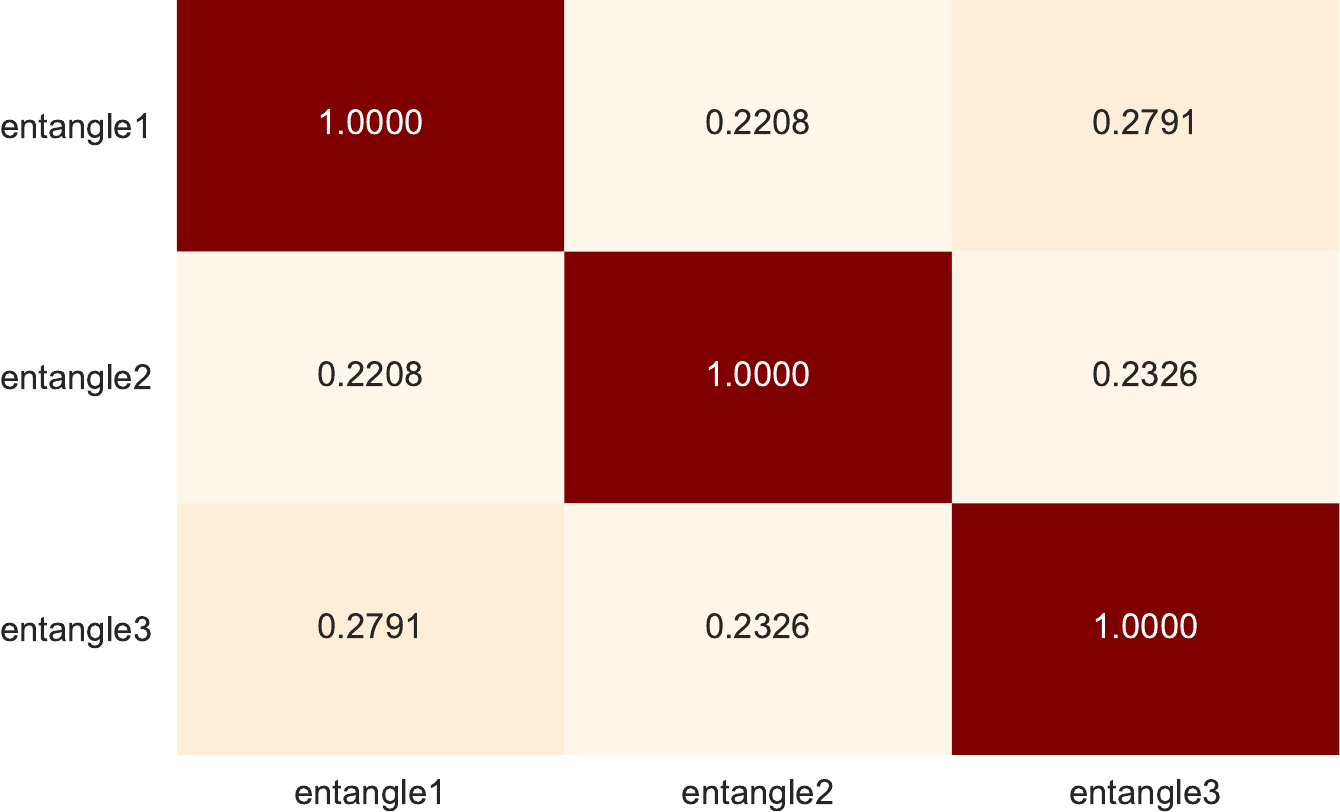}  
        \caption{Entangled but with \\low correlations (\gls{IOSS}=0.27)\label{fig:supp-entangle_ys}}
  \label{fig:supp-entangle_ys}
\end{subfigure}
\caption{Disentangled features have independent support even though
they may be correlated. Moreover, \gls{IOSS} can distinguish
disentangled and entangled features. This figure illustrates how
entangled and disentangled features differ using pairwise scatter
plots.  \Cref{fig:supp-uncorr_ys} considers the ground truth features
(shape, scale, orientation) of the dsprites dataset. These features
are disentangled. They also have independent support, e.g.,
conditional on `scale', the set of values that `orientation' can take
does not change with `scale.' Visually, these causally disentangled
features have scatter plots that occupy rectangular (or
hyperrectangular) region. \Cref{fig:supp-corr_ys} considers the same
features but in a subset of the dsprites dataset where the features
are correlated. These features, though correlated, are still
disentangled; they also have independent support.
\Cref{fig:supp-entangle_ys} considers three entangled features, each
of which is a nonlinear transformation of the three ground-truth
features. These features are not disentangled. Their support are also
not independent. Conditional on `entangle1', the possible values
`entangle2' can take depends on the value of `entangle1.'
\label{fig:supp-disentanglement-truefeatures}}
\end{figure}

To generate the factors for \Cref{fig:supp-uncorr_ys}, we randomly
sample $7,000$ data points from the dsprites dataset.

To generate the factors for \Cref{fig:supp-corr_ys}, we first sort the
data points of the dsprites dataset in descending order by columns
['shape', 'scale', 'orientation', 'positionX', 'positionY']. Next we
create the correlated factors dataset by taking the top 5000 data
points and then randomly draw 700 data points from the rest.

To generate the factors for \Cref{fig:supp-entangle_ys}, we randomly subsample 7,000 data points and then generate the entangled factors as follows:
\begin{align*}
\mathrm{entangle1} &= 6 \cdot\mathrm{shape}^1 + 8 \cdot(\mathrm{scale}/\mathrm{sd}(\mathrm{scale}))^3 + 1 \cdot(\mathrm{orientation}/\mathrm{sd}(\mathrm{orientation}))^3 + 0.2\cdot\cN(0,1),\\
\mathrm{entangle2} &= 12 \cdot\mathrm{shape}^2 + 1 \cdot(\mathrm{scale}/\mathrm{sd}(\mathrm{scale}))^2 + 8 \cdot(\mathrm{orientation}/\mathrm{sd}(\mathrm{orientation}))^1 + 0.2\cdot\cN(0,1),\\
\mathrm{entangle3} &= 0 \cdot\mathrm{shape}^3 + 4 \cdot(\mathrm{scale}/\mathrm{sd}(\mathrm{scale}))^1 + 4 \cdot(\mathrm{orientation}/\mathrm{sd}(\mathrm{orientation}))^2 + 0.2\cdot\cN(0,1).
\end{align*}

\section{Proof of \Cref{thm:disentanglement-support}}
\label{sec:disentanglement-support-proof}
\begin{proof}
We first prove $\supp(Z_1, \ldots, Z_d) = \supp(Z_1)\times \cdots
\supp(Z_d)$. Notice that the causal disentanglement of $Z_1, \ldots,
Z_d$ (\Cref{fig:disentanglement-scm}) implies that
\begin{align}
P(Z_1, \ldots, Z_d) 
=\int P(Z_1, \ldots, Z_d\g C) P(C)\dif C 
=\int P(Z_1\g C) \cdots P(Z_d\g C) P(C)\dif C.
\end{align}
Therefore, 
we have
\begin{align}
P(Z_1, \ldots, Z_d) > 0 \Leftrightarrow P(Z_j\g C\in
\mathcal{C}) > 0 \qquad \forall j\in \{1, \ldots, d\}\text{ for some set
$P(C\in\mathcal{C}) > 0$}.
\end{align}

Together with the positivity condition, which requires
\begin{align}
P(Z_j\g C\in \mathcal{C}) > 0 \text{ for any set $P(C\in\mathcal{C}) >
0$}\Leftrightarrow P(Z_j) > 0,
\end{align}
we have 
\begin{align}
P(Z_1, \ldots, Z_d) > 0 \Leftrightarrow P(Z_j) >
0 \qquad\forall j\in \{1, \ldots, d\}.
\end{align}
As $\supp(G) = \mathbb{I}\{P(G)>0\}$, we can rewrite it as
\begin{align}
\supp(Z_1, \ldots, Z_d) = \supp(Z_1)\times \cdots \supp(Z_d),
\end{align}
i.e. $Z_1, \ldots, Z_d$ satisfy the full support condition.

Next we show that $\supp(Z_i\g Z_\cS) = \supp(Z_i)$.

First notice that
\begin{align}
\mathbb{I}\{P(Z_i, Z_\cS)>0\}
&= \mathbb{I}\{P(Z_i\g Z_\cS) P(Z_\cS)>0\}\\
&=  \mathbb{I}\{P(Z_i\g Z_\cS) > 0\}\times \mathbb{I}\{P(Z_\cS) > 0\}.
\end{align}
Therefore the full support condition $\mathbb{I}\{P(Z_i, Z_\cS)>0\} =
\mathbb{I}\{P(Z_i) > 0\}\times \mathbb{I}\{P(Z_\cS) > 0\}$ implies that
\begin{align}
\mathbb{I}\{P(Z_i) > 0\} = \mathbb{I}\{P(Z_i\g Z_\cS) > 0\}\text{ for
$Z_\cS$ s.t. $\mathbb{I}\{P(Z_\cS) > 0\}$}.
\end{align}

This implies $\supp(Z_i\g Z_\cS) = \supp(Z_i)$ for $Z_\cS\in\supp(Z_\cS)$.

Finally, we consider the support of $\mathrm{do}$ interventions.
\begin{align}
P(Z_\cS \g \mathrm{do}(Z_i=Z_i)) &= \int P(Z_\cS\g Z_i, C) P(C)\dif C,\\
P(Z_\cS \g Z_i=Z_i) &= \int P(Z_\cS\g Z_i, C) P(C\g Z_i)\dif C.
\end{align}

\sloppy
Positivity guarantees that the support of $P(C)$ must be the same as
$P(C\g Z_i)$. (Because $\supp(C\g Z_i)\subseteq \supp(C)$ and
positivity guarantees the other direction.) Therefore, the support of
$P(Z_\cS \g \mathrm{do}(Z_i=Z_i))$ and $P(Z_\cS \g Z_i=Z_i)$ are the
same.

Therefore, if $\supp(Z_\cS\g Z_i)\ne \supp(Z_\cS)$, then $\supp(P(Z_\cS \g
\mathrm{do}(Z_i=Z_i))) \ne \supp(P(Z_\cS))$ and hence $P(Z_\cS \g
\mathrm{do}(Z_i=Z_i))\ne P(Z_\cS)$.

\end{proof}

\section{Proof of \Cref{thm:indep-support-id}}

\label{sec:indep-support-id-proof}

\begin{proof}

We follow a proof-by-contradiction argument. Suppose the opposite:
$\mbZ=\mathbf{h}(\mathbf{Z'})$ where $\mathbf{h}$ nontrivially depends
on two or more of its arguments.

The first two assumptions imply that
$\mathbf{h}:\mathbb{R}^d\rightarrow
\mathbb{R}^d$ is differentiable and invertible.

Moreover, \Cref{eq:h-assumption} ensures that the (nontrivial)
dependency graph (whether $h_i$ nontrivially depends on $Z_j$) of
$h_i(Z_1, \ldots, Z_d), i\in \{1, \ldots, d\}$ on $Z_1, \ldots, Z_d$
remains constant over all values of $Z_1, \ldots, Z_d$ within the
support. This assumption ensures that it is sufficient to study the
boundary of the support of $\mbZ, \mathbf{Z'}$ to conclude about the
dependency structure of $\mathbf{h}(\cdot)$ and thus disentanglement.
In other words, ``$\mathbf{h}$ nontrivially depends on two or more of
its arguments'' implies that ``$\mathbf{h}$ nontrivially depends on
two or more of its arguments at the boundary of the support.''

Finally, we will try to argue that if (1) $\mathbf{Z'}$ satisfies the
independent support condition (2) the function $h_i$ (for some $i$)
nontrivially depends on two or more of $Z'_j, Z'_k$, $j,k\in \{1,
\ldots, d\}$, then $\mbZ$ will violate the independent support
condition. (The fourth assumption is to exclude the corner cases where
some support constraints are vacuous.)


Begin with two examples that illustrate this argument.

\parhead{Warm-up example 1.} Consider two two-dimensional
representations $\mathbf{Z'} = (Z'_1, Z'_2)$ and $\mbZ=(Z_1, Z_2)$.
They satisfy $\mbZ = (Z_1, Z_2) = \mathbf{h}(\mbZ')$ where
\begin{align}
Z_1 &= h_1(Z'_1, Z'_2) = Z'_1 + Z'_2,\\
Z_2 &= h_2(Z'_1, Z'_2) = Z'_1 - Z'_2. 
\end{align}

Moreover, assume the support of $\mbZ$ is $\supp(Z_1, Z_2) =
[a_1,b_1]\times [a_2,b_2],$ which satisfies the independent support
condition.

Next we show that $\mathbf{h}(Z'_1, Z'_2)$ cannot satisfy the
independent support condition.

Recall that that $\supp(Z_1, Z_2) = [a_1,b_1]\times [a_2,b_2]$ implies
that $\supp(Z'_1 + Z'_2) = [a_1,b_1]$ and $\supp(Z'_1-Z'_2) =
[a_2,b_2]$. It implies that $\supp(Z'_1\g Z'_2) = [a_1-Z'_2,
b_1-Z'_2]\cap[a_2+Z'_2, b_2+Z'_2]$. All end points of this
intersection involves $Z'_2$, so the support of $\supp(Z'_1|Z'_2)$
depends on $Z'_2$ (for values of $Z'_2$ s.t. $\supp(Z'_1\g Z'_2)$ is
non-empty); it violates the independent support condition.

More broadly, $\mbZ'$ does not satisfy the independent support
condition because both $h_1(Z'_1, Z'_2)$ and $h_2(Z'_1, Z'_2)$
non-trivially depend on both $Z'_1, Z'_2$, across all values of $Z'_1,
Z'_2$ within the support. 

Visually, independent support condition ensures that the support is
rectangular. If we apply a rotation-like transformation, then it can
no longer be rectangular in the original coordinate system. More
generally, if we apply any transformations that will ``entangle'' the
coordinates of the original representation at its support boundary,
then the independent support condition cannot hold for both the
original representation and the post-transformation representation.
That said, we still need \Cref{eq:h-assumption} to extend argument
about the boundary to one about the whole function being entangling.

\parhead{Warm-up wxample 2.} We next consider a different mapping
between $\mbZ$ and $\mathbf{Z'}$,
\begin{align}
Z_1 &= h_1(Z'_1, Z'_2) = Z'_1(Z'_1- Z'_2),\\
Z_2 &= h_2(Z'_1, Z'_2) = Z'_1 - Z'_2. 
\end{align}

Again assume independent support for $\mbZ$: $\supp(Z_1,
Z_2)=[a_1,b_1]\times [a_2,b_2].$


In this case, $\supp(Z'_1\g Z'_2) = \{Z'_1: Z'_1(Z'_1-Z'_2) \in [a_1,
b_1] \} \cap [a_2+Z'_2, b_2+Z'_2].$ Again the boundary of the support
depends on $Z_2'$.

The boundary of the support $\{Z'_1: h_1(Z'_1, Z'_2) \in [a_1, b_1]
\}$ does not depend on $Z'_2$ only when $[a_1, b_2]$ is a vacuous
bound for $h_1$, e.g. $a_1=0, b_1=1, h_1(Z'_1, Z'_2) =
\frac{\exp(-(Z'_1+Z'_2))}{1+\exp(-(Z'_1+Z'_2))}$.

\parhead{General argument.} Suppose some $h_l$ nontrivially depend on
two or more of its arguments, say $Z'_j, Z'_k$, $j,k\in \{1, \ldots,
d\}$. Then we argue that the support of $Z'_j, Z'_k$ will not be
independent.

Recall that $\mbZ$ satisfies the independent support condition,
$\supp(Z_1,\ldots, Z_d) = [a_1, b_1]\times \cdots \times[a_d, b_d]$.
We consider the support $\supp(Z'_j\g Z'_k\s \mbZ'_{-j,-k}) =
\cap_{l=1}^d \{Z'_j: h_l(Z'_1,\ldots, Z'_d)\in [a_l, b_l] \},$ fixing
$\mbZ'_{-j,-k}$ at some value within the support. We argue that (1)
$\{Z'_j: h_l(Z'_1,\ldots, Z'_d)\in [a_l, b_l] \}$ implies the support
of $Z'_j$ is dependent on $Z'_k$, and (2) $\supp(Z'_j\g Z'_k\s
\mbZ'_{-j,-k})$ also depends on $Z'_k$.

The first argument is via the invertibility of $h_l$ for all values of
$\mbZ'_{-j}$. Fix $\mbZ'_{-j,-k}$ at some value within the support.
Then $\supp(Z'_j\g Z'_k\s \mbZ'_{-j,-k}) = [\min\{B_l^1, B_l^2\},
\max\{B_l^1, B_l^2\}]$, where $B_l^1 =h_l^{-1}(a_l\s Z'_k,
\mbZ'_{-j,-k})$, $B_l^2 =h_l^{-1}(b_l\s Z'_k, \mbZ'_{-j,-k})$, where
$h_l^{-1}(b_l\s Z'_k, \mbZ'_{-j,-k})$ nontrivially depends on $Z'_k$
because of \Cref{eq:h-assumption} and that $h_l$ nontrivially depends
on $Z'_k, Z'_j$.

The second argument is via the assumption \Cref{eq:trivial-h}, which
ensures that the support constraint imposed by $h_l$ is not subsumed
by other $h_j, j\ne l$, and its non-trivial dependence on $Z'_k$ will
make the support $\supp(Z'_j\g Z'_k\s \mbZ'_{-j,-k})$ also
non-trivially depend on $Z'_k$ for some values of $\mbZ'_{-j,-k}$
within the support.

\end{proof}

\section{Details of empirical studies of IOSS and additional empirical
results}

\subsection{Details of \Cref{subsec:disentangle-measure}}

\Cref{fig:disentangle-measure-1,fig:disentangle-measure-2,fig:disentangle-measure-3,fig:disentangle-measure-4}
also show that \gls{IOSS} can better separate disentangled and
entangled representations than existing unsupervised disentanglement
metrics.

\begin{figure}[t!]
  \centering
  \captionsetup[subfigure]{justification=centering}
\begin{subfigure}{0.31\textwidth}
  \centering
    \includegraphics[width=\linewidth]{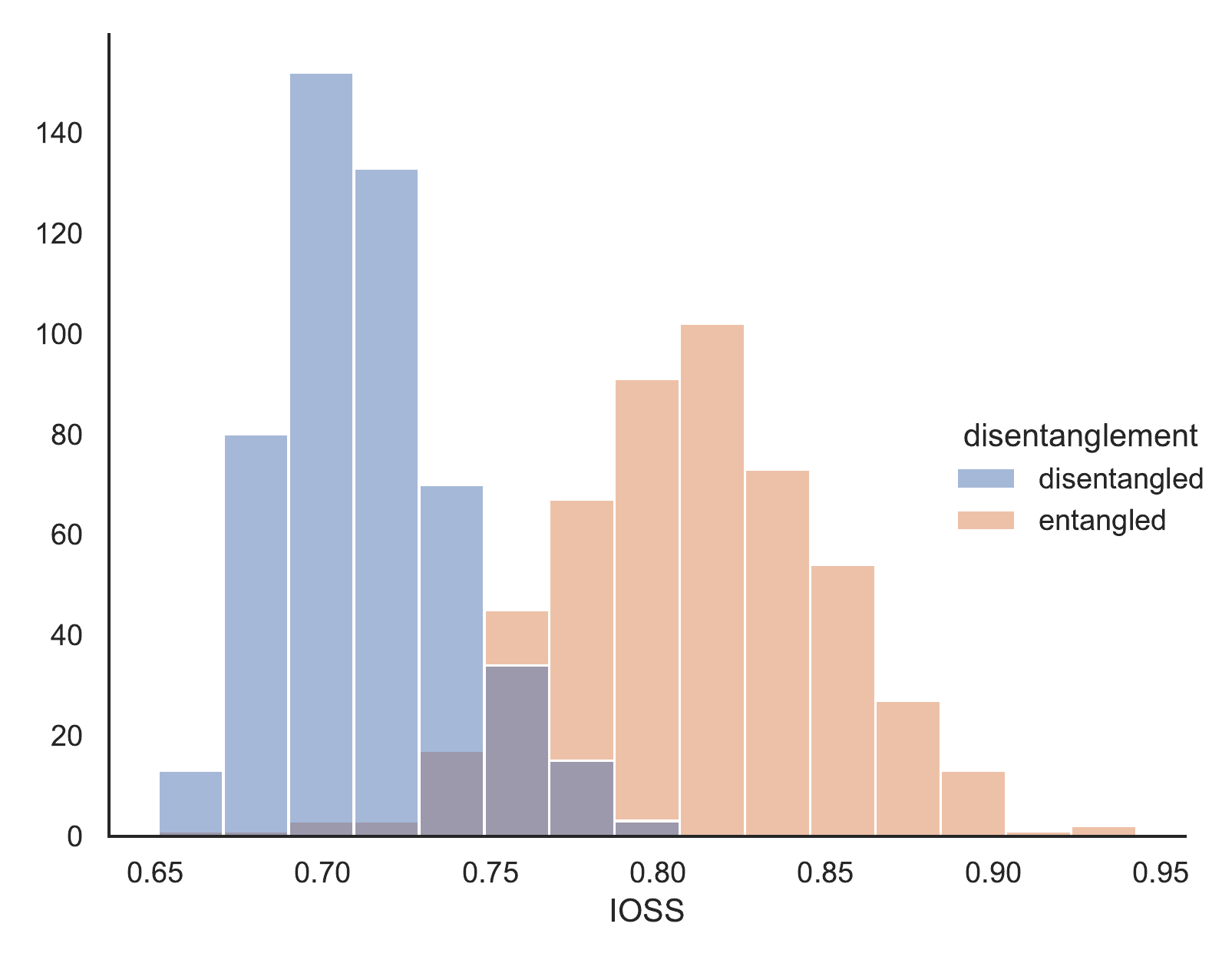}
  \caption{\gls{IOSS}}
\end{subfigure}
\begin{subfigure}{0.31\textwidth}
  \centering
    \includegraphics[width=\linewidth]{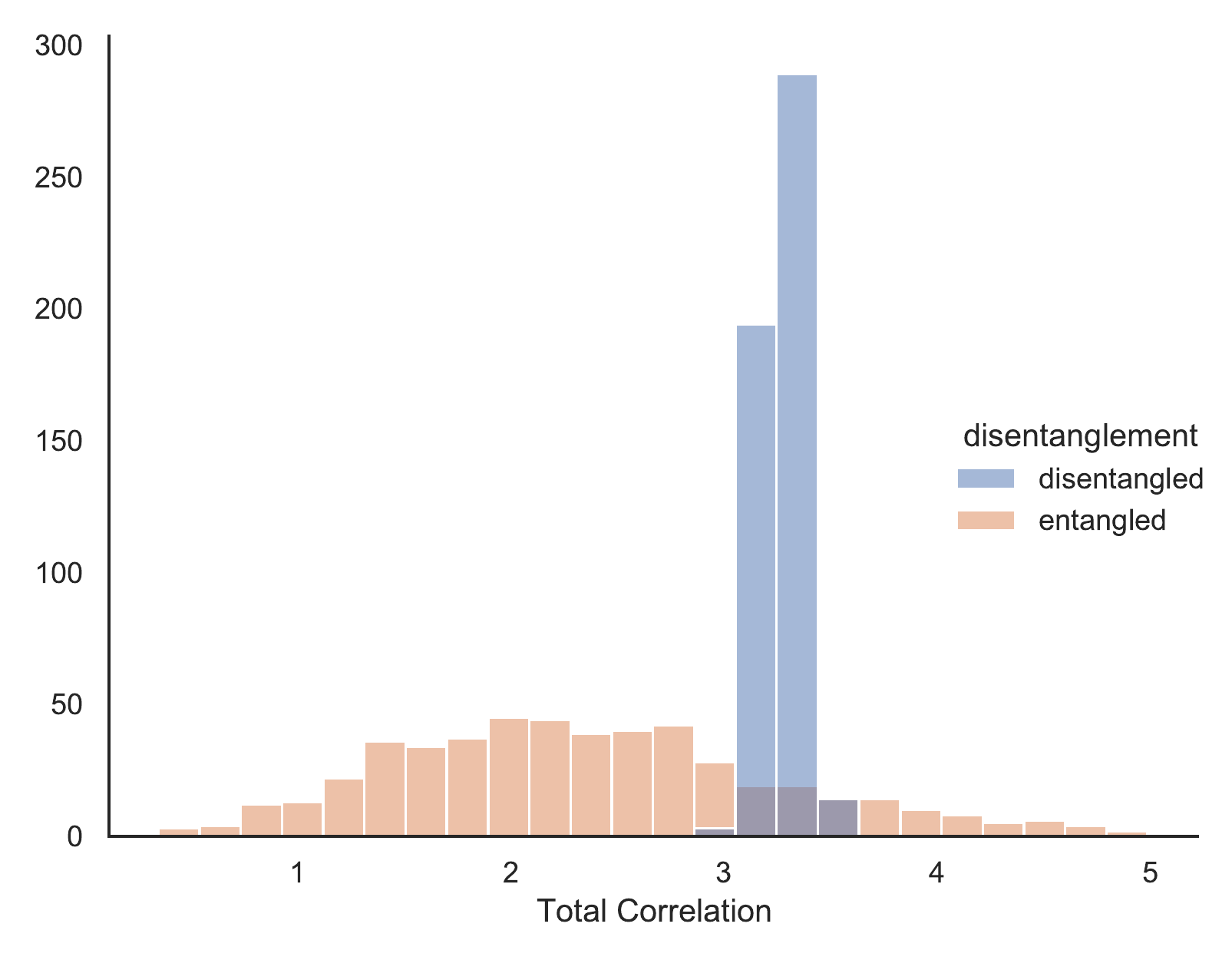}
  \caption{Total Correlation}
\end{subfigure}
\begin{subfigure}{0.31\textwidth}
  \centering
    \includegraphics[width=\linewidth]{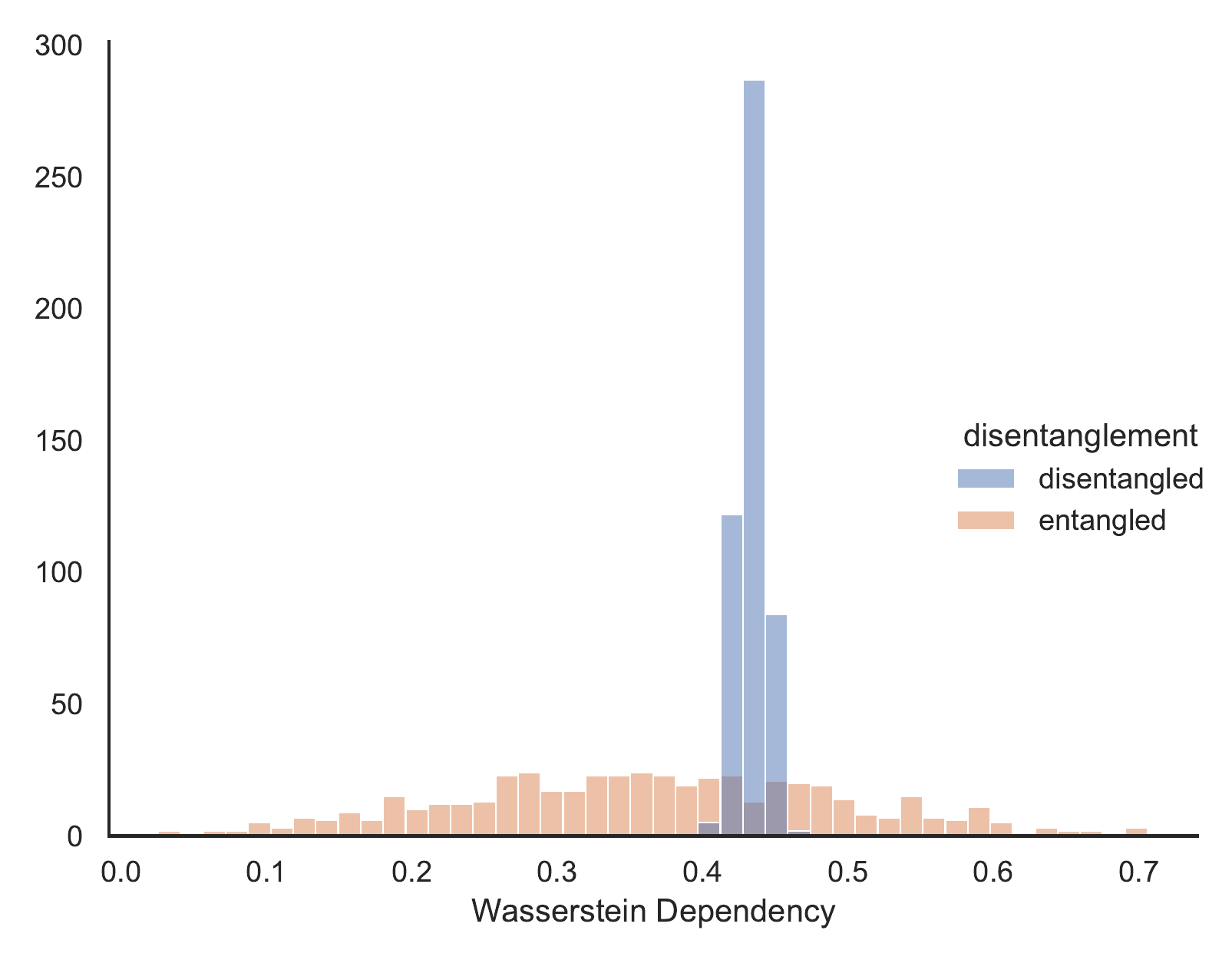}
  \caption{Wasserstein Dependency}
\end{subfigure}
\caption{IOSS can better distinguish entangled and disentangled
representations than existing unsupervised disentanglement metrics on
the dsprites dataset. \label{fig:disentangle-measure-2}}
\end{figure}
\begin{figure}[t!]
  \centering
  \captionsetup[subfigure]{justification=centering}
\begin{subfigure}{0.31\textwidth}
  \centering
    \includegraphics[width=\linewidth]{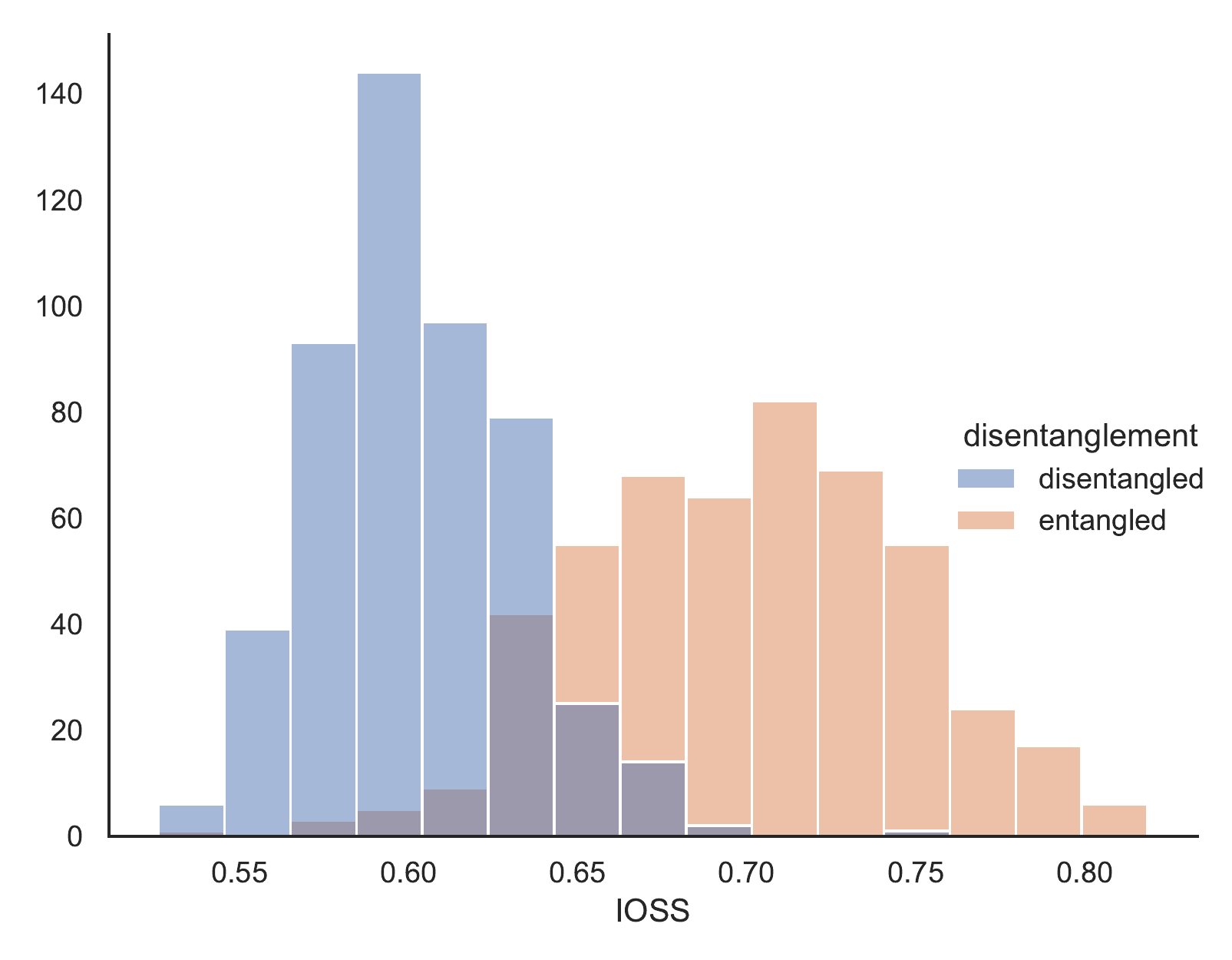}
  \caption{\gls{IOSS}}
\end{subfigure}
\begin{subfigure}{0.31\textwidth}
  \centering
    \includegraphics[width=\linewidth]{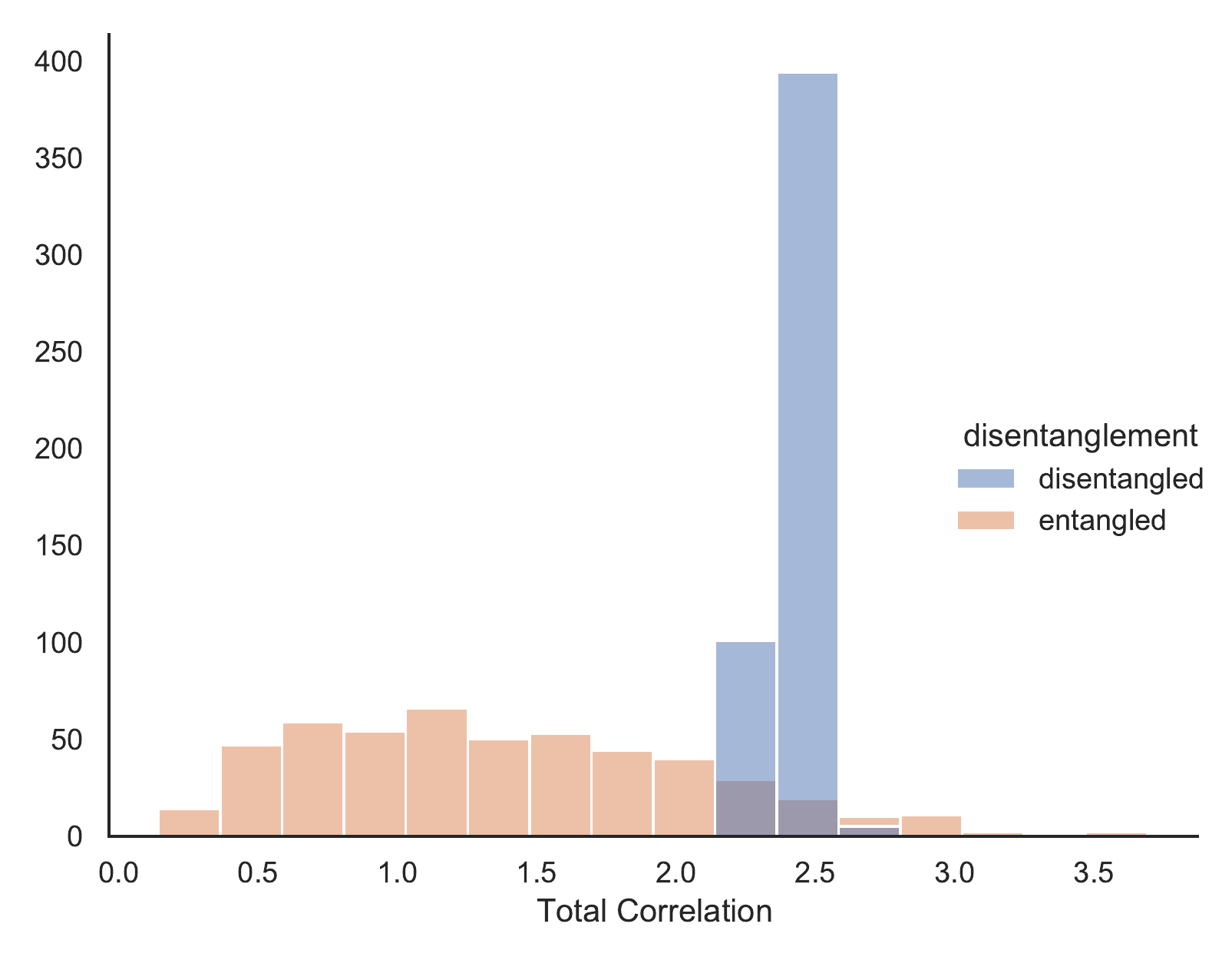}
  \caption{Total Correlation}
\end{subfigure}
\begin{subfigure}{0.31\textwidth}
  \centering
    \includegraphics[width=\linewidth]{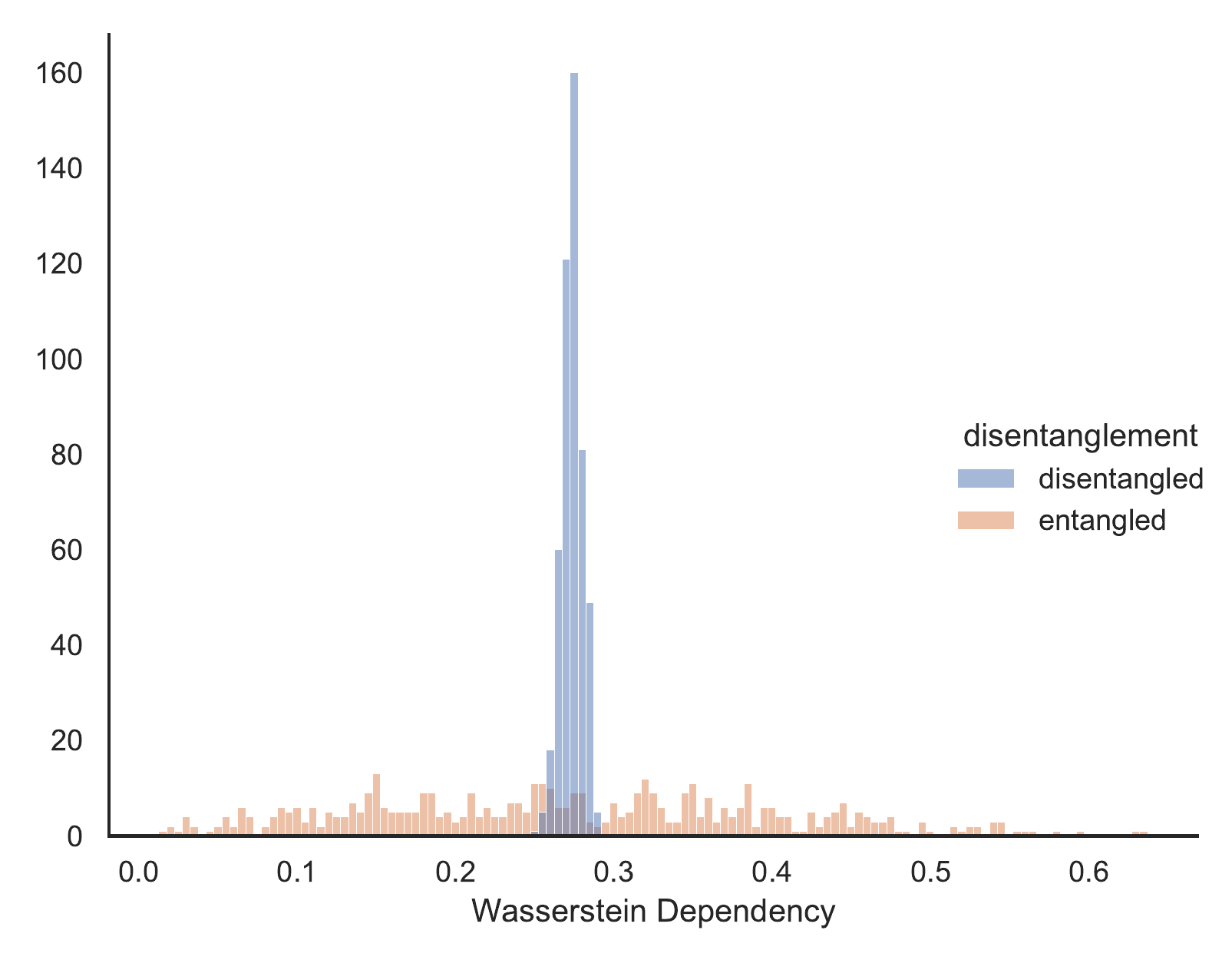}
  \caption{Wasserstein Dependency}
\end{subfigure}
\caption{IOSS can better distinguish entangled and disentangled
representations than existing unsupervised disentanglement metrics on
the smallnorb dataset. \label{fig:disentangle-measure-3}}
\end{figure}
\begin{figure}[t!]
  \centering
  \captionsetup[subfigure]{justification=centering}
\begin{subfigure}{0.31\textwidth}
  \centering
    \includegraphics[width=\linewidth]{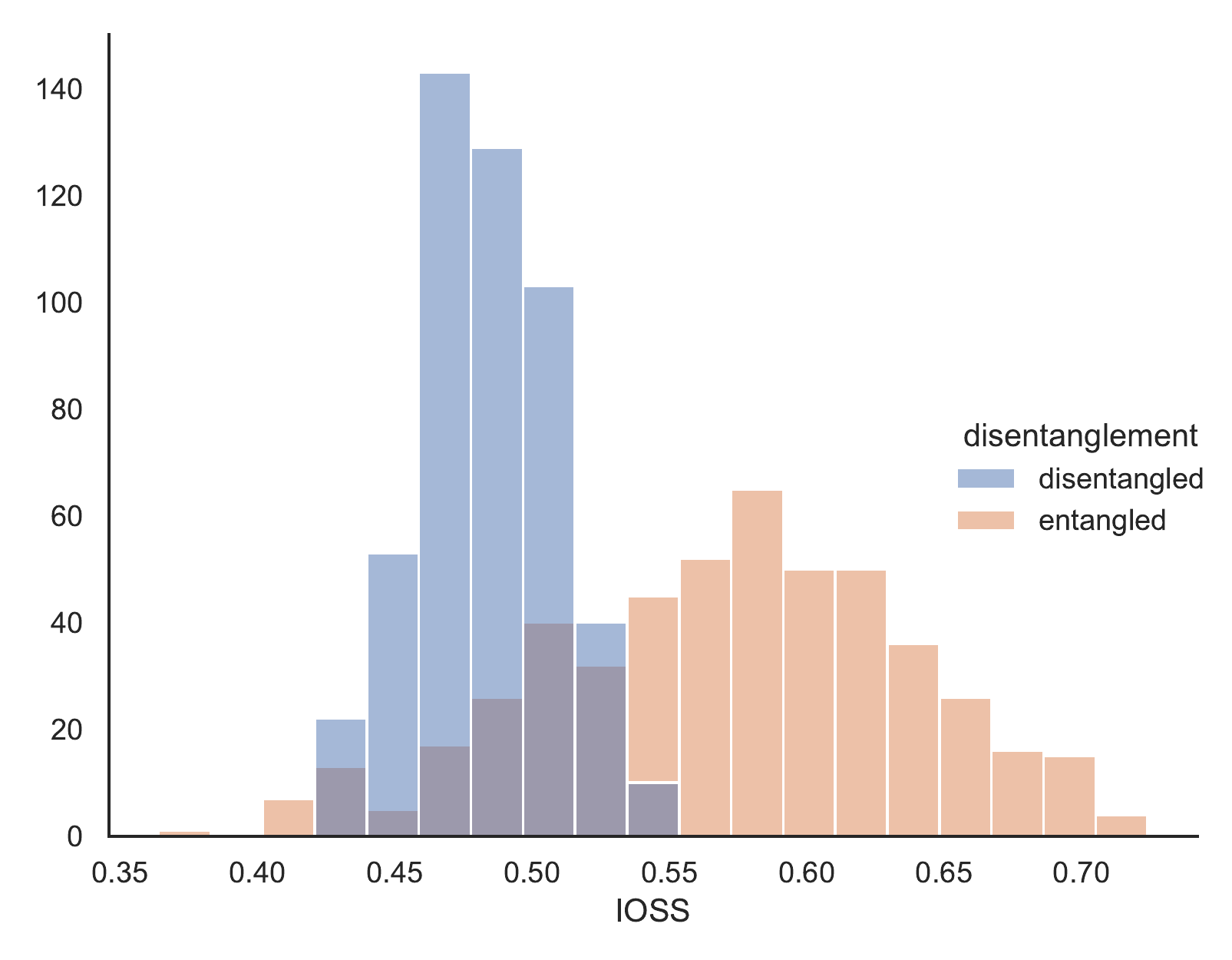}
  \caption{\gls{IOSS}}
\end{subfigure}
\begin{subfigure}{0.31\textwidth}
  \centering
    \includegraphics[width=\linewidth]{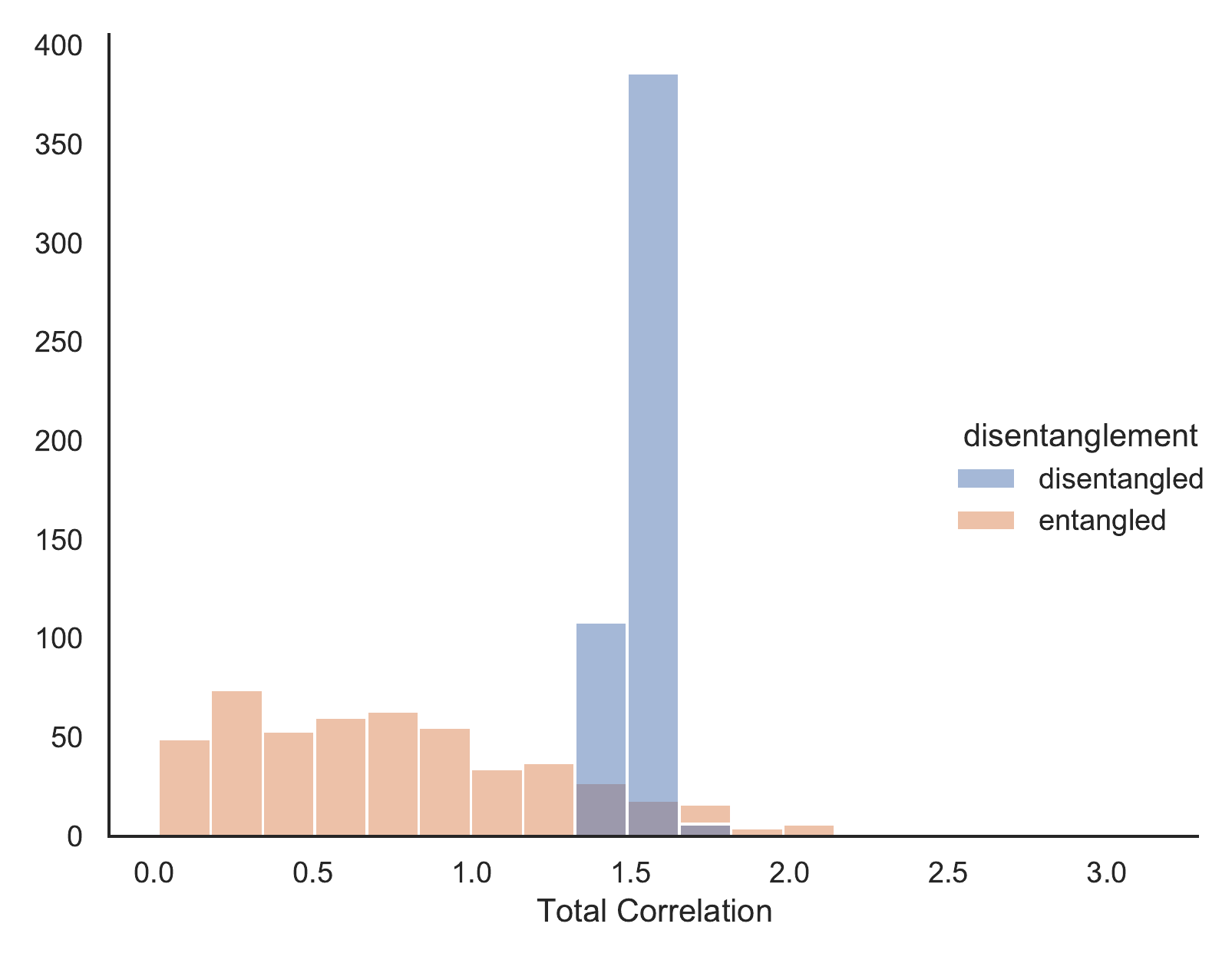}
  \caption{Total Correlation}
\end{subfigure}
\begin{subfigure}{0.31\textwidth}
  \centering
    \includegraphics[width=\linewidth]{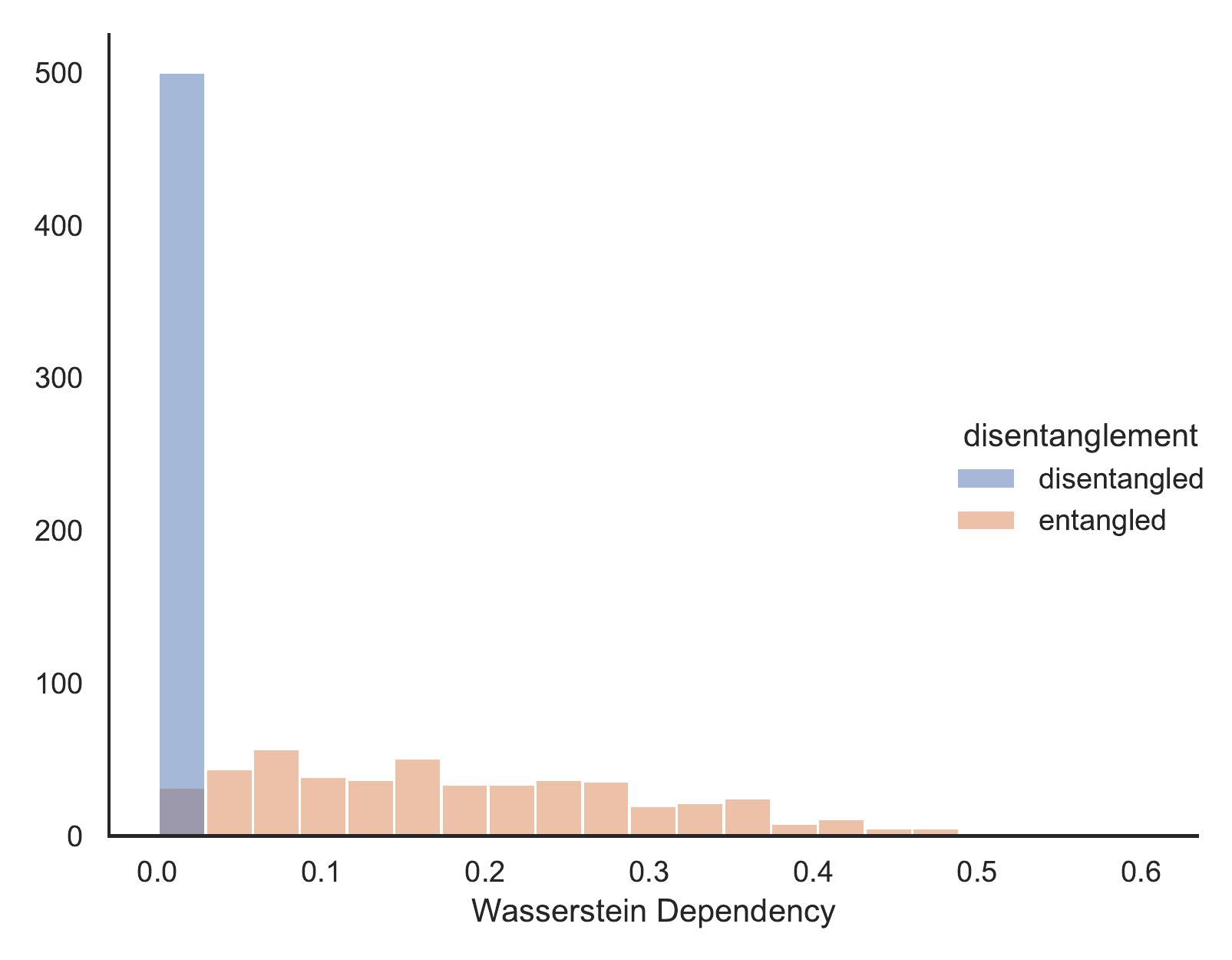}
  \caption{Wasserstein Dependency}
\end{subfigure}
\caption{IOSS is competitive in distinguishing entangled and
disentangled representations compared with existing unsupervised
disentanglement metrics on the cars3d dataset.
\label{fig:disentangle-measure-4}}
\end{figure}

\subsection{Details of \Cref{subsubsec:disentangle-learn}}

\Cref{fig:ioss-support} illustrates that regularizing for \gls{IOSS}
indeed encourages independent support across different dimensions of
the learned representations.

\begin{figure}[t]
  \centering
  \captionsetup[subfigure]{justification=centering}
\begin{subfigure}{0.45\textwidth}
  \centering
    \includegraphics[width=\linewidth]{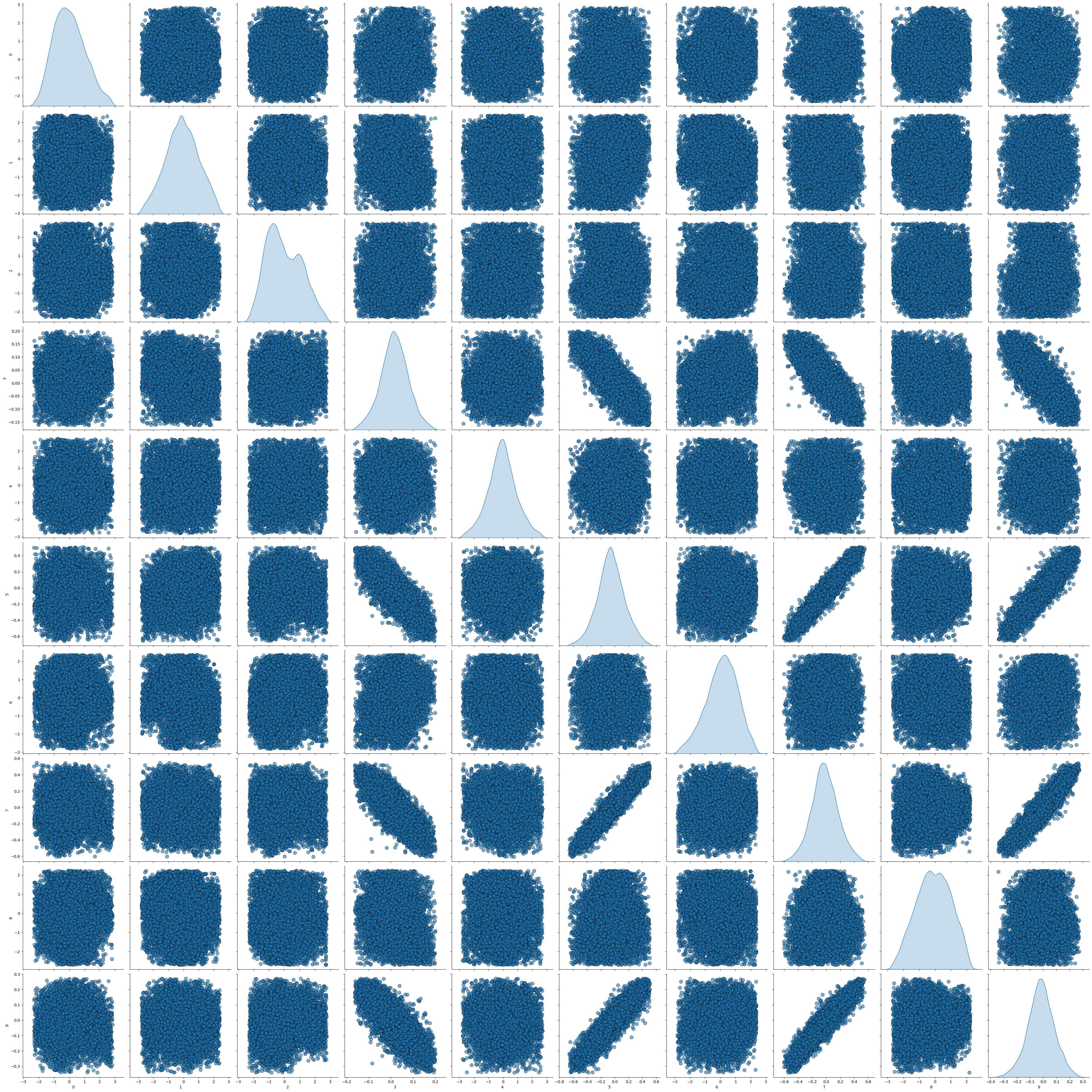}
  \caption{\gls{IOSS} regularization strength=0}
\end{subfigure}
\begin{subfigure}{0.45\textwidth}
  \centering
    \includegraphics[width=\linewidth]{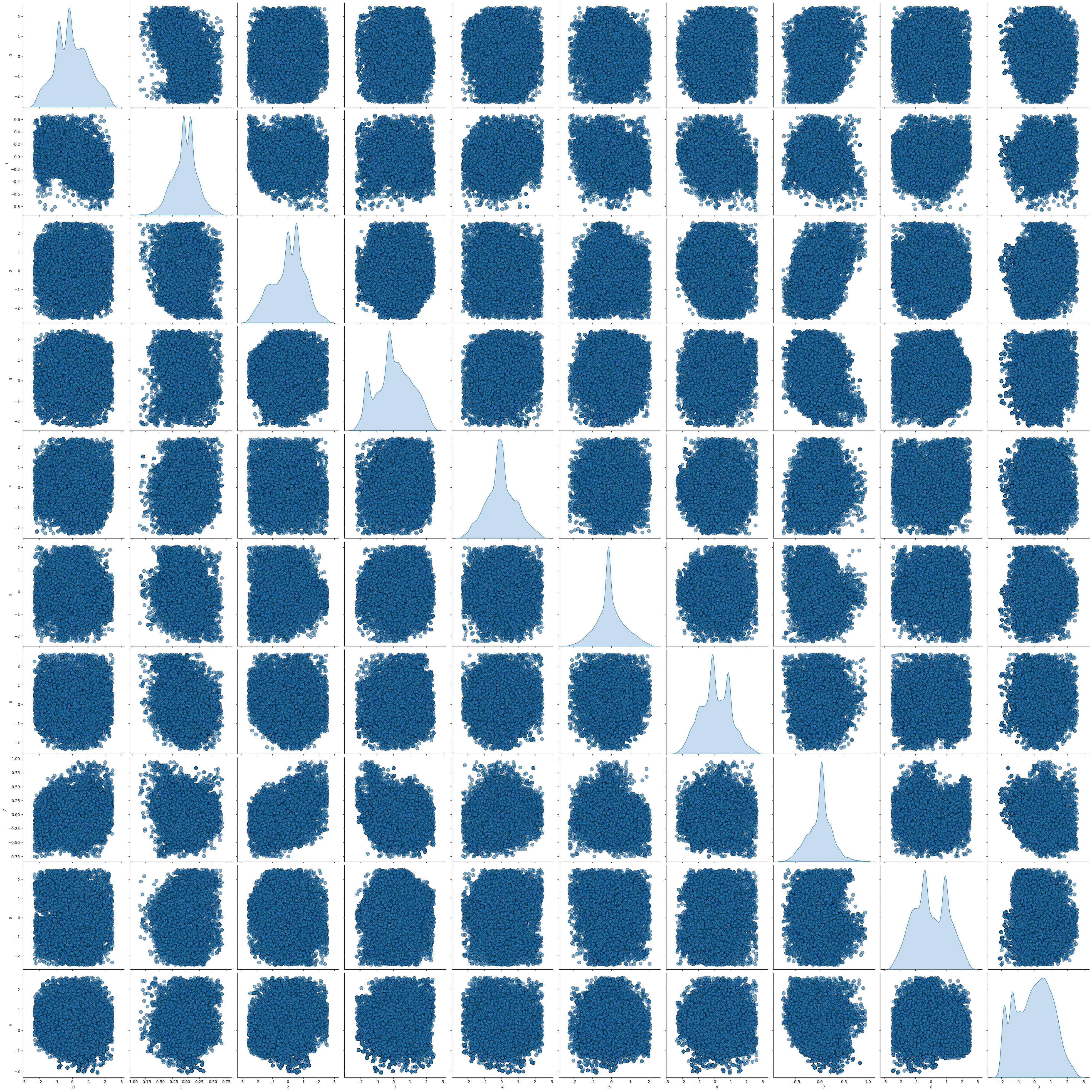}
  \caption{\gls{IOSS} regularization strength=1000}
\end{subfigure}
\caption{Increasing regularization strengths of IOSS encourages the
learned representation to have independent (i.e. hyperrectangular)
support. (a) The pairwise scatter plot of the learned representations
without \gls{IOSS} regularization; some dimensions of the
representation do not have independent support. (b) The pairwise
scatter plot of the learned representations with \gls{IOSS}
regularization; most dimensions of the representation have independent
support.
\label{fig:ioss-support}}
\end{figure}

\clearpage


\begin{thebibliography}{}

\bibitem[Achille \& Soatto, 2018]{achille2018emergence}
Achille, A. \& Soatto, S. (2018).
\newblock Emergence of invariance and disentanglement in deep representations.
\newblock {\em Journal of Machine Learning Research}, 19(1), 1947--1980.

\bibitem[Airoldi et~al., 2008]{airoldi2008mixed}
Airoldi, E.~M., Blei, D.~M., Fienberg, S.~E., \& Xing, E.~P. (2008).
\newblock Mixed membership stochastic blockmodels.
\newblock {\em Journal of Machine Learning Research}, 9, 1981--2014.

\bibitem[Arjovsky et~al., 2019]{arjovsky2019invariant}
Arjovsky, M., Bottou, L., Gulrajani, I., \& Lopez-Paz, D. (2019).
\newblock Invariant risk minimization.
\newblock {\em arXiv preprint arXiv:1907.02893}.

\bibitem[Bai \& Li, 2016]{bai2016maximum}
Bai, J. \& Li, K. (2016).
\newblock Maximum likelihood estimation and inference for approximate factor
  models of high dimension.
\newblock {\em Review of Economics and Statistics}, 98(2), 298--309.

\bibitem[Bengio et~al., 2013]{bengio2013representation}
Bengio, Y., Courville, A., \& Vincent, P. (2013).
\newblock Representation learning: A review and new perspectives.
\newblock {\em IEEE Transactions on Pattern Analysis and Machine Intelligence},
  35(8), 1798--1828.

\bibitem[Blei et~al., 2017]{blei2017variational}
Blei, D.~M., Kucukelbir, A., \& McAuliffe, J.~D. (2017).
\newblock Variational inference: A review for statisticians.
\newblock {\em Journal of the American Statistical Association}, 112(518),
  859--877.

\bibitem[Blei et~al., 2003]{blei2003latent}
Blei, D.~M., Ng, A.~Y., \& Jordan, M.~I. (2003).
\newblock Latent {D}irichlet allocation.
\newblock {\em Journal of Machine Learning Research}, 3, 993--1022.

\bibitem[Bouchacourt et~al., 2018]{bouchacourt2018multi}
Bouchacourt, D., Tomioka, R., \& Nowozin, S. (2018).
\newblock Multi-level variational autoencoder: Learning disentangled
  representations from grouped observations.
\newblock In {\em AAAI Conference on Artificial Intelligence}, volume~32.

\bibitem[Burgess et~al., 2018]{burgess2018understanding}
Burgess, C.~P., Higgins, I., et~al. (2018).
\newblock Understanding disentangling in beta-{VAE}.
\newblock {\em arXiv preprint arXiv:1804.03599}.

\bibitem[Chalupka et~al., 2017]{Chalupka2017}
Chalupka, K., Eberhardt, F., \& Perona, P. (2017).
\newblock {Causal feature learning: an overview}.
\newblock {\em Behaviormetrika}, 44(1), 137--164.

\bibitem[Chalupka et~al., 2015]{Chalupka2015}
Chalupka, K., Perona, P., \& Eberhardt, F. (2015).
\newblock {Visual causal feature learning}.
\newblock {\em Uncertainty in Artificial Intelligence - Proceedings of the 31st
  Conference, UAI 2015}, (pp.\ 181--190).

\bibitem[Chen et~al., 2018]{chen2018isolating}
Chen, R.~T., Li, X., Grosse, R., \& Duvenaud, D. (2018).
\newblock Isolating sources of disentanglement in {VAE}s.
\newblock In {\em Neural Information Processing Systems}  (pp.\ 2615--2625).

\bibitem[Chen et~al., 2020a]{chen2020simple}
Chen, T., Kornblith, S., Norouzi, M., \& Hinton, G. (2020a).
\newblock A simple framework for contrastive learning of visual
  representations.
\newblock In {\em International Conference on Machine Learning}  (pp.\
  1597--1607).

\bibitem[Chen et~al., 2020b]{chen2020structured}
Chen, Y., Li, X., \& Zhang, S. (2020b).
\newblock Structured latent factor analysis for large-scale data:
  Identifiability, estimability, and their implications.
\newblock {\em Journal of the American Statistical Association}, 115(532),
  1756--1770.

\bibitem[Cheng \& Lu, 2017]{cheng2017causal}
Cheng, P.~W. \& Lu, H. (2017).
\newblock Causal invariance as an essential constraint for creating a causal
  representation of the world: Generalizing.
\newblock {\em The Oxford Handbook of Causal Reasoning}, (pp.\~65).

\bibitem[Correa \& Bareinboim, 2020]{correa2020calculus}
Correa, J. \& Bareinboim, E. (2020).
\newblock {A} calculus for stochastic interventions: Causal effect
  identification and surrogate experiments.
\newblock In {\em Proceedings of the AAAI Conference on Artificial
  Intelligence}, volume~34  (pp.\ 10093--10100).

\bibitem[Creager et~al., 2021]{creager2021environment}
Creager, E., Jacobsen, J.-H., \& Zemel, R. (2021).
\newblock Environment inference for invariant learning.
\newblock In {\em International Conference on Machine Learning}  (pp.\
  2189--2200).: PMLR.

\bibitem[D'Amour, 2019a]{d2019comment}
D'Amour, A. (2019a).
\newblock Comment: Reflections on the deconfounder.
\newblock {\em Journal of the American Statistical Association}, 114(528),
  1597--1601.

\bibitem[D'Amour, 2019b]{d2019multi}
D'Amour, A. (2019b).
\newblock On multi-cause approaches to causal inference with unobserved
  counfounding: Two cautionary failure cases and a promising alternative.
\newblock In {\em The 22nd International Conference on Artificial Intelligence
  and Statistics}  (pp.\ 3478--3486).

\bibitem[D'Amour et~al., 2020a]{d2020overlap}
D'Amour, A., Ding, P., Feller, A., Lei, L., \& Sekhon, J. (2020a).
\newblock Overlap in observational studies with high-dimensional covariates.
\newblock {\em Journal of Econometrics}, 221, 644--654.

\bibitem[D'Amour et~al., 2020b]{damour2020underspecification}
D'Amour, A., Heller, K., et~al. (2020b).
\newblock Underspecification presents challenges for credibility in modern
  machine learning.
\newblock {\em arXiv preprint arXiv:2011.03395}.

\bibitem[Deng, 2012]{deng2012mnist}
Deng, L. (2012).
\newblock The mnist database of handwritten digit images for machine learning
  research [best of the web].
\newblock {\em IEEE Signal Processing Magazine}, 29(6), 141--142.

\bibitem[Eberhardt \& Scheines, 2007]{eberhardt2007interventions}
Eberhardt, F. \& Scheines, R. (2007).
\newblock Interventions and causal inference.
\newblock {\em Philosophy of Science}, 74(5), 981--995.

\bibitem[Erosheva \& Fienberg, 2005]{erosheva2005}
Erosheva, E.~A. \& Fienberg, S.~E. (2005).
\newblock Bayesian mixed membership models for soft clustering and
  classification.
\newblock In C. Weihs \& W. Gaul (Eds.), {\em Classification --- the Ubiquitous
  Challenge}  (pp.\ 11--26).  Berlin, Heidelberg: Springer Berlin Heidelberg.

\bibitem[Fong \& Grimmer, 2016]{fong2016discovery}
Fong, C. \& Grimmer, J. (2016).
\newblock Discovery of treatments from text corpora.
\newblock In {\em Proceedings of the 54th Annual Meeting of the Association for
  Computational Linguistics (Volume 1: Long Papers)}  (pp.\ 1600--1609).

\bibitem[Galhotra et~al., 2021]{galhotra2021explaining}
Galhotra, S., Pradhan, R., \& Salimi, B. (2021).
\newblock Explaining black-box algorithms using probabilistic contrastive
  counterfactuals.
\newblock In {\em Proceedings of the 2021 International Conference on
  Management of Data}  (pp.\ 577--590).

\bibitem[Gelman \& Imbens, 2013]{gelman2013ask}
Gelman, A. \& Imbens, G. (2013).
\newblock {\em Why ask why? Forward causal inference and reverse causal
  questions}.
\newblock Technical report, National Bureau of Economic Research.

\bibitem[Gelman et~al., 1996]{gelman1996posterior}
Gelman, A., Meng, X.-L., \& Stern, H. (1996).
\newblock Posterior predictive assessment of model fitness via realized
  discrepancies.
\newblock {\em Statistica Sinica}, (pp.\ 733--760).

\bibitem[Gilks et~al., 1995]{gilks1995markov}
Gilks, W.~R., Richardson, S., \& Spiegelhalter, D. (1995).
\newblock {\em Markov Chain Monte Carlo in Practice}.
\newblock CRC press.

\bibitem[Golub et~al., 1976]{golub1976rank}
Golub, G., Klema, V., \& Stewart, G.~W. (1976).
\newblock {\em Rank degeneracy and least squares problems}.
\newblock Technical report, Stanford University, Department of Computer
  Science.

\bibitem[Goodfellow et~al., 2014]{goodfellow2014generative}
Goodfellow, I.~J., Pouget-Abadie, J., et~al. (2014).
\newblock Generative adversarial networks.
\newblock {\em arXiv preprint arXiv:1406.2661}.

\bibitem[Grimmer \& Fong, 2021]{grimmer2020causal}
Grimmer, J. \& Fong, C. (2021).
\newblock Causal inference with latent treatments.
\newblock {\em American Journal of Political Science}.

\bibitem[Hadsell et~al., 2006]{hadsell2006dimensionality}
Hadsell, R., Chopra, S., \& LeCun, Y. (2006).
\newblock Dimensionality reduction by learning an invariant mapping.
\newblock In {\em 2006 IEEE Computer Society Conference on Computer Vision and
  Pattern Recognition (CVPR'06)}, volume~2  (pp.\ 1735--1742).: IEEE.

\bibitem[Heinze-Deml et~al., 2018]{heinze2018causal}
Heinze-Deml, C., Maathuis, M.~H., \& Meinshausen, N. (2018).
\newblock Causal structure learning.
\newblock {\em Annual Review of Statistics and Its Application}, 5, 371--391.

\bibitem[Higgins et~al., 2017]{higgins2017beta}
Higgins, I., Matthey, L., et~al. (2017).
\newblock Beta-{VAE}: Learning basic visual concepts with a constrained
  variational framework.
\newblock In {\em International Conference on Learning Representations}.

\bibitem[Hosoya, 2019]{hosoya2019group}
Hosoya, H. (2019).
\newblock Group-based learning of disentangled representations with
  generalizability for novel contents.
\newblock In {\em International Joint Conference on Artificial Intelligence}
  (pp.\ 2506--2513).

\bibitem[Imai \& Jiang, 2019]{imai2019comment}
Imai, K. \& Jiang, Z. (2019).
\newblock Comment: The challenges of multiple causes.
\newblock {\em Journal of the American Statistical Association}, 114(528),
  1605--1610.

\bibitem[Imbens \& Rubin, 2015]{imbens2015causal}
Imbens, G.~W. \& Rubin, D.~B. (2015).
\newblock {\em {Causal Inference in Statistics, Social, and Biomedical
  Sciences}}.
\newblock Cambridge University Press.

\bibitem[Janzing \& Sch{\"o}lkopf, 2015]{janzing2015semi}
Janzing, D. \& Sch{\"o}lkopf, B. (2015).
\newblock Semi-supervised interpolation in an anticausal learning scenario.
\newblock {\em Journal of Machine Learning Research}, 16(1), 1923--1948.

\bibitem[Johansson et~al., 2020]{johansson2020generalization}
Johansson, F.~D., Shalit, U., Kallus, N., \& Sontag, D. (2020).
\newblock Generalization bounds and representation learning for estimation of
  potential outcomes and causal effects.
\newblock {\em arXiv preprint arXiv:2001.07426}.

\bibitem[Johansson et~al., 2019]{johansson2019support}
Johansson, F.~D., Sontag, D., \& Ranganath, R. (2019).
\newblock Support and invertibility in domain-invariant representations.
\newblock In {\em The 22nd International Conference on Artificial Intelligence
  and Statistics}  (pp.\ 527--536).

\bibitem[Khasanova \& Frossard, 2017]{khasanova2017graph}
Khasanova, R. \& Frossard, P. (2017).
\newblock Graph-based isometry invariant representation learning.
\newblock In {\em International Conference on Machine Learning}  (pp.\
  1847--1856).

\bibitem[Khemakhem et~al., 2020]{khemakhem2020variational}
Khemakhem, I., Kingma, D., Monti, R., \& Hyvarinen, A. (2020).
\newblock Variational autoencoders and nonlinear {ICA}: A unifying framework.
\newblock In {\em Artificial Intelligence and Statistics}  (pp.\ 2207--2217).

\bibitem[Kilbertus et~al., 2018]{kilbertus2018generalization}
Kilbertus, N., Parascandolo, G., \& Sch{\"o}lkopf, B. (2018).
\newblock Generalization in anti-causal learning.
\newblock {\em arXiv preprint arXiv:1812.00524}.

\bibitem[Kim \& Mnih, 2018]{kim2018disentangling}
Kim, H. \& Mnih, A. (2018).
\newblock Disentangling by factorising.
\newblock In {\em International Conference on Machine Learning}  (pp.\
  2649--2658).

\bibitem[Kingma \& Welling, 2014]{kingma2014auto}
Kingma, D.~P. \& Welling, M. (2014).
\newblock Auto-encoding variational {B}ayes.
\newblock {\em International Conference on Learning Representations}.

\bibitem[Kommiya~Mothilal et~al., 2021]{kommiya2021towards}
Kommiya~Mothilal, R., Mahajan, D., Tan, C., \& Sharma, A. (2021).
\newblock Towards unifying feature attribution and counterfactual explanations:
  Different means to the same end.
\newblock In {\em Proceedings of the 2021 AAAI/ACM Conference on AI, Ethics,
  and Society}  (pp.\ 652--663).

\bibitem[Kumar et~al., 2018]{kumar2018variational}
Kumar, A., Sattigeri, P., \& Balakrishnan, A. (2018).
\newblock Variational inference of disentangled latent concepts from unlabeled
  observations.
\newblock In {\em International Conference on Learning Representations}.

\bibitem[Liu et~al., 2015]{liu2015faceattributes}
Liu, Z., Luo, P., Wang, X., \& Tang, X. (2015).
\newblock Deep learning face attributes in the wild.
\newblock In {\em Proceedings of International Conference on Computer Vision
  (ICCV)}.

\bibitem[Locatello et~al., 2019a]{locatello2019challenging}
Locatello, F., Bauer, S., et~al. (2019a).
\newblock Challenging common assumptions in the unsupervised learning of
  disentangled representations.
\newblock In {\em International Conference on Machine Learning}  (pp.\
  4114--4124).

\bibitem[Locatello et~al., 2020]{locatello2020weakly}
Locatello, F., Poole, B., et~al. (2020).
\newblock Weakly-supervised disentanglement without compromises.
\newblock In {\em International Conference on Machine Learning}  (pp.\
  6348--6359).

\bibitem[Locatello et~al., 2019b]{locatello2019disentangling}
Locatello, F., Tschannen, M., et~al. (2019b).
\newblock Disentangling factors of variations using few labels.
\newblock In {\em International Conference on Learning Representations}.

\bibitem[Lu et~al., 2021]{lu2021invariant}
Lu, C., Wu, Y., Hern{\'a}ndez-Lobato, J.~M., \& Sch{\"o}lkopf, B. (2021).
\newblock Invariant causal representation learning.

\bibitem[McLachlan \& Basford, 1988]{mclachlan1988mixture}
McLachlan, G.~J. \& Basford, K.~E. (1988).
\newblock {\em Mixture Models: Inference and Applications to Clustering},
  volume~38.
\newblock M. Dekker New York.

\bibitem[Mitrovic et~al., 2020]{mitrovic2020representation}
Mitrovic, J., McWilliams, B., Walker, J., Buesing, L., \& Blundell, C. (2020).
\newblock Representation learning via invariant causal mechanisms.
\newblock {\em arXiv preprint arXiv:2010.07922}.

\bibitem[Moraffah et~al., 2019]{moraffah2019deep}
Moraffah, R., Shu, K., Raglin, A., \& Liu, H. (2019).
\newblock Deep causal representation learning for unsupervised domain
  adaptation.
\newblock {\em arXiv preprint arXiv:1910.12417}.

\bibitem[Moyer et~al., 2018]{moyer2018invariant}
Moyer, D., Gao, S., Brekelmans, R., Steeg, G.~V., \& Galstyan, A. (2018).
\newblock Invariant representations without adversarial training.
\newblock {\em arXiv preprint arXiv:1805.09458}.

\bibitem[Mueller et~al., 2021]{mueller2021causes}
Mueller, S., Li, A., \& Pearl, J. (2021).
\newblock Causes of effects: Learning individual responses from population
  data.
\newblock {\em arXiv preprint arXiv:2104.13730}.

\bibitem[Nabi et~al., 2020]{nabi2020semiparametric}
Nabi, R., McNutt, T., \& Shpitser, I. (2020).
\newblock Semiparametric causal sufficient dimension reduction of high
  dimensional treatments.
\newblock {\em arXiv preprint arXiv:1710.06727}.

\bibitem[Parascandolo et~al., 2018]{parascandolo2018learning}
Parascandolo, G., Kilbertus, N., Rojas-Carulla, M., \& Sch{\"o}lkopf, B.
  (2018).
\newblock Learning independent causal mechanisms.
\newblock In {\em International Conference on Machine Learning}  (pp.\
  4036--4044).

\bibitem[Paul, 2017]{paul2017feature}
Paul, M. (2017).
\newblock Feature selection as causal inference: Experiments with text
  classification.
\newblock In {\em Proceedings of the 21st Conference on Computational Natural
  Language Learning (CoNLL 2017)}  (pp.\ 163--172).

\bibitem[Pearl, 1995]{pearl1995causal}
Pearl, J. (1995).
\newblock Causal diagrams for empirical research.
\newblock {\em Biometrika}, 82(4), 669--688.

\bibitem[Pearl, 2011]{Pearl2011}
Pearl, J. (2011).
\newblock {\em {Causality: Models, Reasoning, and Inference}}.
\newblock Cambridge {U}niversity Press.

\bibitem[Pearl, 2019a]{pearl2019seven}
Pearl, J. (2019a).
\newblock The seven tools of causal inference, with reflections on machine
  learning.
\newblock {\em Communications of the ACM}, 62(3), 54--60.

\bibitem[Pearl, 2019b]{Pearl2019}
Pearl, J. (2019b).
\newblock {Sufficient causes: On oxygen, matches, and fires}.
\newblock {\em Journal of Causal Inference}, 7(2), 1--8.

\bibitem[Pritchard et~al., 2000]{pritchard2000inference}
Pritchard, J.~K., Stephens, M., \& Donnelly, P. (2000).
\newblock Inference of population structure using multilocus genotype data.
\newblock {\em Genetics}, 155(2), 945--959.

\bibitem[Pryzant et~al., 2020]{pryzant2020causal}
Pryzant, R., Card, D., Jurafsky, D., Veitch, V., \& Sridhar, D. (2020).
\newblock Causal effects of linguistic properties.
\newblock {\em arXiv preprint arXiv:2010.12919}.

\bibitem[Puli et~al., 2020]{puli2020causal}
Puli, A., Perotte, A., \& Ranganath, R. (2020).
\newblock Causal estimation with functional confounders.
\newblock {\em Advances in Neural Information Processing Systems}, 33.

\bibitem[Puli et~al., 2021]{puli2021predictive}
Puli, A., Zhang, L.~H., Oermann, E.~K., \& Ranganath, R. (2021).
\newblock Predictive modeling in the presence of nuisance-induced spurious
  correlations.
\newblock {\em arXiv preprint arXiv:2107.00520}.

\bibitem[Ranganath \& Perotte, 2018]{ranganath2019multiple}
Ranganath, R. \& Perotte, A. (2018).
\newblock Multiple causal inference with latent confounding.
\newblock {\em arXiv preprint arXiv:1805.08273}.

\bibitem[Rezende et~al., 2014]{rezende2014stochastic}
Rezende, D.~J., Mohamed, S., \& Wierstra, D. (2014).
\newblock Stochastic backpropagation and variational inference in deep latent
  gaussian models.
\newblock In {\em International Conference on Machine Learning}, volume~2
  (pp.\~2).

\bibitem[Sch{\"o}lkopf et~al., 2012]{scholkopf2012causal}
Sch{\"o}lkopf, B., Janzing, D., et~al. (2012).
\newblock On causal and anticausal learning.
\newblock In {\em International Conference on Machine Learning}  (pp.\
  1255--1262).

\bibitem[Sch{\"o}lkopf et~al., 2013]{scholkopf2013semi}
Sch{\"o}lkopf, B., Janzing, D., et~al. (2013).
\newblock Semi-supervised learning in causal and anticausal settings.
\newblock In {\em Empirical Inference}  (pp.\ 129--141). Springer.

\bibitem[Shen et~al., 2020]{shen2020disentangled}
Shen, X., Liu, F., et~al. (2020).
\newblock Disentangled generative causal representation learning.
\newblock {\em arXiv preprint arXiv:2010.02637}.

\bibitem[Shi et~al., 2020]{shi2020invariant}
Shi, C., Veitch, V., \& Blei, D. (2020).
\newblock Invariant representation learning for treatment effect estimation.
\newblock {\em arXiv preprint arXiv:2011.12379}.

\bibitem[Shu et~al., 2019]{shu2019weakly}
Shu, R., Chen, Y., Kumar, A., Ermon, S., \& Poole, B. (2019).
\newblock Weakly supervised disentanglement with guarantees.
\newblock In {\em International Conference on Learning Representations}.

\bibitem[Stewart, 1984]{stewart1984rank}
Stewart, G. (1984).
\newblock Rank degeneracy.
\newblock {\em SIAM Journal on Scientific and Statistical Computing}, 5(2),
  403--413.

\bibitem[Suter et~al., 2019]{Suter2019}
Suter, R., Miladinovic, D., Sch{\"o}lkopf, B., \& Bauer, S. (2019).
\newblock {Robustly disentangled causal mechanisms: Validating deep
  representations for interventional robustness}.
\newblock In {\em International Conference on Machine Learning}  (pp.\
  6056--6065).

\bibitem[Thomas et~al., 2018]{thomas2018disentangling}
Thomas, V., Bengio, E., et~al. (2018).
\newblock Disentangling the independently controllable factors of variation by
  interacting with the world.
\newblock {\em arXiv preprint arXiv:1802.09484}.

\bibitem[Thomas et~al., 2017]{thomas2017independently}
Thomas, V., Pondard, J., et~al. (2017).
\newblock Independently controllable factors.
\newblock {\em arXiv preprint arXiv:1708.01289}.

\bibitem[Tian \& Pearl, 2000]{tian2000probabilities}
Tian, J. \& Pearl, J. (2000).
\newblock Probabilities of causation: Bounds and identification.
\newblock {\em Annals of Mathematics and Artificial Intelligence}, 28(1),
  287--313.

\bibitem[Tipping \& Bishop, 1999]{tipping1999probabilistic}
Tipping, M.~E. \& Bishop, C.~M. (1999).
\newblock Probabilistic principal component analysis.
\newblock {\em Journal of the Royal Statistical Society: Series B (Statistical
  Methodology)}, 61(3), 611--622.

\bibitem[Tr{\"a}uble et~al., 2020]{Trauble2020}
Tr{\"a}uble, F., Creager, E., et~al. (2020).
\newblock {Is independence all you need? On the generalization of
  representations learned from correlated data}.
\newblock {\em arXiv preprint arXiv:2006.07886}.

\bibitem[Udell \& Townsend, 2019]{udell2019}
Udell, M. \& Townsend, A. (2019).
\newblock Why are big data matrices approximately low rank?
\newblock {\em SIAM Journal on Mathematics of Data Science}, 1(1), 144--160.

\bibitem[Veitch et~al., 2021]{veitch2021counterfactual}
Veitch, V., D'Amour, A., Yadlowsky, S., \& Eisenstein, J. (2021).
\newblock Counterfactual invariance to spurious correlations: Why and how to
  pass stress tests.
\newblock {\em arXiv preprint arXiv:2106.00545}.

\bibitem[Veitch et~al., 2020]{veitch2020adapting}
Veitch, V., Sridhar, D., \& Blei, D. (2020).
\newblock Adapting text embeddings for causal inference.
\newblock In {\em Conference on Uncertainty in Artificial Intelligence}  (pp.\
  919--928).: PMLR.

\bibitem[Veitch et~al., 2019]{veitch2019using}
Veitch, V., Wang, Y., \& Blei, D.~M. (2019).
\newblock Using embeddings to correct for unobserved confounding in networks.
\newblock {\em arXiv preprint arXiv:1902.04114}.

\bibitem[Wang et~al., 2010]{wang2010latent}
Wang, H., Lu, Y., \& Zhai, C. (2010).
\newblock Latent aspect rating analysis on review text data: {A} rating
  regression approach.
\newblock In {\em Proceedings of the 16th ACM SIGKDD International Conference
  on Knowledge Discovery and Data Mining}  (pp.\ 783--792).

\bibitem[Wang et~al., 2011]{wang2011latent}
Wang, H., Lu, Y., \& Zhai, C. (2011).
\newblock Latent aspect rating analysis without aspect keyword supervision.
\newblock In {\em Proceedings of the 17th ACM SIGKDD International Conference
  on Knowledge Discovery and Data Mining}  (pp.\ 618--626).

\bibitem[Wang \& Blei, 2021]{wang2021proxy}
Wang, Y. \& Blei, D. (2021).
\newblock A proxy variable view of shared confounding.
\newblock In {\em International Conference on Machine Learning}  (pp.\
  10697--10707).: PMLR.

\bibitem[Wang \& Blei, 2019a]{wang2019blessings}
Wang, Y. \& Blei, D.~M. (2019a).
\newblock The blessings of multiple causes.
\newblock {\em Journal of the American Statistical Association}, 114(528),
  1574--1596.

\bibitem[Wang \& Blei, 2019b]{wang2019frequentist}
Wang, Y. \& Blei, D.~M. (2019b).
\newblock Frequentist consistency of variational {B}ayes.
\newblock {\em Journal of the American Statistical Association}, 114(527),
  1147--1161.

\bibitem[Wang \& Blei, 2020]{wang2020towards}
Wang, Y. \& Blei, D.~M. (2020).
\newblock Towards clarifying the theory of the deconfounder.
\newblock {\em arXiv preprint arXiv:2003.04948}.

\bibitem[Wang \& Culotta, 2020]{wang2020identifying}
Wang, Z. \& Culotta, A. (2020).
\newblock Identifying spurious correlations for robust text classification.
\newblock In {\em Proceedings of the 2020 Conference on Empirical Methods in
  Natural Language Processing: Findings}  (pp.\ 3431--3440).

\bibitem[Wang \& Culotta, 2021]{wang2020robustness}
Wang, Z. \& Culotta, A. (2021).
\newblock Robustness to spurious correlations in text classification via
  automatically generated counterfactuals.
\newblock In {\em Proceedings of the AAAI Conference on Artificial
  Intelligence}, volume~35  (pp.\ 14024--14031).

\bibitem[Watson et~al., 2021]{watson2021local}
Watson, D., Gultchin, L., Taly, A., \& Floridi, L. (2021).
\newblock Local explanations via necessity and sufficiency: unifying theory and
  practice.
\newblock {\em arXiv preprint arXiv:2103.14651}.

\bibitem[Weichwald et~al., 2014]{weichwald2014causal}
Weichwald, S., Sch{\"o}lkopf, B., Ball, T., \& Grosse-Wentrup, M. (2014).
\newblock Causal and anti-causal learning in pattern recognition for
  neuroimaging.
\newblock In {\em 2014 International Workshop on Pattern Recognition in
  Neuroimaging}  (pp.\ 1--4).

\bibitem[Wu \& Fukumizu, 2021]{wu2021identifying}
Wu, P. \& Fukumizu, K. (2021).
\newblock Identifying treatment effects under unobserved confounding by causal
  representation learning.
\newblock {\em arXiv preprint arXiv:2101.06662}.

\bibitem[Xiao \& Wang, 2019]{xiao2019disentangled}
Xiao, Y. \& Wang, W.~Y. (2019).
\newblock Disentangled representation learning with {W}asserstein total
  correlation.
\newblock {\em arXiv preprint arXiv:1912.12818}.

\bibitem[Yang et~al., 2020]{yang2020causalvae}
Yang, M., Liu, F., et~al. (2020).
\newblock {CausalVAE}: Disentangled representation learning via neural
  structural causal models.
\newblock {\em arXiv preprint arXiv:2004.08697}.

\bibitem[Zhao et~al., 2019]{zhao2019learning}
Zhao, H., Des~Combes, R.~T., Zhang, K., \& Gordon, G. (2019).
\newblock On learning invariant representations for domain adaptation.
\newblock In {\em International Conference on Machine Learning}  (pp.\
  7523--7532).

\end{thebibliography}
\end{document}